\documentclass[10.5pt, a4paper,english]{article}

 \usepackage{amsmath,amssymb}
 \usepackage{bm}
 \usepackage{ascmac}
 \usepackage{amsthm}
 \usepackage{enumerate}
 \usepackage[dvips]{color}
 \usepackage{here}
\usepackage[pdftex]{graphicx}

\usepackage{mathrsfs} 

\theoremstyle{definition}
\newtheorem{theorem}{Theorem}[section]

\newtheorem{lemma}{Lemma}[section]

\newtheorem{definition}{Definition}[section]

\makeatletter 
 \def\section{\@startsection {section}{1}{\z@}{3.5ex plus -1ex minus -.2ex}{2.3 ex plus .2ex}{\bf}}
 \def\@seccntformat#1{\csname the#1\endcsname.\ }
 \makeatother
 
  \makeatletter 
 \def\subsection{\@startsection {subsection}{1}{\z@}{3.5ex plus -1ex minus -.2ex}{2.3 ex plus .2ex}{\bf}}
 \def\@seccntformat#1{\csname the#1\endcsname.\ }
 \makeatother

\numberwithin{equation}{section} 

\newcommand{\f}{^\forall}

\newcommand{\argmin}{\operatornamewithlimits{argmin}}
\newcommand{\argmax}{\operatornamewithlimits{argmax}}

\usepackage{algorithm}
\usepackage{algcompatible}

\newcommand{\1}{\mbox{1}\hspace{-0.25em}\mbox{l}}

\usepackage{slashbox}

\usepackage{geometry}
\begin{document}
\begin{center}
{\Large Active learning for distributionally robust level-set estimation} \vspace{5mm} 

Yu Inatsu$^1$ \ \ \ \ Shogo Iwazaki$^1$ \ \ \ \ Ichiro Takeuchi$^{1,2,\ast}$ \vspace{3mm}   

$^1$ Department of Computer Science, Nagoya Institute of Technology    \\
$^2$ RIKEN Center for Advanced Intelligence Project \\
$^\ast$ E-mail: takeuchi.ichiro@nitech.ac.jp
\end{center}

\vspace{5mm} 

\begin{abstract}
Many cases exist in which a black-box function $f$ with high evaluation cost depends on two types of variables $\bm x$ and $\bm w$, where $\bm x$ is a controllable \emph{design} variable and $\bm w$ are uncontrollable \emph{environmental} variables that have random variation following a certain distribution $P$. 
In such cases, an important task is to find the range of design variables $\bm x$ such that the function $f(\bm x, \bm w)$ has the desired properties by incorporating the random variation of the environmental variables $\bm w$. 
A natural measure of robustness is the probability that $f(\bm x, \bm w)$ exceeds a given threshold $h$, which is known as the \emph{probability threshold robustness} (PTR) measure in the literature on robust optimization. 
However, this robustness measure cannot be correctly evaluated when the distribution $P$ is unknown. 
In this study, we addressed this problem by considering the \textit{distributionally robust PTR} (DRPTR) measure, which considers the worst-case PTR within given candidate distributions. 
Specifically, we studied the problem of efficiently identifying a reliable set $H$, which is defined as a region in which the DRPTR measure exceeds a certain desired probability $\alpha$, which can be interpreted as a level set estimation (LSE) problem for DRPTR. 
We propose a theoretically grounded and computationally efficient active learning method for this problem. 
We show that the proposed method has theoretical guarantees on convergence and accuracy, and confirmed through numerical experiments that the proposed method outperforms existing methods.
\end{abstract}

\section{Introduction}
In the manufacturing industry, product performance often depends on two types of variables: design variables and environmental variables.
The design variables are completely controllable, whereas environmental variables are random variables that change depending on the usage environment of the product.
When considering such a problem, it is important to identify the design variables that allow the product performance to exceed the desired requirement threshold with a sufficiently high degree of confidence, taking into account the randomness of the environmental variables.
In this setting, we must emphasize that there are two distinctly different phases of the product: the development phase and the use phase.
In the development phase, we have full control over the design variables and environmental variables.
In the use phase, on the other hand, the design variables are fixed, and the environmental variables change randomly and cannot be controlled.

Let $f({\bm x},{\bm w} )$ represent the performance of the product, and let $h \in \mathbb{R}$ be a desired performance threshold, 
where ${\bm x}$ is a design variable defined on $\mathcal{X}$, and ${\bm w} $ is  an  environmental variable defined on $\Omega$.  
Then, we consider the following robustness measure:
$$
\text{PTR} ({\bm x} ) = \int _{\Omega}  \1 [ f({\bm x},{\bm w} ) >h ] p^\dagger ({\bm w} )  \text{d} {\bm w} ,
$$
where $\1[ \cdot] $ is the indicator function and $p^\dagger ({\bm w} )$ is the probability density function of ${\bm w} $. 
This measure is called the probability threshold robustness (PTR) measure in the field of robust optimization \cite{beyer2007robust}, and can be interpreted as a measure of how well the design variables behave under randomness in the environmental variables.
In the manufacturing industry, it is desirable to identify the set of controllable variables $\bm x \in \mathcal{X}$ for which ${\rm PTR}(\bm x)$ is greater than a certain threshold. 
In other words, this problem is interpreted as a level-set estimation (LSE) \cite{bryan2006active,Gotovos:2013:ALL:2540128.2540322} of the PTR measure. 
There are two main reasons for considering LSE of the PTR measure.
One is that by enumerating all the design variables that exceed the desired threshold with a high probability, it is possible to respond the usage conditions of various users.
The other is to consider some optimization problem (e.g., to find ${\bm x}$ with the minimum price) for design variables with PTR measures above a certain level.
This is known as the chance-constrained programming problem \cite{charnes1959chance}, and has many applications such as finance, in addition to manufacturing industry.
Unfortunately, however, the PTR measure cannot be correctly evaluated when $p^\dagger ({\bm w})$ is unknown.  If $p^\dagger ({\bm w})$ is unknown and the estimated density is simply plugged in, then $\text{PTR} ({\bm x} ) $ is no longer valid as a robustness measure because of the estimation error.

In this study, we considered a distributionally robust PTR (DRPTR) measure, which includes uncertainty about $p^\dagger ({\bm w})$ under the setting that $p^\dagger ({\bm w})$ is unknown.
Let $\mathcal{A}$ be a user-specified class of candidate distributions of ${\bm w}$.
Then, the DRPTR measure can be defined as
$$
F ({\bm x} ) = \inf _{ p({\bm w} ) \in \mathcal{A}  } \int _{\Omega}  \1 [ f({\bm x},{\bm w} ) >h ] p ({\bm w} )  \text{d} {\bm w} .
$$
The DRPTR measure has the advantage of being robust with respect to using wrong distributions because it can be interpreted as the PTR in the worst case among the candidate distributions.
In this study, we formulated this problem as an active learning problem for the LSE for $F({\bm x})$ instead of ${\rm PTR} ({\bm x} )$, and developed a theoretically grounded and numerically efficient algorithm for its calculation.
The basic ideas of our proposed method are as follows. 
First, we consider the function $f ({\bm x}, {\bm w})$ to be a black-box function with a high evaluation cost, and we employ a Gaussian process (GP) model as a surrogate model.
Next, we predict the target DRPTR measure using the GP model for the black-box function $f ({\bm x}, {\bm w})$.
Finally, we perform LSE using credible intervals of the DRPTR measure calculated on the basis of this prediction.

\subsection{Related work}
Active learning using GP models \cite{williams2006gaussian} 
 for black-box functions have been actively studied in the context of Bayesian optimization (see, e.g., \cite{settles2009active,shahriari2016taking}).
Several studies have been conducted on active learning for LSE \cite{bryan2006active,Gotovos:2013:ALL:2540128.2540322,zanette2018robust,inatsu2020-b-active}.
Furthermore, some researchers applied LSE to efficiently identify safety regions \cite{sui2015safe,turchetta2016safe,sui2018stagewise,wachi2018safe}, and others used LSE to enumerate the local minima of black-box functions \cite{inatsu2020-a-active}.

Many studies have been conducted on active learning under input uncertainty (including random environmental variables). 
In \cite{inatsu2020-c-active}, the authors proposed an efficient method for performing LSE in the setting where the input is a random variable generated from a certain distribution.
In other studies, the researchers formulated the randomness of the input with some robustness measures for performing active learning on it.
For example, the authors of \cite{bogunovic2018adversarially} used the worst-case function value of the input shift as a robustness measure. 
Similarly, other research (\cite{beland2017bayesian,toscano2018bayesian,oliveira2019bayesian,pmlr-v108-frohlich20a,gessner2020active,iwazaki2020mean}) dealt with the stochastic robustness (SR) measure, which is a robustness measure defined by integrating the black-box function against the input distribution.
In another study closely related to the present work, the authors of \cite{iwazaki2020-b-bayesian} proposed an active learning method for LSE in the PTR measure on the basis of random inputs; in \cite{iwazaki2020mean}, the authors considered an active learning method for both LSE and maximization problems in the PTR measure.
However, these two are not distributionally robust settings.
Distributionally robust optimization (DRO), which is not an active learning framework, was first introduced by \cite{scarf1958min}.
DRO is an important topic in the context of robust optimization, and there have been countless related studies (see \cite{rahimian2019distributionally} for comprehensive survey of DRO).
%
%
%
%
Active learning methods for DRO with uncertainty environmental variables have recently been proposed by \cite{pmlr-v108-kirschner20a,pmlr-v108-nguyen20a}. 
The main differences to our problem setup are that they focus on a distributionally robust SR (DRSR) measure for the target function, which is the worst-case SR measure in candidate distributions of the unknown environmental variable, and consider the maximization problem for the DRSR measure.
In particular, for the former, we cannot directly apply their proposed methods and theoretical techniques because the target function is different from ours.
%
%
To the best of our knowledge, none of these studies have addressed the same research problem considered in the present work.

\subsection{Contributions}
The main contributions of this study are summarized as follows:
\begin{itemize}
\item We formulate the LSE problem for the DRPTR measure, i.e., the problem of finding the set of design variables for which the DRPTR measure exceeds a given threshold.
\item We construct non-trivial credible intervals for the DRPTR measure and propose a new acquisition function (AF) based on an expected classification improvement.
Using them, we propose an active learning method for the LSE of the DRPTR measure.
Moreover, because the naive implementation of our proposed AF requires a large computational cost, we propose a computationally efficient technique for its calculation. 
\item We clarify the theoretical property of the proposed method. 
Under mild conditions, we show that the proposed method has desirable accuracy and convergence properties.
\item We describe the empirical performance of the proposed method through the results of numerical experiments with benchmark functions and real data.
\end{itemize}
\section{Preliminary}
Let $f: \mathcal{X} \times \Omega \to \mathbb{R}$ be an expensive-to-evaluate black-box function. 
We assume that $\mathcal{X}$ and $\Omega$ are finite sets. 
For each input $({\bm x}, {\bm w}) \in \mathcal{X} \times \Omega$, the value of $f({\bm x},{\bm w})$ is observed as $f({\bm x},{\bm w}) + \varepsilon$ with an independent noise $\varepsilon$, where 
$\varepsilon$ follows Gaussian distribution $\mathcal{N} (0,\sigma^2)$. 
In our setting, a variable 
 ${\bm w} \in \Omega$ stochastically fluctuates by the (unknown) discrete distribution $P^\dagger$ in the use phase, whereas we can 
specify ${\bm w}$ in the development phase. 
Moreover, let 
  $\mathcal{A} $ be a family of candidate distributions of $P^\dagger$.
In this work, we consider 
 $\mathcal{A}=\{ \text{p.m.f.} \ p({\bm w}) \mid d( p({\bm w} ),p^\ast ({\bm w} ) ) <\epsilon \}$.  
where $p^\ast ({\bm w} )$ is a user-specified reference distribution,     $d(\cdot,\cdot)$ is a given distance metric between two distributions, and 
 $\epsilon >0$. 
Then, under the given threshold $h$, we define the DRPTR $F({\bm x} )$ for each ${\bm x} \in \mathcal{X}$ as
$$
F({\bm x}) = \inf _{p({\bm w} ) \in \mathcal{A} }   \sum_{ {\bm w} \in \Omega }   \1[f({\bm x},{\bm w} ) >h] p({\bm w}) .
$$
The aim of this study was to efficiently identify a subset $H$ of $\mathcal{X}$ that satisfies $F({\bm x})>\alpha$ for a given threshold  
 $\alpha \in (0,1)$:
\begin{align}
H= \{ {\bm{x}} \in \mathcal{X} \mid F({\bm{x}} ) >\alpha \} . \label{eq:HandL}
\end{align}
Moreover, we define the lower set $L$ as $L= \{ {\bm{x}} \in \mathcal{X} \mid F({\bm{x}} ) \leq \alpha \}$.

\paragraph{Gaussian process}
In this study, we used the Gaussian process (GP) to model the unknown black-box function $f$.
First, we assume that the GP, $\mathcal{G}\mathcal{P}(0, k(  ( {\bm x},{\bm w}  ),  ( {\bm x}^\prime,{\bm w}^\prime  )  )       )$ is a prior distribution of $f$, where $k(  ( {\bm x},{\bm w}  ),  ( {\bm x}^\prime,{\bm w}^\prime  )  ) $ is a positive-definite kernel. 
Then, given the dataset $\{ ({\bm x}_i,{\bm w}_i,y_i ) \}_{i=1}^t$, the posterior distribution of $f$ also follows the GP, and its posterior mean 
 $ \mu_t ({\bm x},{\bm w}) $ and posterior variance  $ \sigma^2_t ({\bm x},{\bm w}) $ are given by
\begin{align*}
\mu_t ({\bm x},{\bm w}) &= {\bm k}^\top _t ({\bm x},{\bm w} ) ({\bm K}_t + \sigma^2 {\bm I}_t )^{-1} {\bm y}_t , \\  
\sigma^2_t ({\bm x},{\bm w} )&= k(  ({\bm x},{\bm w}),  ({\bm x},{\bm w})) 
 -{\bm k}^\top_t ({\bm x},{\bm w}) 
({\bm K}_t + \sigma^2 {\bm I}_t )^{-1} {\bm k} _t ({\bm x},{\bm w}),
\end{align*} 
where ${\bm k}_t ({\bm x},{\bm w} ) $ is the $t$-dimensional vector whose $j$th element is $ k(({\bm x},{\bm w} ) ,({\bm x}_j,{\bm w}_j)) $, ${\bm y}_t = (y_1,\ldots , y_t )^\top $, ${\bm I}_t $ is the $t \times t$ identity matrix, and ${\bm K}_t $ is the $t \times t$ 
matrix whose $(j,k)$th element is $ k(({\bm x}_j,{\bm w}_j ) ,({\bm x}_k,{\bm w}_k))$.

\section{Proposed method}\label{sec:proposed}
In this section, we propose an active learning method for efficiently identifying \eqref{eq:HandL}.
The target function $F({\bm x})$ is a random variable because $F({\bm x} )$ is the function of $f({\bm x} ,{\bm w} )$, and 
 $f({\bm x} ,{\bm w} )$ is drawn from GP. 
Thus, a reasonable method to identify \eqref{eq:HandL} is to construct a credible interval of $F({\bm x})$, and estimate $H$ using the lower bound of the constructed credible interval.
Unfortunately, although $f ({\bm x},{\bm w})$ follows GP, $F ({\bm x})$ does not follow GP.
Hence, the credible interval of $F ({\bm x})$ cannot be directly calculated on the basis of normal distributions.
In the next section, we propose a simple and theoretically valid credible interval of $F ({\bm x})$ using the credible interval of $f({\bm x},{\bm w})$.
\subsection{Credible interval and LSE}
For any input 
 $({\bm {x}},{\bm w}) \in \mathcal{X} \times  \Omega$ and step $t$, 
we define a credible interval of  $f({\bm {x}},{\bm w})$ as $Q_t ({\bm {x}},{\bm w}) =[ l_t ({\bm {x}},{\bm w}), u_t ({\bm {x}},{\bm w})]$,  
where 
$l_t ({\bm {x}},{\bm w}) = \mu_t ({\bm {x}},{\bm w}) - \beta^{1/2}_t \sigma_t ({\bm {x}},{\bm w}) $, 
 $u_t ({\bm {x}},{\bm w}) = \mu_t ({\bm {x}},{\bm w}) + \beta^{1/2}_t \sigma_t ({\bm {x}},{\bm w}) $, and $\beta^{1/2}_t \geq 0$. 
Similarly, we define a credible interval of 
$\1[f({\bm x},{\bm w} )>h ]$ on the basis of $Q_t ({\bm {x}},{\bm w})$. 
For the theoretical analysis described in Section \ref{sec:THEOREM}, we introduce a user-specified accuracy parameter $\eta > 0$.
Specifically, we define the credible interval of $\1[f(\bm x, \bm x) > h]$ at step $t$ as
\begin{align*}
&\tilde{Q}_t ({\bm {x}},{\bm w};\eta) \equiv [\tilde{l}_t ({\bm {x}},{\bm w};\eta),\tilde{u}_t ({\bm {x}},{\bm w};\eta)] \\
&=
\begin{cases}
[1,1] & \text{if} \ l_t ({\bm x},{\bm w} ) >h-\eta, \\
[0,1] &   \text{if} \ l_t ({\bm x},{\bm w} ) \leq h-\eta \ {\tt and} \ u_t ({\bm x},{\bm w} ) >h, \\
[0,0] & \text{if} \ l_t ({\bm x},{\bm w} ) \leq h-\eta \ {\tt and} \ u_t ({\bm x},{\bm w} )  \leq h.
\end{cases}
\end{align*}
Note that when the accuracy parameter $\eta = 0$, this credible interval simply indicates that if the lower (resp. upper) bound of $f(\bm x, \bm w)$ is greater (resp. smaller) than $h$, we say that $\1[f(\bm x, \bm w) > h] = 1 \ (\text{resp.} \ 0)$. 
Thus, a credible interval $Q^{(F)}_t ({\bm x};\eta) \equiv [l^{(F)}_t ({\bm x} ;\eta),  u^{(F)}_t ({\bm x} ;\eta)      ]$ of the target function $F({\bm x})$ can be given by
\begin{equation}
\begin{split}
l^{(F)}_t ({\bm x} ;\eta) = \inf _{ p({\bm w} ) \in \mathcal{A} } \sum_{ {\bm w} \in \Omega }  \tilde{l}_t ({\bm {x}},{\bm w};\eta) p({\bm w} ),   \
u^{(F)}_t ({\bm x} ;\eta) = \inf _{ p({\bm w} ) \in \mathcal{A} } \sum_{ {\bm w} \in \Omega }  \tilde{u}_t ({\bm {x}},{\bm w};\eta) p({\bm w} ).  
\end{split}
\label{eq:F_upper_lower}
\end{equation}
Note that if we use the $L1$ (or $L2$)-norm as the distance function $d(\cdot,\cdot)$, 
equation \eqref{eq:F_upper_lower} is equivalent to solving a linear (or second-order cone) 
programming problem. 
In both cases, because solvers exist that can compute the optimal solution quickly, it is easy to compute $Q^{(F)}_t ({\bm x};\eta)$ when using such distance functions. 
Then, we estimate $H$ and $L$ using $Q^{(F)}_t ({\bm x};\eta)$ as follows:
\begin{equation}
\begin{split}
H_t  = \{ {\bm {x}} \in \mathcal{X} \mid  l^{(F)}_t ({\bm {x}};\eta) > \alpha \},  \ 
L_t  = \{ {\bm {x}} \in \mathcal{X} \mid u^{(F)}_t ({\bm {x}};\eta)  \leq \alpha \} .
\end{split}
\nonumber 
\end{equation}
Also, we define the unclassified set as $U_t = \mathcal{X} \setminus (H_t \cup L_t )$.

\subsection{Acquisition function}
In this section, we propose two acquisition functions to select the next evaluation point.
Our proposed acquisition functions are based on the maximum improvement in level-set estimation (MILE) strategy proposed in \cite{zanette2018robust}. 
In MILE, the expected value of the increase in the number of classifications after adding the new point $({\bm x}^\ast,{\bm w}^\ast)$ is calculated, and the point with the largest expected value is selected. 
In this study, owing to the computational cost of calculating the acquisition function, we consider a strategy based on the expected value where points in the unclassified set are classified as $H$.

Let 
$({\bm{x}} ^\ast, {\bm w}^\ast)$ be a new point, and let $y^\ast = f( {\bm{x}} ^\ast ,{\bm w}^\ast)+ \varepsilon$ be a new observation at point $({\bm{x}} ^\ast, {\bm w}^\ast)$. 
Furthermore, let $l^{(F)}_t ({\bm x};0 | {\bm x}^\ast,{\bm w}^\ast,y^\ast )$ be the lower bound of the credible interval of $F({\bm x} )$,  where $\eta =0$ when $ ({\bm{x}}^\ast,{\bm w}^\ast,y^\ast)$ is newly added.
Then, we consider the function $a_t ({\bm x}^\ast,{\bm w}^\ast )$:
\begin{align}
a_t ({\bm  x}^\ast,{\bm w}^\ast ) = \sum_{ {\bm x} \in U_t }  \mathbb{E}_{y^\ast}  [ \1 [l^{(F)}_t ({\bm x} ;0 |{\bm x}^\ast,{\bm w}^\ast, y^\ast) > \alpha ]]  .
 \label{eq:AFteigi}
\end{align}
In this work, we do not directly use \eqref{eq:AFteigi} as the acquisition function because the value of \eqref{eq:AFteigi} is sometimes exactly zero for any point.
A reasonable method to avoid this problem is to consider a different function $b_t ({\bm x}^\ast,{\bm w}^\ast )$ only when the values of \eqref{eq:AFteigi} are all zero. 
 For theoretical treatment, we follow the strategy described in \cite{zanette2018robust}, and 
 consider the acquisition function of the form $\max\{ a_t ({\bm x}^\ast,{\bm w}^\ast ), \gamma b_t ({\bm x}^\ast,{\bm w}^\ast ) \}$  with  a positive constant parameter $\gamma$. 
Note that if we use a sufficiently small $\gamma $, it is almost the same when considering $b_t ({\bm x}^\ast,{\bm w}^\ast ) $ only when  
 the values of \eqref{eq:AFteigi} are all zero; otherwise, $a_t ({\bm x}^\ast,{\bm w}^\ast )$. 
In Section \ref{sec:THEOREM}, we present the theoretical guarantees of our proposed method for this acquisition function.
In this section, we propose two types of $b_t ({\bm x}^\ast,{\bm w}^\ast )$. 
The first is based on the RMILE acquisition function proposed by \cite{zanette2018robust}. 
The basic idea of RMILE is to add an additional variance term $\gamma \sigma_t ({\bm x}^\ast,{\bm w}^\ast ) $ to the original MILE acquisition function. 
By using the same argument, we define the following modified acquisition function:
\begin{definition}[Proposed acquisition function 1]
 Let $a_t ({\bm x}^\ast,{\bm w}^\ast )$ be the function defined by \eqref{eq:AFteigi}, and let $\gamma $ be a positive parameter.  
Then, we propose the following acquisition function $a^{(1)}_t ({\bm x}^\ast,{\bm w}^\ast )$:
\begin{align*}
a^{(1)}_t ({\bm  x}^\ast,{\bm w}^\ast ) 
= \max \{  a_t ({\bm  x}^\ast,{\bm w}^\ast ) , \gamma \sigma_t ({\bm x}^\ast, {\bm w}^\ast )   \}.
\end{align*}
Moreover, we select the next evaluation point $({\bm x}_{t+1},{\bm w}_{t+1} ) $ by maximizing $a^{(1)}_t ({\bm x}^\ast,{\bm w}^\ast )$.
\end{definition}
The other acquisition function we propose uses $\gamma {\rm RMILE}_t ({\bm x}^\ast,{\bm w}^\ast )$ instead of $\gamma \sigma_t ({\bm x}^\ast, {\bm w}^\ast ) $ as the function $b_t ({\bm x}^\ast,{\bm w}^\ast )$, where $ {\rm RMILE}_t ({\bm x}^\ast,{\bm w}^\ast )$ is the  RMILE function proposed in \cite{zanette2018robust}. 
%
\begin{definition}[Proposed acquisition function 2]
 Let $a_t ({\bm x}^\ast,{\bm w}^\ast )$ be the function defined by \eqref{eq:AFteigi}, and let $\gamma $ be a positive parameter.  
Then, we propose the following acquisition function $a^{(2)}_t ({\bm x}^\ast,{\bm w}^\ast )$:
\begin{align*}
a^{(2)}_t (({\bm  x}^\ast,{\bm w}^\ast ) ) 
= \max \{  a_t ({\bm  x}^\ast,{\bm w}^\ast ) , \gamma {\rm RMILE}_t ({\bm x}^\ast,{\bm w}^\ast )   \}.
\end{align*}
Moreover, we select the next evaluation point $({\bm x}_{t+1},{\bm w}_{t+1} ) $ by maximizing $a^{(2)}_t ({\bm x}^\ast,{\bm w}^\ast )$.
\end{definition}
The pseudocode of the proposed method is given in Algorithm \ref{alg:1}.
\begin{algorithm}[t]
    \caption{Active learning for distributionally robust level-set estimation}
    \label{alg:1}
    \begin{algorithmic}
        \REQUIRE GP prior $\mathcal{GP}(0,\ k)$, threshold $h \in \mathbb{R}$, probability $\alpha \in (0,1)$, accuracy parameter 
        ~$\eta > 0$, tradeoff parameter $\{\beta_t\}_{t \leq T}$ 
        \STATE $H_0\leftarrow \emptyset$, $L_0 \leftarrow \emptyset$, $U_0 \leftarrow \mathcal{X}$, $t \leftarrow 1$
        \WHILE{$U_{t-1} \neq \emptyset$}
            \STATE Compute $l^{(F)}_t ({\bm x};\eta ) $ and  $u^{(F)}_t ({\bm x};\eta ) $ for all $\bm{x} \in \mathcal{X}$
            \STATE Choose $(\bm{x}_t, \bm{w}_t)$ by   $(\bm{x}_t,\bm{w}_t)= \argmax _{ ({\bm x}^\ast,{\bm w}^\ast ) \in \mathcal{X} \times \Omega } a^{(1)}_{t-1} ({\bm x}^\ast,{\bm w}^\ast ) $  $ (\text{or }a^{(2)}_{t-1} ({\bm x}^\ast,{\bm w}^\ast ) $ instead of $ a^{(1)}_{t-1} ({\bm x}^\ast,{\bm w}^\ast )  )$
            \STATE Observe $y_t \leftarrow f(\bm{x}_t, \bm{w}_t) + \varepsilon_t$
            \STATE Update GP by adding $((\bm{x}_t, \bm{w}_t), y_t)$ and compute ${H}_t, {L}_t$ and ${U}_t$ 
            \STATE $t \leftarrow t + 1$
        \ENDWHILE
        \STATE $\hat{{H}} \leftarrow {H}_{t-1}, \hat{{L}} \leftarrow {L}_{t-1}$
        \ENSURE Estimated Set $\hat{{H}}, \hat{{L}}$
    \end{algorithmic}
\end{algorithm}

\subsection{Computational techniques}\label{subsec:C_tech}
Our proposed acquisition functions are based on \eqref{eq:AFteigi}, where \eqref{eq:AFteigi} includes the calculation of the expected value.
This expectation cannot be expressed as a simple expression using the cumulative distribution function (CDF) of the standard normal distribution, as in the original MILE \cite{zanette2018robust}. 
One way to solve this problem is to generate many samples from the posterior distribution of $y^\ast$ and numerically calculate the expected value.
However, because one optimization calculation is required to calculate $\1 [l^{(F)}_t ({\bm x} ;0 |{\bm x}^\ast,{\bm w}^\ast, y^\ast) > \alpha ]$,  
if the expected value is calculated using $M$ samples, then $M |U_t |$ 
optimization calculations are required to calculate $a_t ({\bm x}^\ast,{\bm w}^\ast )$ for each $({\bm x}^\ast,{\bm w}^\ast )$. 
Therefore, to calculate $a_t ({\bm x}^\ast,{\bm w}^\ast )$ for all candidate points,  
$M |U_t | | \mathcal{X} \times \Omega|$ optimization calculations are required.
 To reduce this large computational cost, we provide useful lemmas for efficiently computing the acquisition function.  
The expected values in \eqref{eq:AFteigi} can be exactly calculated using the following lemma: 
\begin{lemma}\label{lem:exp_cal}
Let $l_t ({\bm x},{\bm w}_j | {\bm x}^\ast,{\bm w}^\ast,y^\ast )$ be the lower confidence bound of $f ({\bm x},{\bm w}_j )$ after adding 
$({\bm x}^\ast,{\bm w}^\ast,y^\ast ) $ to $\{ ({\bm x}_i , {\bm w}_i ,y_i ) \} _{i=1}^t $. 
 Furthermore, let $r_j $ be a number satisfying $h = l_t ({\bm x},{\bm w}_j | {\bm x}^\ast,{\bm w}^\ast,r_j)$, and let 
$r^{(j)}$ be the $j$th-smallest number in the range $r_1$ to $r_{|\Omega|} $. For each $s \in \{1,\ldots, |\Omega|+1\} \equiv [ |\Omega|+1]$, 
define $R_s = (r^{(s-1)},r^{(s)}]$, where $r^{(0)} = -\infty$ and $r^{(|\Omega|+1)} =\infty$. Moreover, let $c_s $ be a real number satisfying $c_s \in R_s$. 
Then, $ \mathbb{E}_{y^\ast}  [ \1 [l^{(F)}_t ({\bm x} ;0 |{\bm x}^\ast,{\bm w}^\ast, y^\ast) > \alpha ] ] $ can be calculated as follows:
\begin{equation}
\begin{split}
 \mathbb{E}_{y^\ast}  [ \1 [l^{(F)}_t ({\bm x} ;0 |{\bm x}^\ast,{\bm w}^\ast, y^\ast) > \alpha ] ]   
= 
\sum_{s=1}^{|\Omega|+1}   \mathbb{P} ( y^\ast \in R_s )   \1 [l^{(F)}_t ({\bm x} ;0 |{\bm x}^\ast,{\bm w}^\ast, c_s) > \alpha ]. 
\end{split}
\label{eq:exp_ana}
\end{equation}
\end{lemma}
Lemma \ref{lem:exp_cal} implies that $|\Omega|+1$ optimization calculations are required to calculate $ \mathbb{E}_{y^\ast}  [ \1 [l^{(F)}_t ({\bm x} ;0 |{\bm x}^\ast,{\bm w}^\ast, y^\ast) > \alpha ] ] $, but the following lemma shows that the number of optimization calculations can be reduced by checking a simple inequality:
\begin{lemma}\label{lem:teigi_bound}
Let $c_1,\ldots, c_{|\Omega|+1}$ be numbers defined as in Lemma \ref{lem:exp_cal}. Suppose that $c_s$ satisfies 
\begin{align*}
\sum_{ {\bm w} \in \Omega } \1[l_t ({\bm x},{\bm w} | {\bm x}^\ast,{\bm w}^\ast,c_s ) > h ]  p^\ast ({\bm w} ) \leq \alpha .
\end{align*}
Then, $ \1 [l^{(F)}_t ({\bm x} ;0 |{\bm x}^\ast,{\bm w}^\ast, c_s) > \alpha ]=0$.
\end{lemma}

Finally, noting that  $0 \leq  \mathbb{P} ( y^\ast \in R_s )  \leq 1$ and $ 0 \leq \1 [l^{(F)}_t ({\bm x} ;0 |{\bm x}^\ast,{\bm w}^\ast, c_s) > \alpha ] \leq 1$, we can approximate \eqref{eq:exp_ana} with any approximation accuracy $\zeta >0$:
\begin{lemma}\label{lem:approx_any_accuracy}
Let $\zeta >0$, and define 
\begin{align*}
\hat{a}_t ({\bm x}^\ast,{\bm w}^\ast ) &= \sum _{ s \in S_t }   \mathbb{P} ( y^\ast \in R_s )   \1 [l^{(F)}_t ({\bm x} ;0 |{\bm x}^\ast,{\bm w}^\ast, c_s) > \alpha ], \\
S_t &= \{    s \in  [ |\Omega|+1] \mid  \mathbb{P} ( y^\ast \in R_s ) \geq  \zeta /( |\Omega|+1) \}.
\end{align*}
Then, $\hat{a}_t ({\bm x}^\ast,{\bm w}^\ast )$ satisfies the following inequality:
$$
|  \mathbb{E}_{y^\ast}  [ \1 [l^{(F)}_t ({\bm x} ;0 |{\bm x}^\ast,{\bm w}^\ast, y^\ast) > \alpha ] ] - \hat{a}_t ({\bm x}^\ast,{\bm w}^\ast ) | \leq \zeta.
$$
\end{lemma}
Lemma \ref{lem:approx_any_accuracy} implies that the number of optimization calculations for \eqref{eq:exp_ana} can be further reduced if the error $\zeta$ is allowed.
In addition, we must emphasize that $ \mathbb{P} ( y^\ast \in R_s )  $ is often very small for most $s$ when we actually calculate \eqref{eq:exp_ana}.
Therefore, from these properties, if we apply Lemma \ref{lem:approx_any_accuracy} using a sufficiently small $\zeta $, we can reduce the computational cost of \eqref{eq:exp_ana} significantly with almost no error.
Detailed numerical comparisons are provided in Section \ref{sec:sec5}.

\section{Theoretical analysis}\label{sec:THEOREM}
In this section, we provide three theorems regarding the accuracy and convergence properties of our methods.
First, we define the misclassification loss $e_\alpha ({\bm x})$ for each 
 ${\bm  x} \in \mathcal{X}$ as follows:
\begin{align*}
e_\alpha ({\bm x}) = 
\left \{
\begin{array}{ll}
\max \{ 0,  F({\bm x}) -\alpha \}  & \text{if} \  {\bm x} \in \hat{L}  \\
\max \{ 0, \alpha- F({\bm x}) \} & \text{if} \ {\bm x} \in \hat{H} 
\end{array}
\right . .
\end{align*}
Furthermore, for theoretical reasons, we assume that the black-box function $f$ follows GP $\mathcal{G}\mathcal{P} (0,k(({\bm x},{\bm w}),({\bm x}^\prime,{\bm w}^\prime)))$.
In addition, for technical reasons, we assume that the prior variance 
 $k(({\bm x},{\bm w}),({\bm x},{\bm w})) \equiv \sigma^2_0 ({\bm x},{\bm w} ) $ satisfies
\begin{align*}
0< \sigma^2_{0,min} \equiv \min_{( {\bm x},{\bm w} ) \in \mathcal{X}\times \Omega } \sigma^2_0 ({\bm x},{\bm w} ) 
\leq \max_{( {\bm x},{\bm w} ) \in \mathcal{X}\times \Omega } \sigma^2_0 ({\bm x},{\bm w} ) \leq 1.
\end{align*}
Moreover, let $\kappa_T$ be the  maximum information gain at step $T$. 
Note that $\kappa_T$ is a measure often used to show theoretical guarantee for GP-based active learning methods (see, e.g.,  \cite{SrinivasGPUCB}), and can be expressed using mutual information $I({\bm y};f)$ between 
the observed vector ${\bm y}$ and $f$ 
as 
$
\kappa_T = \max_{ A \subset \mathcal{X}\times \Omega } I({\bm y}_A;f ). 
$ 
Then, the following theorem regarding accuracy holds:
\begin{theorem}\label{thm:seido}
Let $h \in \mathbb{R}$, $\alpha \in (0,1)$, $t \geq1$, and $\delta \in (0,1)$, and define 
$\beta_t = 2 \log (|\mathcal{X}\times \Omega | \pi^2 t^2 /(3\delta ) )$. 
Moreover, for a user-specified accuracy parameter $\xi >0 $, we define $\eta >0$ as 
$$
\eta = \min \left \{
\frac{\xi \sigma_{0,min} }{2}, \frac{\xi ^2 \delta \sigma_{0,min}}{8 |\mathcal{X}\times \Omega | }
\right \}.
$$
Then, when Algorithm \ref{alg:1} terminates, with a probability of at least $1-\delta$, the misclassification loss is bounded by $\xi$, that is, the following inequality holds:
$$
\mathbb{P} \left ( \max _{{\bm x} \in \mathcal{X}} e_\alpha ({\bm x}) \leq \xi \right ) \geq 1-\delta.
$$
\end{theorem}
Theorem \ref{thm:seido} does not state whether Algorithm \ref{alg:1} terminates. 
The following theorem guarantees the convergence property in Algorithm \ref{alg:1}:
\begin{theorem}\label{thm:convergence1}
Under the same setting as described in 
Theorem \ref{thm:seido}, let $\gamma >0$ and  $C_1 = 2/ \log (1+\sigma^{-2} ) $. 
In addition, let 
 $T$ be the smallest positive integer satisfying the following four inequalities:  
\begin{align*}
&(1)\quad \frac{\sigma^{-2} \beta^{1/2}_T C_1 \kappa_T }{T}< \frac{ \eta }{2}, \quad 
 (2) \quad 
\frac{ \sigma^{-2} C_1 \kappa_T }{T} < \frac{\eta^2}{4} , \quad
(3) \quad \frac{C_1 \beta_T \kappa_T }{T} < \frac{\eta^2}{4} ,  \\
& (4) \quad 
\frac{1}{2} \log \beta_T -\frac{  T \eta^2 \sigma^2}{8C_1 \kappa_T}  < \log ( |\mathcal{X}| ^{-1} 2^{-|\Omega| }   \eta \gamma (2 \pi)^{1/2}/2 ).
\end{align*}
Then, Algorithm \ref{alg:1} terminates (i.e., $U_T = \emptyset $) after at most $T$ trials when we use the 
 acquisition function 
 $a^{(1)}_t ({\bm  x}^\ast,{\bm w}^\ast  ) $.
\end{theorem}
Furthermore, 
the similar theorem holds if the acquisition function 
 $a^{(2)}_t (({\bm  x}^\ast,{\bm w}^\ast ) ) $ is used. 
In this study, owing to the practical performance, we modified the original RMILE to 
\begin{align*}
{\rm RMILE}_t ({\bm x}^\ast,{\bm w}^\ast ) &= \max \{   {\rm MILE}_t  ({\bm x}^\ast,{\bm w}^\ast ), \tilde{\gamma} \sigma_t  ({\bm x}^\ast,{\bm w}^\ast ) \} , \\
  {\rm MILE}_t  ({\bm x}^\ast,{\bm w}^\ast ) &=  \sum_{ ({\bm x},{\bm w} ) \in U_t \times \Omega }  \mathbb{E} _{y^\ast }  [ \1 [  l_t ({\bm x},{\bm w} |{\bm x}^\ast,{\bm w}^\ast,y^\ast   )>h ]  ]\\
&- | \{   ({\bm x},{\bm w} ) \in U_t \times \Omega  \mid l_t ({\bm x},{\bm w} ) >h-\eta \} |.
\end{align*}
Then, the following theorem holds:
\begin{theorem}\label{thm:convergence2}
Under the same setting described in 
Theorem \ref{thm:seido}, let  $\gamma >0$, $\tilde{\gamma } >0$, and $C_1 = 2/ \log (1+\sigma^{-2} ) $. 
In addition, let $T$ be the smallest positive integer satisfying the following five inequalities:
\begin{align*}
&(1) \quad \frac{\sigma^{-2} \beta^{1/2}_T C_1 \kappa_T }{T}< \frac{ \eta }{2}, \quad 
(2) \quad \frac{\sigma^{-2} C_1 \kappa_T }{T} < \frac{\eta^2 }{4} , \quad 
(3) \quad \frac{C_1 \beta_T \kappa_T }{T} < \frac{\eta^2}{4} , \\
&(4) \quad \frac{1}{2} \log \beta_T -\frac{  T \eta^2 \sigma^2}{8C_1 \kappa_T} < \log ( |\mathcal{X}| ^{-1} 2^{-|\Omega| }   \eta \gamma \tilde{\gamma} (2 \pi)^{1/2}/2 ), \\
&(5) \quad \frac{1}{2} \log \beta_T -\frac{  T \eta^2 \sigma^2}{8C_1 \kappa_T} < \log (|\mathcal{X} \times \Omega| ^{-1}    \eta  \tilde{\gamma} (2 \pi)^{1/2}/2 ).
\end{align*}
Then, Algorithm \ref{alg:1} terminates (i.e., $U_T = \emptyset $) after at most $T$ trials when we use the 
 acquisition function 
 $a^{(2)}_t ({\bm  x}^\ast,{\bm w}^\ast  ) $.
\end{theorem}
The order of the maximum information gain $\kappa_T$ is known to be sublinear under mild conditions
 \cite{SrinivasGPUCB}. 
Hence, because the order of $\beta_T $ is $O ( \log T )$, there exist positive integers satisfying the inequalities in Theorems \ref{thm:convergence1} and \ref{thm:convergence2}.
\section{Numerical experiments}\label{sec:sec5}

We confirmed the performance of the proposed method using both synthetic and real data.
Because of space limitation, we provide a part of experimental results in the main text.
All experimental results and detail parameter settings are given in the Appendix.
The input space $\mathcal{X} \times \Omega$ was defined as a set of grid points that uniformly cut the region $[L_1,U_1] \times [L_2,U_2]$ into $50 \times 50$.
In all experiments, we used the following Gaussian kernel as the kernel function:
$$
k( (x,w),(x^\prime,w^\prime ) ) = \sigma^2_f \exp (-\{(x-x^\prime)^2+(w-w^\prime)^2 \}/L).
$$
Moreover, we used $L1$-norm 
as the distance functions between distributions. 
Furthermore, we considered the following two distributions as the reference distribution $p^\ast (w)$:
\begin{description}
\item [Uniform:] $p^\ast (w) = 1/50. $
\item [Normal:]  $$ \hspace{-7mm} p^\ast (w) =\frac{a(w)}{  \sum_{w \in \Omega} a(w)  },\quad 
   a(w)=\frac{1}{\sqrt{20 \pi}} \exp (-w^2 /20 ) .$$
\end{description}
 Then, we compared the following acquisition functions:
\begin{description}
			\item [Random:] Select $(x_{t+1},w_{t+1} )$ by using random sampling. 
			\item [US:] Perform uncertainty sampling, i.e., $(x_{t+1},w_{t+1} ) = \argmax _{ (x,w) \in \mathcal{X} \times \Omega} \sigma^2_t (x,w)$.
			\item [Straddle\_f:] Perform straddle strategy \cite{bryan2006active}, i.e., 
$(x_{t+1},w_{t+1} ) = \argmax _{ (x,w) \in \mathcal{X} \times \Omega} v_t (x,w)$, where $v_t (x,w) = \min \{ u_t (x,w)-h,h-l_t (x,w) \}$. 
			\item  [Straddle\_US:] Select $x_{t+1}$ and $w_{t+1}$ by using the straddle of $F(x)$ and  $\sigma_t ({ x}_{t+1},{ w} ) $, respectively, i.e., $x_{t+1} = \argmax_{x \in \mathcal{X} } {v}^F_t (x) $ and $w_{t+1} =\argmax_{w \in \Omega} \sigma^2_t ({ x}_{t+1},{ w} ) $, where $v^F_t (x) = \min \{ u^F_t (x;\eta ) -\alpha, \alpha -l^F_t (x;\eta )    \}.$ 
			\item [Straddle\_random:] 
Replace the selection method of $w_{t+1}$ in straddle\_US with random sampling.
			\item [MILE:] Perform the original MILE strategy, i.e., $(x_{t+1},w_{t+1} ) $ was selected by using (6) in \cite{zanette2018robust}. 
			\item [Proposed1\_$0.1$:]  Perform $a^{(1)}_t ({\bm x}^\ast,{\bm w}^\ast )$ with $\gamma =0.1$.
			\item [Proposed1\_$0.01$:]  Perform $a^{(1)}_t ({\bm x}^\ast,{\bm w}^\ast )$ with $\gamma =0.01$.
			\item [Proposed2\_$0.1$:]  Perform $a^{(2)}_t ({\bm x}^\ast,{\bm w}^\ast )$ with $\gamma =0.1$.
			\item [Proposed2\_$0.01$:]   Perform $a^{(2)}_t ({\bm x}^\ast,{\bm w}^\ast )$ with $\gamma =0.01$.
		\end{description}
Here, for simplicity, we set the accuracy parameter $\eta$ to zero. 
Similarly, because of the computational cost of calculating acquisition functions, 
we replaced $\mathbb{P} (y^\ast \in R_s ) \1 [ l^{(F)}_t (x;0|x^\ast,w^\ast,c_s )>\alpha ] $ in  \eqref{eq:exp_ana}
with zero when $\mathbb{P} (y^\ast \in R_s ) $ satisfies $\mathbb{P} (y^\ast \in R_s ) < 0.005 $.
%
%
In other words, we used Lemma \ref{lem:approx_any_accuracy} with $\zeta /(|\Omega| +1)= 0.005$ to approximate \eqref{eq:exp_ana}.

\subsection{Synthetic data experiments}\label{syn_experiment}
We confirmed the performance of the proposed method using synthetic functions. 
We considered the following four functions, which are commonly used benchmark functions (the last one adds $-4000$ to the original definition):
\begin{description}
\item [Booth:] $f(x,w) = (x+2w-7)^2 + (2x+w -5 )^2$. 
\item [Matyas:] $f(x,w) = 0.26 (x^2+w^2) - 0.48 xw$. 
\item [McCormick:] $f(x,w) = \sin (x+w) + (x-w)^2 - 1.5 x + 2.5w +1$. 
\item [Styblinski-Tang:] $f(x,w)=(x^4-16x^2+5x)/2 + (w^4-16w^2+5w)/2-4000$.
\end{description}
Under this setup, we took one initial point at random and ran the algorithms until the number of iterations reached 300 (or 200), 
where the parameters used for each experiment are listed in Table \ref{tab:setting1} in the Appendix.
We performed 50 Monte Carlo simulations and obtained the average F-score as follows:
$$
\text{F-score}= \frac{2 \text{pre} \times \text{rec}}{ \text{pre}+\text{rec}}, \ \text{pre}= \frac{ | H \cap H_t    |}{|H_t|}, \ 
 \text{rec}= \frac{ | H \cap H_t    |}{|H|}.
$$
From Figures \ref{fig:exp1} and \ref{fig:exp2}, it can be confirmed that our proposed methods outperform other existing methods. 
On the other hand, in the existing methods, Straddle\_f and MILE exhibit high performance,
because the MILE acquisition function increases the expected number of $(x,w)$ satisfying $l_t (x,w) >h$. As a result,  because $\tilde{l}_t (x,w;\eta )$ and $l^{(F)}_t (x;\eta )$ become large early, the number of elements in $H_t$ also increases early. 
Similarly, because the Straddle\_f acquisition function can efficiently search for $(x,w)$ satisfying $l_t (x,w)>h$ or $u_t (x,w)<h$, the number of elements in $H_t$ also increases efficiently from the same argument as before. 
Furthermore, when comparing Proposed1 and Proposed2, one of the reasons why the latter exhibits better performance is the fact that RMILE exhibits better performance than uncertainty sampling. 
%
Other experiments, a comparison of the difference in $\gamma$ is 
described in the Appendix.

\begin{figure}[!t]
\begin{center}
 \begin{tabular}{cccc}
 \includegraphics[width=0.225\textwidth]{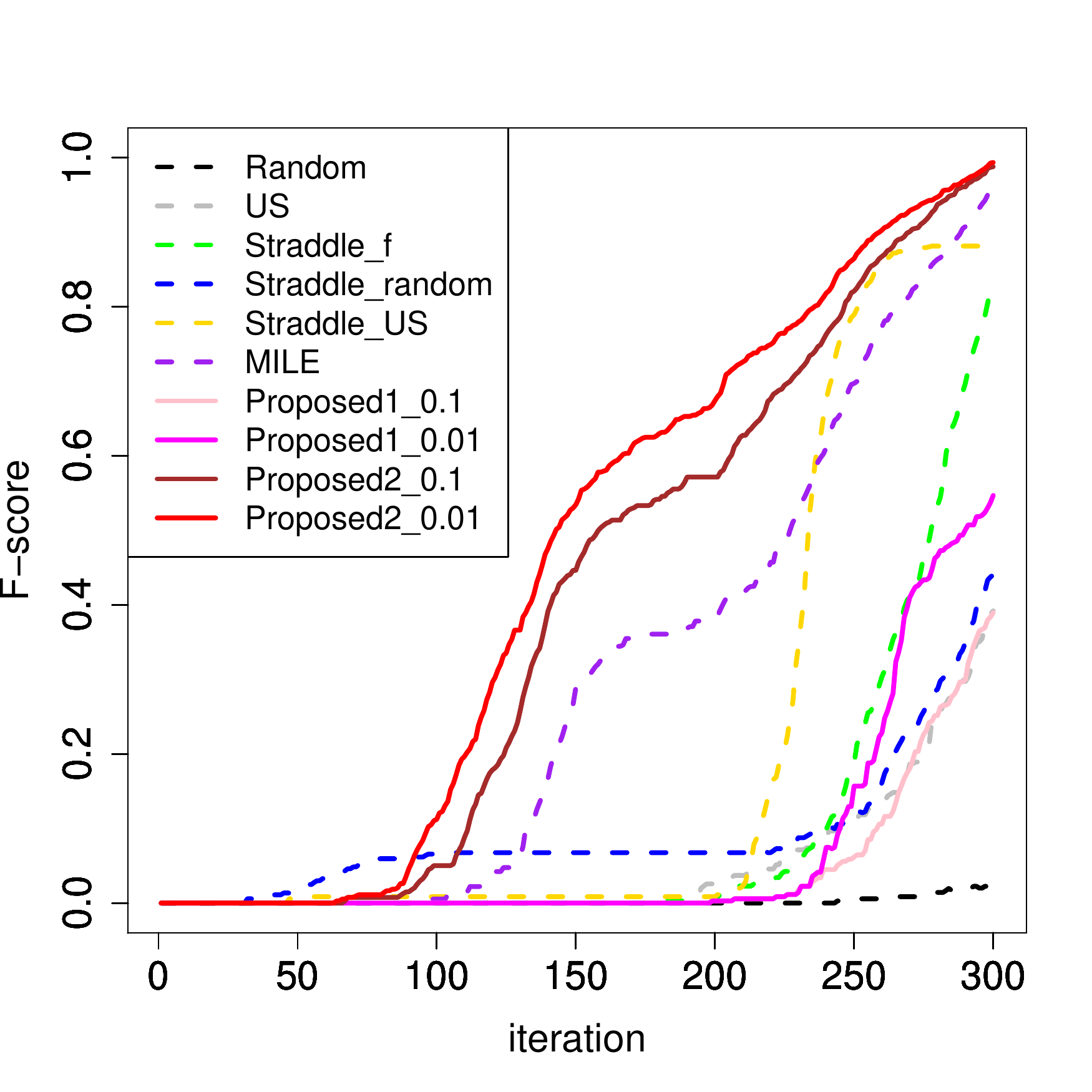} 
&
 \includegraphics[width=0.225\textwidth]{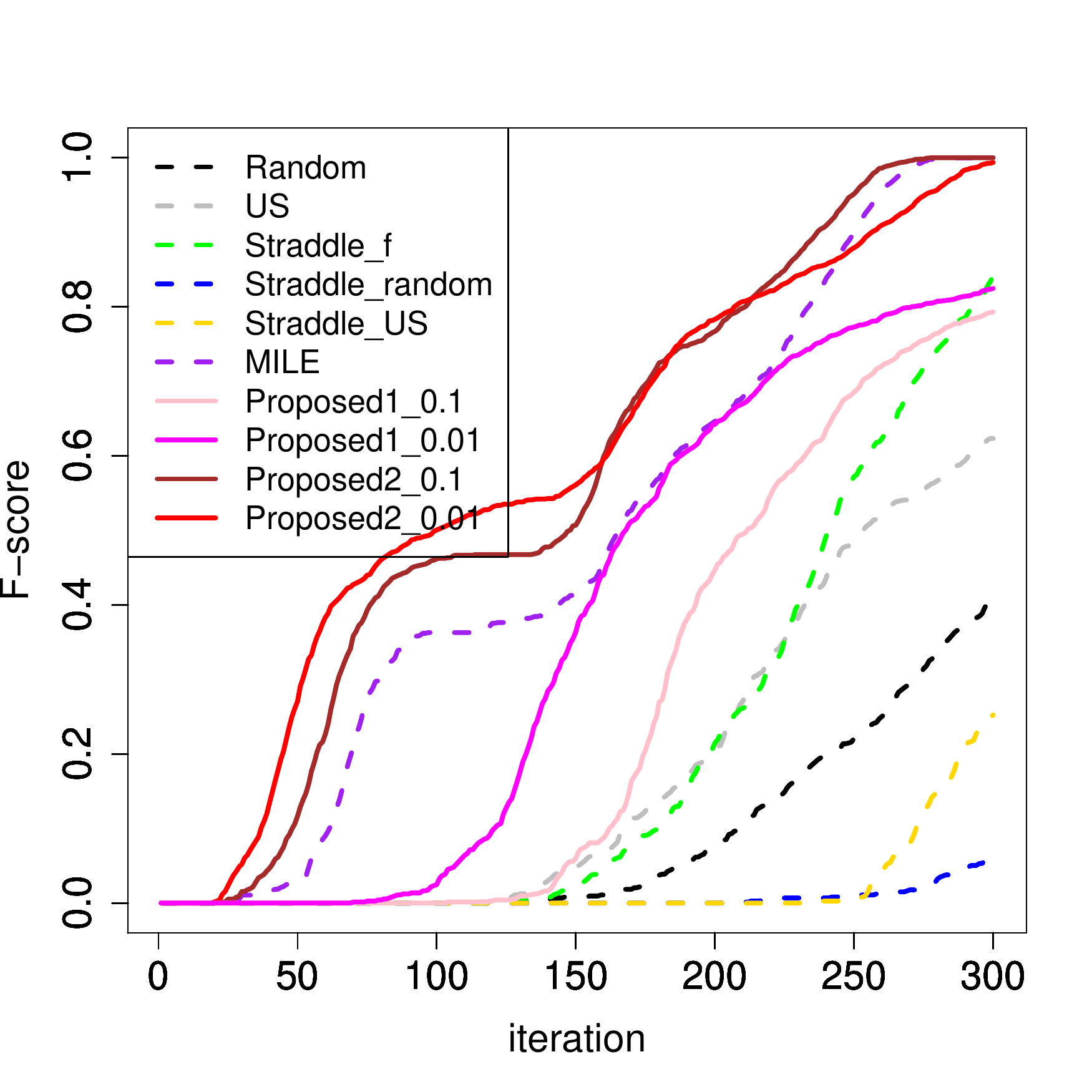} &
\includegraphics[width=0.225\textwidth]{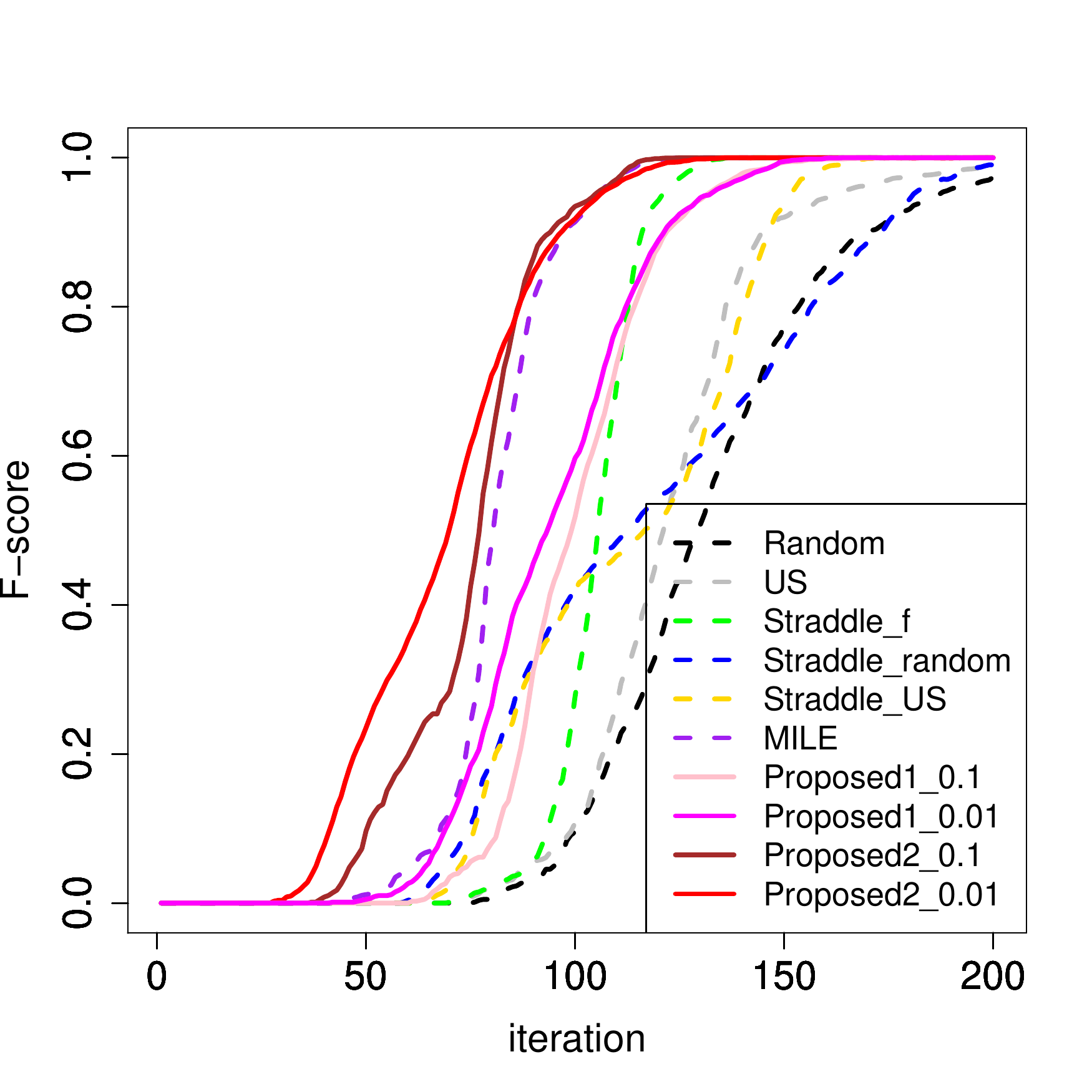} 
&
 \includegraphics[width=0.225\textwidth]{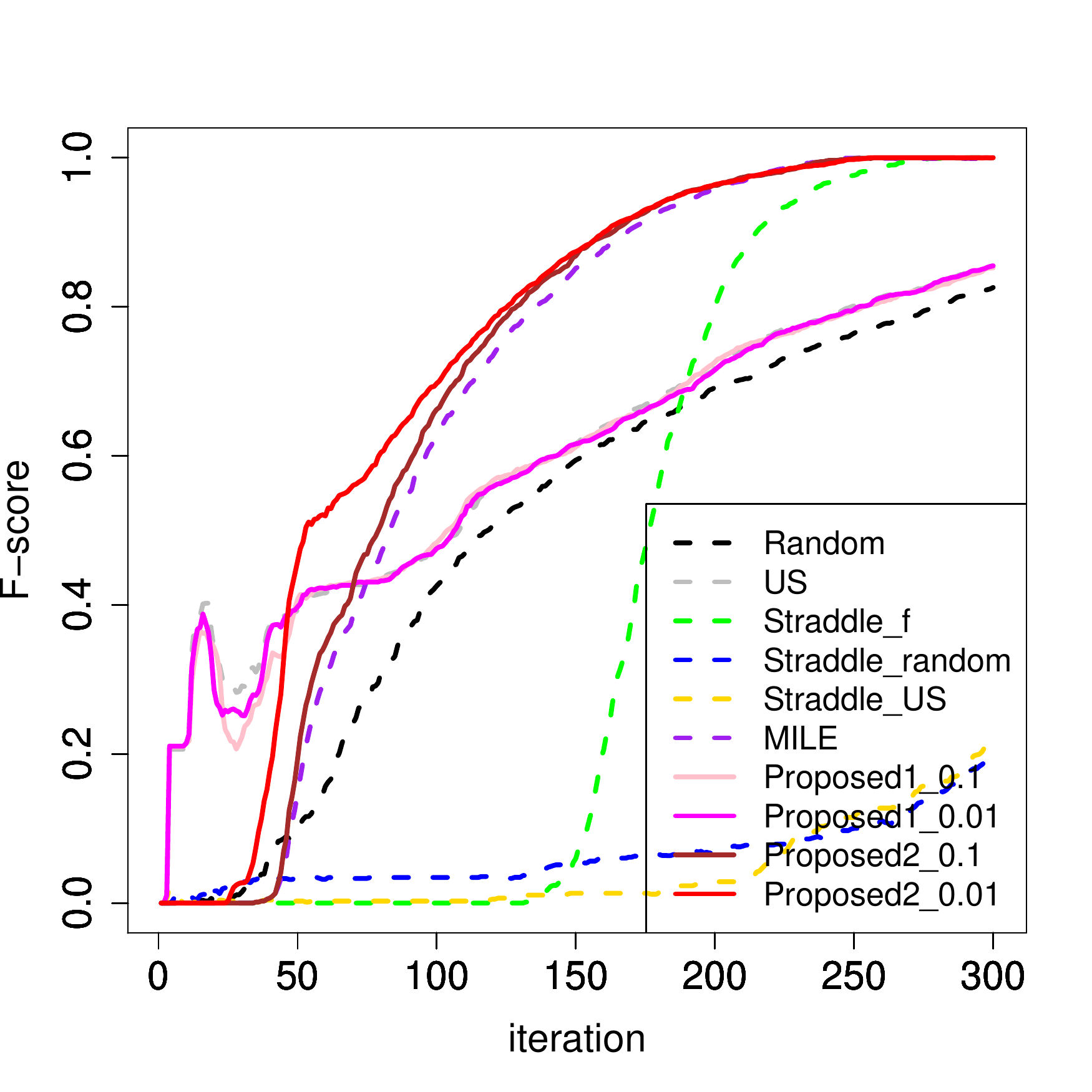} \\
Booth & Matyas &  McCormick & Styblinski-Tang 
 \end{tabular}
\end{center}
 \caption{Average F-score over 50 simulations with four benchmark functions when the distance function and reference distribution are $L1$-norm and Uniform, respectively.}
\label{fig:exp1}
\end{figure}

 	\begin{figure}[!t]
\begin{center}
 \begin{tabular}{cccc}
 \includegraphics[width=0.225\textwidth]{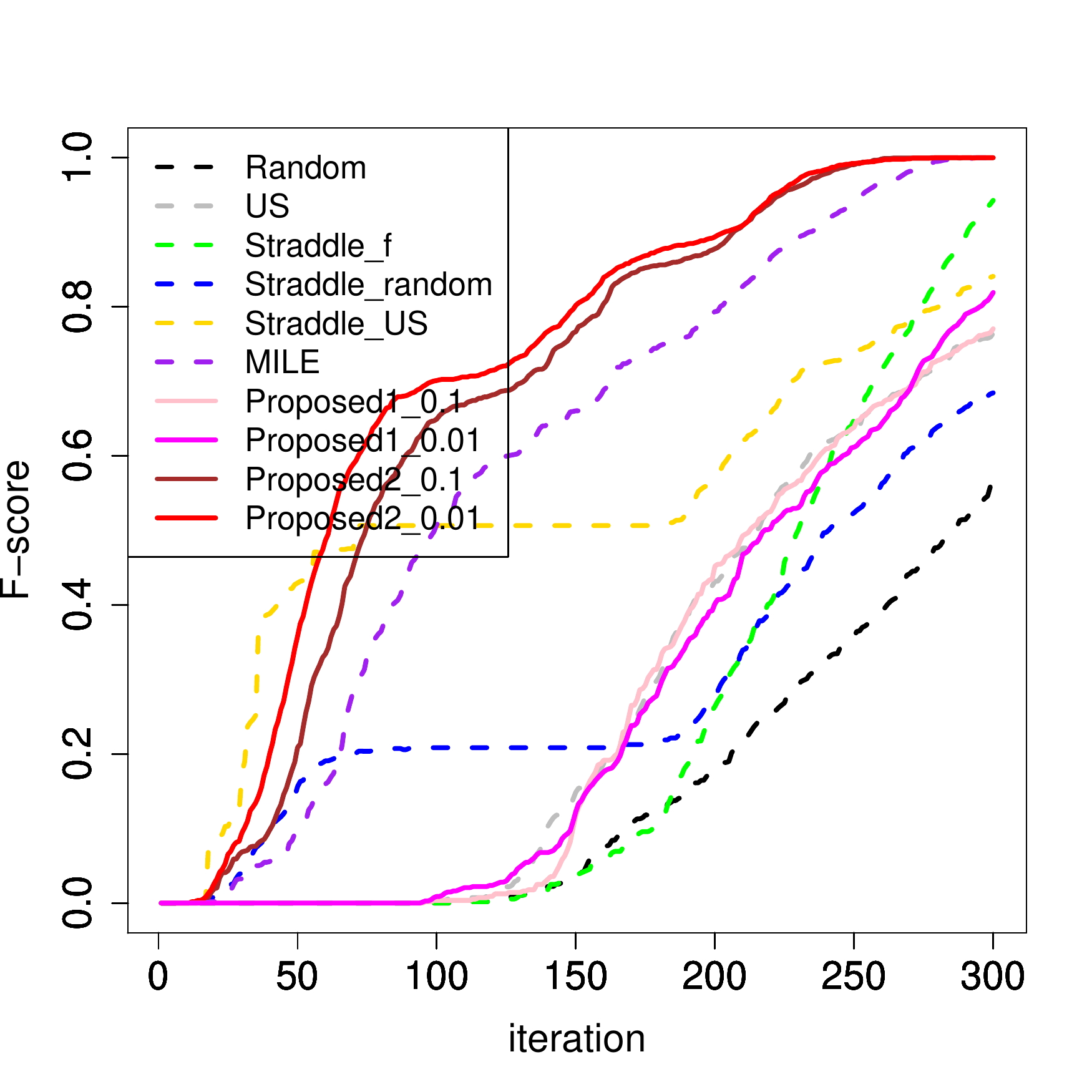} 
&
 \includegraphics[width=0.225\textwidth]{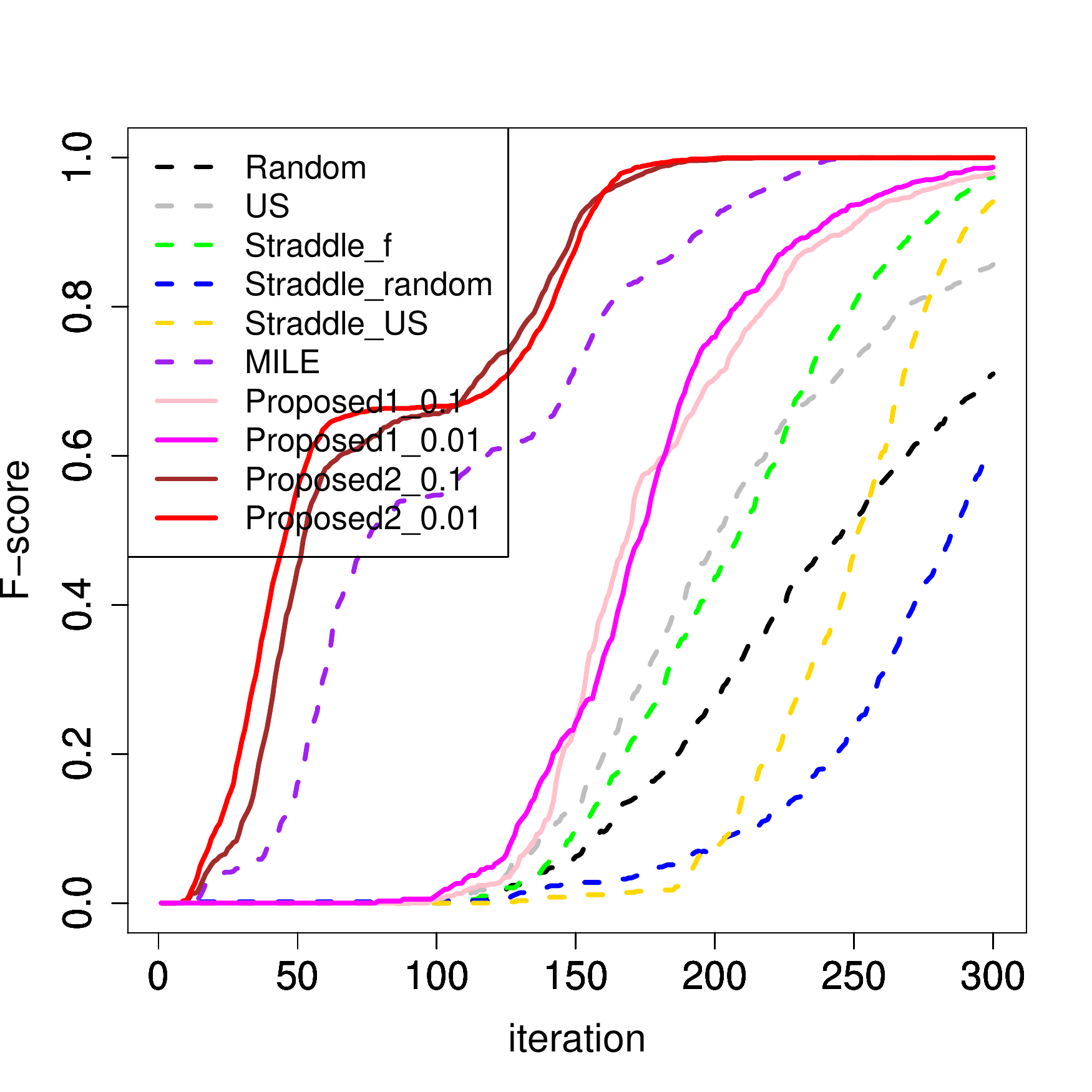} 
& 
\includegraphics[width=0.225\textwidth]{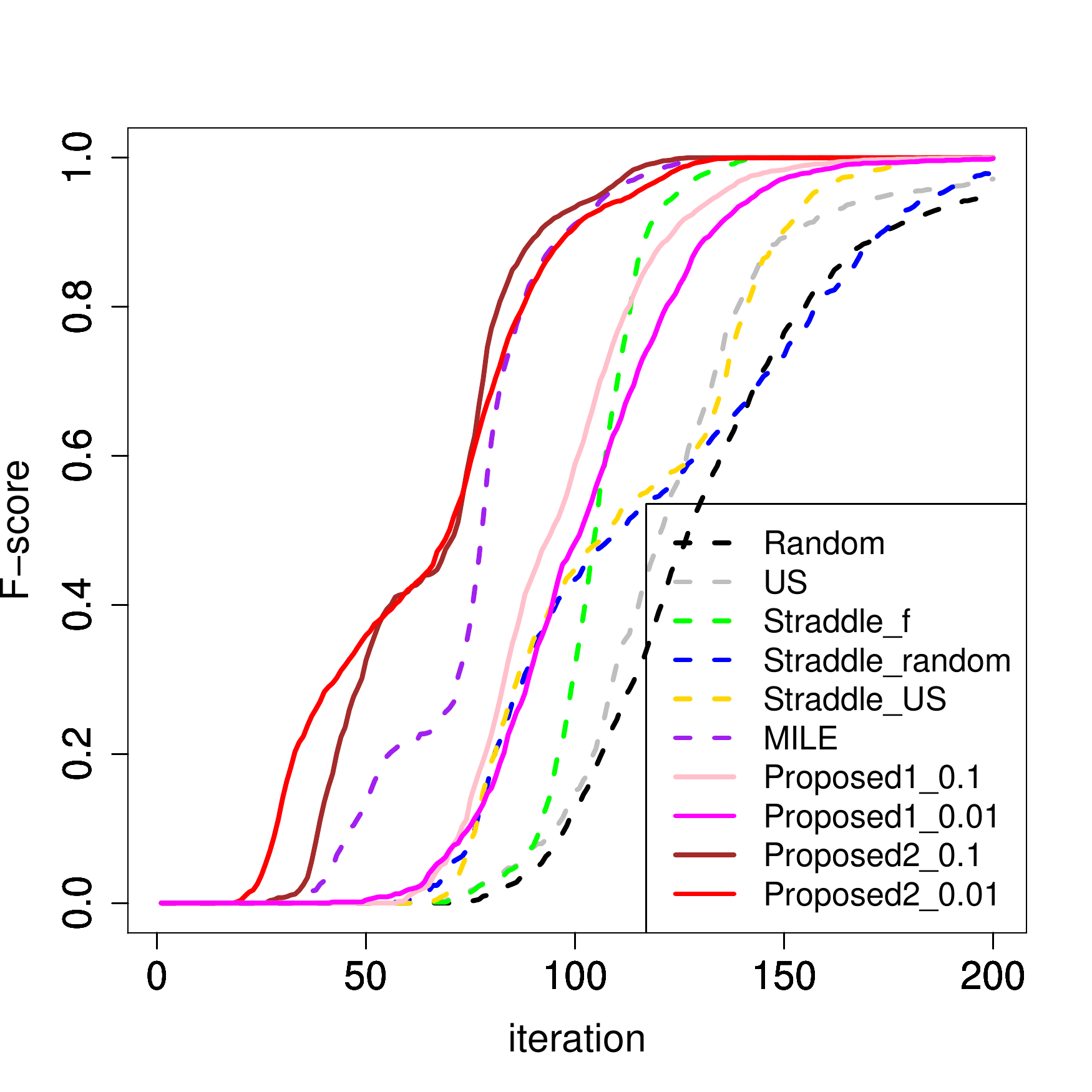} 
&
 \includegraphics[width=0.225\textwidth]{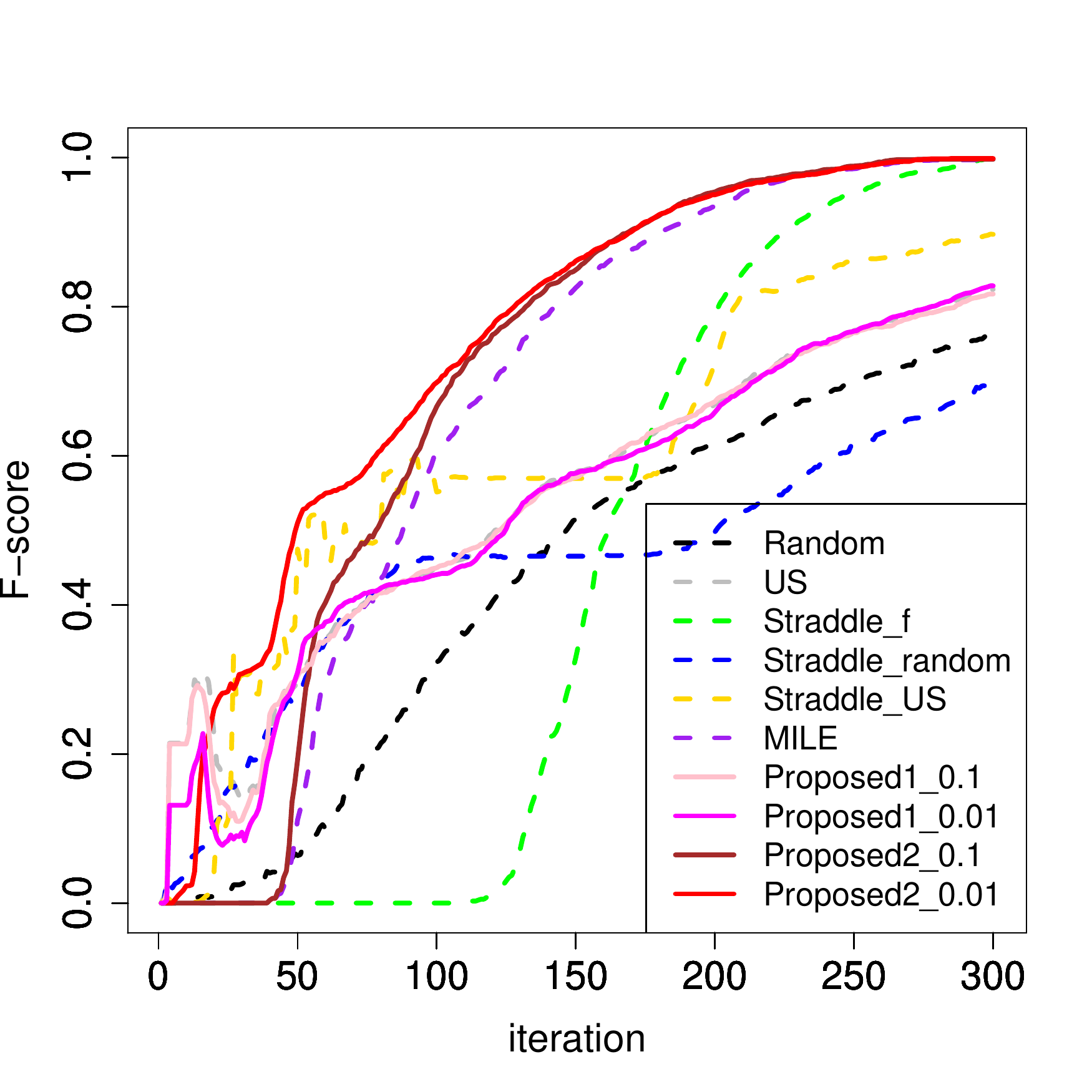} 
\\
Booth & Matyas &
 McCormick & Styblinski-Tang
 \end{tabular}
\end{center}
 \caption{Average F-score over 50 simulations with four benchmark functions when the distance function and reference distribution are $L1$-norm and Normal, respectively.}
\label{fig:exp2}
\end{figure}

\subsection{Computation time experiments}\label{subsec:comp_time}
In this section, we confirmed how much the computation time of \eqref{eq:AFteigi} can be improved by using Lemma \ref{lem:exp_cal}, \ref{lem:teigi_bound} and \ref{lem:approx_any_accuracy}.
We evaluated the computation time of \eqref{eq:AFteigi} when we performed the same experiment as in Subsection \ref{syn_experiment} using Proposed1\_$0.01$ and Proposed2\_$0.01$ for the Booth function. 
The experiments for Matyas, McCormick and Styblinski-Tang functions are described in the Appendix.
Here, as for the parameter settings, we considered only the case of $L1$-Normal in Table \ref{tab:setting1}. 
We compared the computation time of the following six methods for calculating \eqref{eq:AFteigi}:
\begin{description}
\item [Naive:] For each $({\bm x}^\ast,{\bm w}^\ast )$, we generate $M$ samples $y^\ast_1,\ldots,y^\ast_M$ from the posterior distribution of $f({\bm x}^\ast,{\bm w}^\ast )$, and approximate \eqref{eq:AFteigi} by
$$
\sum_{{\bm x} \in U_t } \frac{1}{M} \sum_{m=1}^M \1 [l^{(F)}_t ({\bm x};0|{\bm x}^\ast,{\bm w}^\ast,y^\ast_m)>\alpha],
$$
where we set $M=1000$.
\item [L1:] Compute \eqref{eq:AFteigi} using Lemma \ref{lem:exp_cal}.
\item [L2:] Compute \eqref{eq:AFteigi} using Lemma \ref{lem:exp_cal} and \ref{lem:teigi_bound}.
\item [L3 $(10^{-4})$:] Compute \eqref{eq:AFteigi} using Lemma \ref{lem:exp_cal}, \ref{lem:teigi_bound} and \ref{lem:approx_any_accuracy} with $\zeta =  (|\Omega|+1) 10^{-4}$. 
\item [L3 $(10^{-8})$:] Compute \eqref{eq:AFteigi} using Lemma \ref{lem:exp_cal}, \ref{lem:teigi_bound} and \ref{lem:approx_any_accuracy} with $\zeta =  (|\Omega|+1) 10^{-8}$. 
\item [L3 $(10^{-12})$:] Compute \eqref{eq:AFteigi} using Lemma \ref{lem:exp_cal}, \ref{lem:teigi_bound} and \ref{lem:approx_any_accuracy} with $\zeta =  (|\Omega|+1) 10^{-12}$. 
\end{description}
Under this setup, we took one initial point at random and ran the algorithms until the number of iterations reached to 300. 
Furthermore, for each trial $t$, we evaluated the computation time to calculate \eqref{eq:AFteigi} for all candidate points $({\bm x}^\ast,{\bm w}^\ast ) \in \mathcal{X} \times \Omega $, and calculated the average computation time over 300 trials.
From Table \ref{tab:time_Booth}, 
it can be confirmed that the computation time is improved as the proposed computational techniques are used. 
Moreover, comparing L3 $(10^{-4})$, L3 $(10^{-8})$ and L3 $(10^{-12})$, it can be confirmed that the computation time becomes shorter when a large $\zeta$ is used.
However, it can be seen that the computation time of L3 $(10^{-12})$ is still very small compared to the computation time of
 Naive, L1 and L2.
Therefore, from $|\Omega| =50$ and Lemma \ref{lem:approx_any_accuracy}, it implies that by using proposed computational techniques, we can improve the computation time significantly even if the error from the true $a_t ({\bm x}^\ast, {\bm w}^\ast )$ is kept to a very small value such as $51 \times 10^{-12}=5.1 \times 10^{-11}$.

\begin{table*}[!t]
  \begin{center}
    \caption{Computation time (second) for the Booth function setting}
\scalebox{0.85}{
    \begin{tabular}{c||c|c|c|c|c|c} \hline \hline
       & Naive & L1 & L2 & L3 $(10^{-4})$ & L3 $(10^{-8})$ & L3 $(10^{-12})$ \\ \hline 
   Proposed1\_$0.01$   & $138505.60 \pm 13334.87$  & $7621.59 \pm 1166.23$  &   $2370.02 \pm 586.94$    & $71.16 \pm 25.33$& $80.55 \pm 31.37$ & $86.73 \pm 35.34$   \\ \hline 
   Proposed2\_$0.01$   &  $106306.10 \pm 12331.01$ & $5835.06 \pm 1028.99$  &   $2608.30 \pm 976.06$    & $63.14 \pm 10.29$& $72.53 \pm 13.99$ & $78.74 \pm 16.29$   \\ \hline \hline
    \end{tabular}
}
    \label{tab:time_Booth}
  \end{center}
\end{table*}

\subsection{Real data experiments}\label{subsec:real}
We compared our proposed method with other existing methods by using the 
infection control problem \cite{kermack1927contribution}. 
We considered a simulation-based decision-making problem for an epidemic, which aims to determine an acceptable infection rate $x$ under an uncertain recovery rate $w$ with as few simulations as possible. 
The motivation for this simulation was to evaluate the tradeoff between economic risk and a controllable infection rate.
For example, if the infection rate $x$ is minimized by shutting down all economic activities, the economic risk will become extremely high. 
In contrast, if nothing is done, the infection rate will remain high, resulting in the spread of the disease, and economic risk will still be high. 
Therefore, we considered finding a target infection rate that can achieve an acceptable economic risk threshold $h$ with a probability of at least $\alpha$.
In this experiment, to simulate epidemic behavior, we used the SIR model \cite{kermack1927contribution}.
The model computes the evolution of the number of infected people by using an infection rate $x$ and recovery rate $w$.
In our experiment, we considered the infection rate as the design variable $x$ and the recovery rate as the environmental variable $w$ following an unknown distribution.
In addition, we regarded economic risk as a black-box function $f(x,w)$.
Note that similar numerical experiments were performed in \cite{iwazaki2020bayesian} under the setting where the distribution of $w$, $p^\dagger (w)$, is known. 
Furthermore, we rescaled the ranges of $x$ and $w$ in the interval $[-1,1]$. 
The input space $\mathcal{X} \times \Omega$ is defined as a set of grid points that uniformly cut the region $[-1,1] \times [-1,1]$ into $50 \times 50$. 
We used the following economic risk function $f(x,w)$:
$
f(x,w) = n_{\text{infected}} (x,w) - 150x, 
$
where $ n_{\text{infected}} (x,w)$ is the maximum number of infected people in a given period of time, calculated using the SIR model.
Note that this risk function was also used by \cite{iwazaki2020bayesian}, and in this experiment, we used the same function they used in their experiment.
Under this setup, we took one initial point at random and ran the algorithms until the number of iterations reached 100.
 From 50 Monte Carlo simulations, we calculated average F-scores, 
where we used the following parameters for all problem settings:
\begin{align*}
h=135,\ \alpha =0.9,\  \sigma^2=0.025,\ \sigma^2_f = 250^2,\  L=0.5, \ 
  \beta^{1/2}_t =4,\   \epsilon = 0.05.
\end{align*}
In this experiment, we used the following modified reference function as Normal:
$$p^\ast (w) =\frac{a(w)}{  \sum_{w \in \Omega} a(w)  },\quad a(w)=\frac{1}{\sqrt{0.1 \pi}} \exp (-w^2 /0.1 ) .$$

From Figure \ref{fig:exp5}, it can be confirmed that Proposed2 and MILE performed better than the others.

\begin{figure}[!t]
\begin{center}
 \begin{tabular}{cc}
\includegraphics[width=0.45\textwidth]{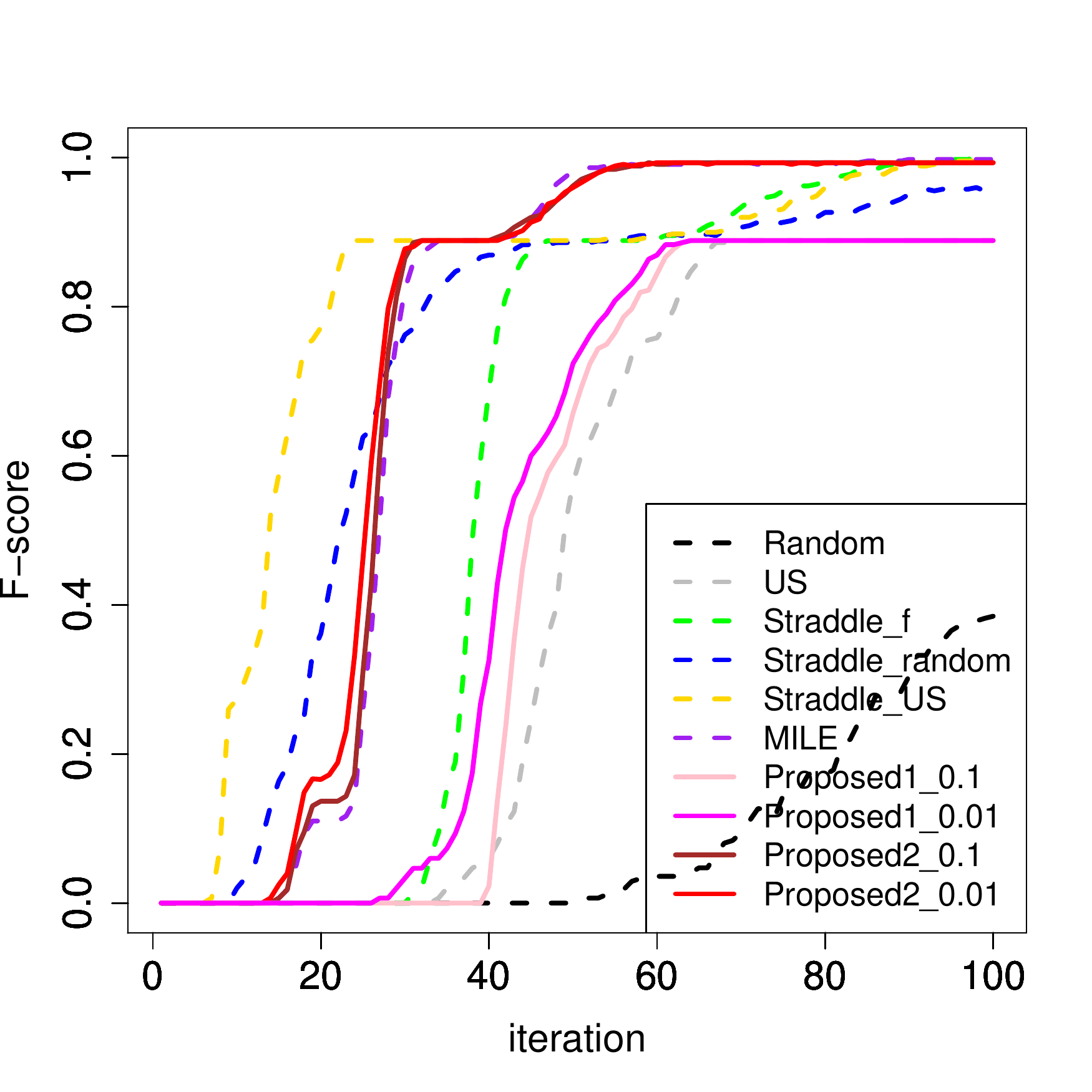} 
&
\includegraphics[width=0.45\textwidth]{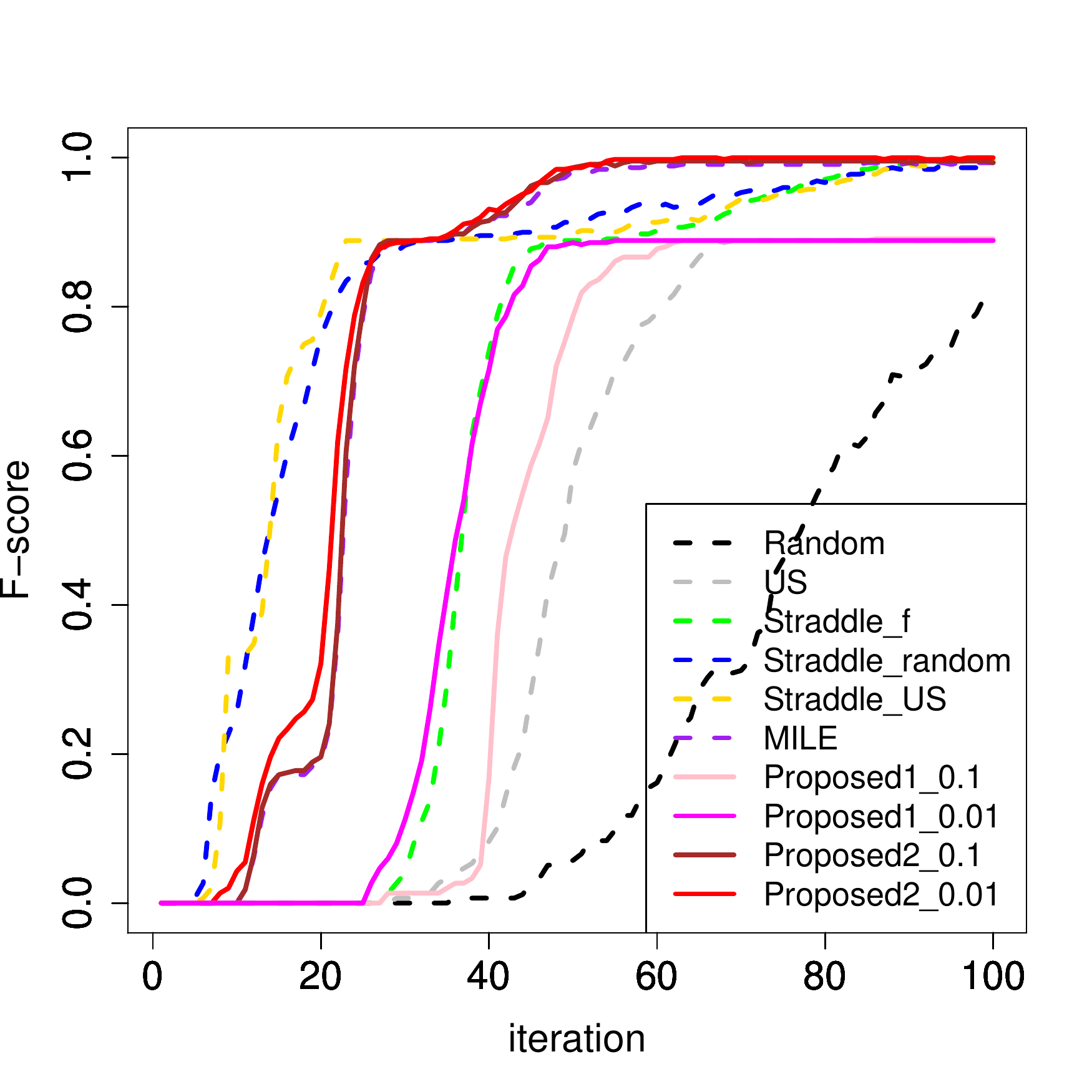} 
\\
$L1$-Uniform & $L1$-Normal 
 \end{tabular}
\end{center}
 \caption{Average F-score over 50 simulations in the infection control problem with two different settings.}
\label{fig:exp5}
\end{figure}

\section{Conclusion}
We proposed active learning methods for identifying the reliable set of distributionally robust probability threshold robustness (DRPTR) measure under uncertain environmental variables. 
We showed that our proposed methods satisfy theoretical guarantees about convergence and accuracy, and outperform existing methods in numerical experiments.

\section*{Acknowledgement}
This work was partially supported by MEXT KAKENHI (20H00601, 16H06538), JST CREST (JPMJCR1502), and 
RIKEN Center for Advanced Intelligence Project.

\bibliography{myref}
\bibliographystyle{plain}

\newpage
\section*{Appendix}

\setcounter{section}{0}
\renewcommand{\thesection}{\Alph{section}}
\renewcommand{\thesubsection}{\thesection.\arabic{subsection}}

\section{Proofs}
\subsection{Proof of Theorem \ref{thm:seido}}
In this section, we prove Theorem \ref{thm:seido}. 
First, we show two lemmas. 
\begin{lemma}\label{lem:unif_bound}
Let $\delta \in (0,1)$, and define $\beta_t =2 \log ( | \mathcal{X}\times \Omega | \pi^2 t^2 /(6 \delta ) )$. 
Then, with a probability of at least $1- \delta$, the following inequality holds:
\begin{align}
| f({\bm x},{\bm w} ) - \mu_{t-1} ({\bm x} ,{\bm w} ) |  \leq  \beta^{1/2}_t \sigma_{t-1} ({\bm x},{\bm w} ) ,
\quad \f ({\bm x},{\bm w} ) \in \mathcal{X} \times \Omega, \f t \geq 1. \nonumber
\end{align}
\end{lemma} 
\begin{proof}
By replacing $D$ and $\pi_t $ in Lemma 5.1 of \cite{SrinivasGPUCB} with $\mathcal{X} \times \Omega$ and $\pi^2 t^2 /6$, respectively, 
we have Lemma \ref{lem:unif_bound}.
\end{proof}

\begin{lemma}\label{lem:nokori}
Let $\delta \in (0,1)$, $\xi >0$ and 
$
\eta = \min \left \{
\frac{\xi \sigma_{0,min} }{2}, \frac{\xi ^2 \delta \sigma_{0,min}}{8 |\mathcal{X}\times \Omega | }
\right \}.
$ 
Then, with a probability of at least $1-\delta/2$, the following holds for any ${\bm x} \in \mathcal{X} $ and $p({\bm w} ) \in \mathcal{A}$:
\begin{align}
\tilde{F}_{\eta,p} ({\bm x} ) \equiv \sum_{ {\bm w} \in \Omega }   \1 [h \geq f({\bm x},{\bm w} ) > h-\eta ] p({\bm w}) < \xi .\nonumber
\end{align}
\end{lemma}
\begin{proof}
From Chebyshev's inequality, for any $\nu >0$ and $({\bm x},{\bm w} ) \in \mathcal{X} \times \Omega$, the following inequality holds: 
$$
\mathbb{P} (   |  g_\eta ({\bm x},{\bm w} ) - \mu^{(g_\eta)} ({\bm x},{\bm w} ) | \geq \nu ) \leq \frac{ \mathbb{V}[ g_\eta ({\bm x},{\bm w} ) ]    }{\nu ^2},
$$
where   $g_\eta ({\bm x},{\bm w} ) =    \1 [h \geq f({\bm x},{\bm w} ) > h-\eta ]$     and $\mu^{(g_\eta)} ({\bm x},{\bm w} )  = \mathbb{E} [ g_\eta ({\bm x},{\bm w} )]$. 
Hence, by replacing $\nu$ with $   ( \delta /( 2 | \mathcal{X} \times \Omega | ) )^{-1/2} ( \mathbb{V}[ g_\eta ({\bm x},{\bm w} ) ]  )^{1/2}$, with a probability of at least $1-\delta /2$, the following holds for any $({\bm x},{\bm w} ) \in \mathcal{X} \times \Omega$:
$$
 |  g_\eta ({\bm x},{\bm w} ) - \mu^{(g_\eta)} ({\bm x},{\bm w} ) | < \frac{ \sqrt{\mathbb{V}[ g_\eta ({\bm x},{\bm w} ) ] }  }{\sqrt{ \delta /( 2 | \mathcal{X} \times \Omega | ) }  }.
$$
This implies that 
\begin{align}
   g_\eta ({\bm x},{\bm w} )   <  \mu^{(g_\eta)} ({\bm x},{\bm w} )+\frac{ \sqrt{\mathbb{V}[ g_\eta ({\bm x},{\bm w} ) ] }  }{\sqrt{ \delta /( 2 | \mathcal{X} \times \Omega | ) }  }. \label{eq:ap1}
\end{align}
Moreover, noting that $ g_\eta ({\bm x},{\bm w} )$ follows Bernoulli distribution, we get
\begin{align}
\mathbb{V}[ g_\eta ({\bm x},{\bm w} ) ] = \mathbb{E}[ g_\eta ({\bm x},{\bm w} ) ] (1-\mathbb{E}[ g_\eta ({\bm x},{\bm w} ) ]) 
\leq \mathbb{E}[ g_\eta ({\bm x},{\bm w} ) ]  = \mu^{(g_\eta)} ({\bm x},{\bm w} ) . \label{eq:ap2}
\end{align}
In addition, $\mu^{(g_\eta)} ({\bm x},{\bm w} )$ can be expressed as 
$$
\mu^{(g_\eta)} ({\bm x},{\bm w} ) = \Phi  \left (  \frac{ h }{ \sigma_0 ({\bm x},{\bm w}) }   \right ) -\Phi  \left (  \frac{ h -\eta}{ \sigma_0 ({\bm x},{\bm w}) }   \right ).
$$
Furthermore, by using Taylor's expansion, for any $a <b$ it holds that 
$$
\Phi (b) = \Phi (a) +  \phi (c) (b-a)   \leq \Phi (a) + \phi(0) (b-a) \leq \Phi (a) + (b-a),
$$
where $c \in (a,b)$. Thus, we obtain 
\begin{align}
\mu^{(g_\eta)} ({\bm x},{\bm w} ) \leq \frac{ \eta}{\sigma_0 ({\bm x},{\bm w})} \leq \frac{ \eta}{\sigma_{0,min}}. \label{eq:ap3}
\end{align}
Thus, by substituting \eqref{eq:ap2} and \eqref{eq:ap3} into \eqref{eq:ap1}, we have 
$$
 g_\eta ({\bm x},{\bm w} )   <  \frac{ \eta}{\sigma_{0,min}} +    \sqrt{   \frac{ 2 \eta |\mathcal{X} \times \Omega|  }{ \delta \sigma_{0,min}    }  }.
$$
Hence, from the definition of $\eta $, we get 
$$
 g_\eta ({\bm x},{\bm w} )   < \frac{\xi}{2} + \sqrt{  \frac{\xi^2}{4} } = \xi .
$$
Therefore, for any $p({\bm w} ) \in \mathcal{A}$, the following holds: 
$$
\tilde{F}_{\eta,p} ({\bm x} ) =  \sum_{ {\bm w} \in \Omega }   g_\eta ({\bm x},{\bm w} ) p({\bm w} ) < \sum_{ {\bm w} \in \Omega } \xi  p({\bm w} ) = \xi.
$$
\end{proof}
By using Lemma \ref{lem:unif_bound} and \ref{lem:nokori}, we prove Theorem  \ref{thm:seido}.
\begin{proof}
Let $\delta \in (0,1)$ and $\beta_t =2 \log ( | \mathcal{X}\times \Omega | \pi^2 t^2 /(3 \delta ) )$. 
Then, from Lemma \ref{lem:unif_bound},  with a probability of at least $1-\delta/2 $ the following holds: 
\begin{align}
l_ t ({\bm x},{\bm w} ) \leq f({\bm x} ,{\bm w} )   \leq u_t ({\bm x},{\bm w} ), \quad \f  ({\bm x},{\bm w} ) \in \mathcal{X} \times \Omega, \f t \geq 1.  \label{eq:apt1}
\end{align}
Thus, from the definition of $\tilde{Q}_t ({\bm x},{\bm w};\eta )$, it holds that 
$$
\1 [ f({\bm x},{\bm w} ) >h ]  \leq \tilde{u}_t ({\bm x},{\bm w} ;\eta ).
$$
This implies that 
$$
F({\bm x} ) =  \inf _{p({\bm w} ) \in \mathcal{A} }   \sum_{ {\bm w} \in \Omega }   \1[f({\bm x},{\bm w} ) >h] p({\bm w}) 
\leq  \inf _{p({\bm w} ) \in \mathcal{A} }   \sum_{ {\bm w} \in \Omega } \tilde{u}_t ({\bm x},{\bm w} ;\eta ) p({\bm w}) =u^{(F)}_t ({\bm x};\eta ).
$$
Therefore, noting that the definition of $L_t $, we have 
\begin{align}
{\bm x} \in L_t  \Rightarrow    F({\bm x} )  \leq u^{(F)}_t ({\bm x};\eta ) \leq \alpha   .   \label{eq:a6}
\end{align}

On the other hand, for any ${\bm x} \in \mathcal{X}$ and $p({\bm w} ) \in \mathcal{A}$, it holds that 
$$
\sum_{ {\bm w} \in \Omega }   \1[f({\bm x},{\bm w} ) >h] p({\bm w}) + \tilde{F}_{\eta,p}  ({\bm x} ) = 
\sum_{ {\bm w} \in \Omega }   \1[f({\bm x},{\bm w} ) >h-\eta ] p({\bm w}).  
$$
Moreover, from Lemma \ref{lem:nokori}, with a probability of at least $1-\delta /2$, the following holds: 
\begin{align}
\sum_{ {\bm w} \in \Omega }   \1[f({\bm x},{\bm w} ) >h] p({\bm w}) + \xi > 
\sum_{ {\bm w} \in \Omega }   \1[f({\bm x},{\bm w} ) >h-\eta ] p({\bm w}).  \label{eq:apt2}
\end{align}
Thus, we get the following inequality:
\begin{align}
 \inf _{p({\bm w} ) \in \mathcal{A} } \left ( \sum_{ {\bm w} \in \Omega }   \1[f({\bm x},{\bm w} ) >h] p({\bm w})  +\xi \right ) 
=  F({\bm x} ) + \xi      > 
 \inf _{p({\bm w} ) \in \mathcal{A} }  \sum_{ {\bm w} \in \Omega }   \1[f({\bm x},{\bm w} ) >h-\eta] p({\bm w})  . \label{eq:ap7}
\end{align}
Furthermore, from the definition of  $\tilde{Q}_t ({\bm x},{\bm w};\eta )$, the following inequality holds: 
$$
\1 [ f({\bm x},{\bm w} ) >h -\eta]  \geq \tilde{l}_t ({\bm x},{\bm w} ;\eta ).
$$
Therefore, we have 
\begin{align}
 \inf _{p({\bm w} ) \in \mathcal{A} }  \sum_{ {\bm w} \in \Omega }   \1[f({\bm x},{\bm w} ) >h-\eta] p({\bm w}) \geq 
\inf _{p({\bm w} ) \in \mathcal{A} }   \sum_{ {\bm w} \in \Omega } \tilde{l}_t ({\bm x},{\bm w} ;\eta ) p({\bm w}) =l^{(F)}_t ({\bm x};\eta ).
 \label{eq:ap8}
\end{align}
Hence, by combining \eqref{eq:ap7} and \eqref{eq:ap8}, we obtain 
$$
l^{(F)}_t ({\bm x};\eta ) < F({\bm x} ) + \xi.
$$
Thus, from the definition of $H_t$, it holds that 
\begin{align}
{\bm x} \in H_t \Rightarrow \alpha <   F ({\bm x} ) + \xi   \Rightarrow \alpha - \xi < F({\bm x} ). \label{eq:ap9}
\end{align}
Hence, from \eqref{eq:a6}, \eqref{eq:ap9} and   the definition of $e_\alpha ({\bm x} ) $, the following inequality holds: 
\begin{align}
\max_{ {\bm x} \in \mathcal{X} }  e_\alpha ({\bm x} )    \leq \xi. \nonumber 
\end{align}
Finally, since both \eqref{eq:apt1} and \eqref{eq:apt2} hold with a probability of at least $1-\delta$, the following holds for any $ t \geq 1$:
$$
\mathbb{P}   \left ( \max_{ {\bm x} \in \mathcal{X} }  e_\alpha ({\bm x} )    \leq \xi \right ) \geq 1- \delta.
$$
\end{proof}


\subsection{Proof of Theorem \ref{thm:convergence1} and \ref{thm:convergence2}}
In this section, we prove Theorem \ref{thm:convergence1} and \ref{thm:convergence2}. 
First, we show related lemmas. 
\begin{lemma}\label{lem:eta_end}
Let $\eta >0$ and $\beta_t >0$. Suppose that the following holds for some $T \geq 1$: 
\begin{align}
2 \beta^{1/2}_T \sigma_{T-1} ({\bm x},{\bm w} )   < \eta, \quad \f ({\bm x},{\bm w} ) \in \mathcal{X} \times \Omega. \label{eq:app_b1}
\end{align}
Then, Algorithm \ref{alg:1} terminates after at most $T$ iterations. 
\end{lemma}
\begin{proof}
From the definition of $\tilde{Q}_t ({\bm x},{\bm w};\eta )$, if $l_T ({\bm x},{\bm w} ) >h-\eta$, 
then $\tilde{l}_T ({\bm x},{\bm w};\eta ) = \tilde{u}_T ({\bm x},{\bm w};\eta ) =1$. 
On the other hand, noting that $u_T ({\bm x},{\bm w} ) -l_T ({\bm x},{\bm w} ) =2 \beta^{1/2}_T \sigma_{T-1} ({\bm x},{\bm w} )  $ 
and \eqref{eq:app_b1}, if 
$l_T ({\bm x},{\bm w} ) \leq h-\eta$, then $u_T ({\bm x},{\bm w} ) \leq h$. 
This implies that $\tilde{l}_T ({\bm x},{\bm w};\eta ) = \tilde{u}_T ({\bm x},{\bm w};\eta ) =0$. 
Thus, under  \eqref{eq:app_b1}, the following holds for any $({\bm x},{\bm w} ) \in \mathcal{X} \times \Omega$:
$$
\tilde{l}_T ({\bm x},{\bm w};\eta ) = \tilde{u}_T ({\bm x},{\bm w};\eta ). 
$$
Hence, from the definitions of $l^{(F)}_t ({\bm x},{\bm w};\eta)$ and  $u^{(F)}_t ({\bm x},{\bm w};\eta)$, we have $l^{(F)}_t ({\bm x},{\bm w};\eta) = u^{(F)}_t ({\bm x},{\bm w};\eta)$. 
Therefore, for any ${\bm x} \in \mathcal{X}$, ${\bm x}$ satisfies ${\bm x} \in H_T $ or ${\bm x} \in L_T$, i.e., $U_T = \emptyset$.
\end{proof}
\begin{lemma}\label{lem:PTRMILE_bound}
Let $\eta >0$ and $\beta _t >0$. Suppose that the following inequalities hold for some $({\bm x}^\ast,{\bm w}^\ast ) \in \mathcal{X} \times \Omega$:
\begin{align}
\sigma^{-2} \sigma^2_{t-1} ({\bm x}^\ast,{\bm w}^\ast ) \beta^{1/2}_t &< \frac{\eta }{2}, \label{eq:app_b2} \\
\sigma^{-2} \sigma^2_{t-1} ({\bm x}^\ast,{\bm w}^\ast )  &< \eta^2/4. \label{eq:app_b3}
\end{align}
Then, \eqref{eq:AFteigi} can be bounded as
\begin{align}
a_{t-1} ({\bm x}^\ast,{\bm w}^\ast ) \leq | \mathcal{X} |  2^{|\Omega|} \frac{1}{\sqrt{2 \pi} } \exp \left (
-\frac{\sigma^2 \eta^2}{8\sigma^2_{t-1} ({\bm x}^\ast,{\bm w}^\ast)}
\right ). \nonumber 
\end{align}
\end{lemma}
\begin{proof}
First, we define the set $\mathcal{B}$ as 
$$
\mathcal{B} = \left \{   {\bm b} = (b_1,\ldots, b_{|\Omega|} ) \in \{0,1\}^{|\Omega|} \middle  |  \inf _{p ({\bm w} ) \in \mathcal{A} } \sum_{j=1}^{|\Omega|} p({\bm w}_j )  b_j  > \alpha \right \}.
$$
Moreover, for each ${\bm b} \in \mathcal{B}$, let ${ N}^{({\bm b} ) } $ be a subset of $\{1,\ldots, |\Omega| \}$ satisfying 
$$
\f s \in { N}^{({\bm b} ) } , \ b_s =1.
$$
 Then, the following holds for any ${\bm x} \in U_t$:
\begin{align}
& \mathbb{E}_{y^\ast}  [ \1 [l^{(F)}_t ({\bm x} ;0 |{\bm x}^\ast,{\bm w}^\ast, y^\ast) > \alpha ] ]  \nonumber \\
&=\mathbb{P}_{y^\ast}  [ (\1 [l_t ({\bm x} ,{\bm w}_1 |{\bm x}^\ast,{\bm w}^\ast, y^\ast) > h ] , \ldots , \1 [l_t ({\bm x} ,{\bm w}_{|\Omega|} |{\bm x}^\ast,{\bm w}^\ast, y^\ast) > h ])^\top \in \mathcal{B} ] \nonumber \\
&=\sum_{ {\bm b} \in \mathcal{B}} \mathbb{P}_{y^\ast}  [ \1 [l_t ({\bm x} ,{\bm w}_1 |{\bm x}^\ast,{\bm w}^\ast, y^\ast) > h ]=b_1 , \ldots , \1 [l_t ({\bm x} ,{\bm w}_{|\Omega|} |{\bm x}^\ast,{\bm w}^\ast, y^\ast) > h ] = b_{|\Omega|} ] \nonumber \\
&\leq 
\sum_{ {\bm b} \in \mathcal{B}} \mathbb{P}_{y^\ast}  [ \f s \in N^{({\bm b}) }, \1 [l_t ({\bm x} ,{\bm w}_s |{\bm x}^\ast,{\bm w}^\ast, y^\ast) > h ]=b_s  ], \label{eq:app_b5}
\end{align}
where $l_t ({\bm x},{\bm w}_j | {\bm x}^\ast,{\bm w}^\ast,y^\ast )$ is the lower confidence bound of $f({\bm x},{\bm w}_j )$ after adding $({\bm x}^\ast,{\bm w}^\ast,y^\ast )$ to $\{ ({\bm x}_i ,{\bm w}_i,y_i )\} _{i=1}^t $. 
Next, for any $N^{({\bm b} ) } $, there exists $s_{\bm b}  \in N^{({\bm b} ) } $ such that 
\begin{align}
l_t ({\bm x} ,{\bm w} _{s_{\bm b}} )  \leq  h-\eta .\label{eq:app_b6}
\end{align}
In fact, if $l_t ({\bm x} ,{\bm w} _{s_{\bm b}} )  >  h-\eta $ for any $s \in N^{({\bm b} ) } $, then we get 
$$
 (\1 [l_t ({\bm x} ,{\bm w}_1 ) > h-\eta ] , \ldots , \1 [l_t ({\bm x} ,{\bm w}_{|\Omega|}) > h -\eta])^\top \in \mathcal{B} ,
$$
which contradicts  ${\bm x} \in U_t$. 
Furthermore, from Lemma 2 of \cite{zanette2018robust}, 
$\mathbb{P}_{y^\ast}  [   l_t ({\bm x} ,{\bm w}_{s_{\bm b} } |{\bm x}^\ast,{\bm w}^\ast, y^\ast) > h   ]$ can be calculated as
\begin{align}
\mathbb{P}_{y^\ast}  [l_t ({\bm x} ,{\bm w}_{s_{\bm b} } |{\bm x}^\ast,{\bm w}^\ast, y^\ast) > h   ] =
\Phi \left ( 
\frac{\sqrt{\sigma^2_{t-1} ({\bm x}^\ast,{\bm w}^\ast )   +\sigma^2 }}{|k_{t-1} (({\bm x},{\bm w}_{s_{{\bm b}}}),  ({\bm x}^\ast,{\bm w}^\ast)   )|} (\mu_{t-1}  ({\bm x},{\bm w}_{s_{\bm b} } )  -\beta^{1/2}_t \sigma_{t-1}   ({\bm x},{\bm w}_{s_{\bm b} } | {\bm x}^\ast,{\bm w}^\ast )    -h  ) 
\right ),\label{eq:app_b7}
\end{align}
where $\sigma_{t-1}   ({\bm x},{\bm w}_{s_{\bm b} } | {\bm x}^\ast,{\bm w}^\ast )   $ is the posterior variance of $f ({\bm x},{\bm w}_{s_{\bm b} }) $ after adding $({\bm x}^\ast,{\bm w}^\ast,y^\ast )$ to $\{ ({\bm x}_i ,{\bm w}_i,y_i ) \} _{i=1}^t $. 
Moreover, by using \eqref{eq:app_b6} we obtain 
\begin{align}
&\mu_{t-1}  ({\bm x},{\bm w}_{s_{\bm b} } )  -\beta^{1/2}_t \sigma_{t-1}   ({\bm x},{\bm w}_{s_{\bm b} } | {\bm x}^\ast,{\bm w}^\ast )    -h \nonumber \\
&=\mu_{t-1}  ({\bm x},{\bm w}_{s_{\bm b} } ) -\beta^{1/2}_t  \sigma_{t-1}   ({\bm x},{\bm w}_{s_{\bm b} })+ \beta^{1/2}_t  \sigma_{t-1}   ({\bm x},{\bm w}_{s_{\bm b} })   -\beta^{1/2}_t \sigma_{t-1}   ({\bm x},{\bm w}_{s_{\bm b} } | {\bm x}^\ast,{\bm w}^\ast )    -h \nonumber \\ 
&= l_t ({\bm x} ,{\bm w} _{s_{\bm b}} )   + \beta^{1/2}_t  \sigma_{t-1}   ({\bm x},{\bm w}_{s_{\bm b} })   -\beta^{1/2}_t \sigma_{t-1}   ({\bm x},{\bm w}_{s_{\bm b} } | {\bm x}^\ast,{\bm w}^\ast )    -h  \nonumber  \\
&\leq  -\eta + \beta^{1/2}_t  (    \sigma_{t-1}   ({\bm x},{\bm w}_{s_{\bm b} })   - \sigma_{t-1}   ({\bm x},{\bm w}_{s_{\bm b} } | {\bm x}^\ast,{\bm w}^\ast )      ). \label{eq:app_b8}
\end{align}
In addition, the following three inequalities hold:
\begin{align}
\sigma & \leq \sqrt{\sigma^2_{t-1} ({\bm x}^\ast,{\bm w}^\ast )   +\sigma^2 } , \label{eq:app_b9} \\
|k_{t-1} (({\bm x},{\bm w}_{s_{{\bm b}}}),  ({\bm x}^\ast,{\bm w}^\ast)   )| &\leq \sigma_{t-1} ({\bm x},{\bm w}_{s_{{\bm b}}}) \sigma_{t-1}  ({\bm x}^\ast,{\bm w}^\ast)   \leq 
 \sigma_{0} ({\bm x},{\bm w}_{s_{{\bm b}}}) \sigma_{t-1}  ({\bm x}^\ast,{\bm w}^\ast) \leq \sigma_{t-1}  ({\bm x}^\ast,{\bm w}^\ast), \label{eq:app_b10} \\
 \sigma_{t-1}   ({\bm x},{\bm w}_{s_{\bm b} })   - \sigma_{t-1}   ({\bm x},{\bm w}_{s_{\bm b} } | {\bm x}^\ast,{\bm w}^\ast ) &  \leq 
\frac{  \sigma_{t-1} ({\bm x},{\bm w}_{s_{{\bm b}}}) \sigma^2_{t-1}  ({\bm x}^\ast,{\bm w}^\ast)      }{  \sigma^2_{t-1} ({\bm x}^\ast,{\bm w}^\ast )   +\sigma^2    } \leq 
\frac{  \sigma_{0} ({\bm x},{\bm w}_{s_{{\bm b}}}) \sigma^2_{t-1}  ({\bm x}^\ast,{\bm w}^\ast)      }{ \sigma^2    } \leq 
\frac{ \sigma^2_{t-1}  ({\bm x}^\ast,{\bm w}^\ast)      }{ \sigma^2    } , \label{eq:app_b11}
\end{align}
where the first, second and third inequalities in \eqref{eq:app_b10} can be derived from H\"{o}lder's inequality, monotonicity of the posterior variance and the assumption $\max_{ ({\bm x},{\bm w} ) \in \mathcal{X} \times \Omega } \sigma^2_0 ({\bm x},{\bm w} ) \leq 1$,  respectively.  
Similarly, the first inequality in \eqref{eq:app_b11} can be derived from the equation (39) of \cite{zanette2018robust}.
 Therefore, by substituting \eqref{eq:app_b8}--\eqref{eq:app_b11} and \eqref{eq:app_b2} into \eqref{eq:app_b7}, 
we obtain the following inequality:
\begin{align}
\mathbb{P}_{y^\ast}  [l_t ({\bm x} ,{\bm w}_{s_{\bm b} } |{\bm x}^\ast,{\bm w}^\ast, y^\ast) > h   ] \leq
\Phi \left ( 
\frac{\sigma }{  \sigma_{t-1} ({\bm x}^\ast,{\bm w}^\ast ) } (-\eta/2 ) 
\right ), \label{eq:app_b12}
\end{align}
Moreover, noting that  the assumption \eqref{eq:app_b3} is equal to the condition $1<\sigma \sigma^{-1}_{t-1} ({\bm x}^\ast,{\bm w}^\ast ) (\eta/2)$, 
the right hand side in \eqref{eq:app_b12} can be bounded as
\begin{align}
\Phi \left ( 
\frac{\sigma }{  \sigma_{t-1} ({\bm x}^\ast,{\bm w}^\ast ) } (-\eta/2 ) 
\right ) &= \int _{-\infty} ^ {\frac{\sigma }{  \sigma_{t-1} ({\bm x}^\ast,{\bm w}^\ast ) } (-\eta/2 ) }  \phi (z) \text{d} z \nonumber \\
&= \int^\infty _{\frac{\sigma }{  \sigma_{t-1} ({\bm x}^\ast,{\bm w}^\ast ) } (\eta/2 )} \phi (z) \text{d} z \nonumber \\
&\leq \int^\infty _{\frac{\sigma }{  \sigma_{t-1} ({\bm x}^\ast,{\bm w}^\ast ) } (\eta/2 )} z \phi (z) \text{d} z \nonumber \\
&=   [ - \phi (z) ]^\infty _{\frac{\sigma }{  \sigma_{t-1} ({\bm x}^\ast,{\bm w}^\ast ) } (\eta/2 )} \nonumber \\
&= \frac{1}{\sqrt{2 \pi} }  \exp \left (
-\frac{\sigma^2 \eta^2}{8   \sigma^2_{t-1} ({\bm x}^\ast,{\bm w}^\ast )}
\right ) .\label{eq:app_b13}
\end{align}
Finally, from \eqref{eq:app_b5}, \eqref{eq:app_b12} and \eqref{eq:app_b13}, 
$ \mathbb{E}_{y^\ast}  [ \1 [l^{(F)}_t ({\bm x} ;0 |{\bm x}^\ast,{\bm w}^\ast, y^\ast) > \alpha ] ] $ can be bounded as
\begin{align*}
&\mathbb{E}_{y^\ast}  [ \1 [l^{(F)}_t ({\bm x} ;0 |{\bm x}^\ast,{\bm w}^\ast, y^\ast) > \alpha ] ]  \\
& \leq \sum_{ {\bm b} \in \mathcal{B}} \mathbb{P}_{y^\ast}  [ \f s \in N^{({\bm b}) }, \1 [l_t ({\bm x} ,{\bm w}_s |{\bm x}^\ast,{\bm w}^\ast, y^\ast) > h ]=b_s  ] \\
& \leq 
\sum_{ {\bm b} \in \mathcal{B}} \mathbb{P}_{y^\ast}  [  \1 [l_t ({\bm x} ,{\bm w}_{s_{\bm b} } |{\bm x}^\ast,{\bm w}^\ast, y^\ast) > h ]=b_{s_{\bm b}}  ] \\
&=\sum_{ {\bm b} \in \mathcal{B}} \mathbb{P}_{y^\ast}  [  l_t ({\bm x} ,{\bm w}_{s_{\bm b} } |{\bm x}^\ast,{\bm w}^\ast, y^\ast) > h  ] \\
&\leq \sum_{ {\bm b} \in \mathcal{B}}   \frac{1}{\sqrt{2 \pi} }  \exp \left (
-\frac{\sigma^2 \eta^2}{8   \sigma^2_{t-1} ({\bm x}^\ast,{\bm w}^\ast )} 
\right )  \\
&= | \mathcal{B}  | \frac{1}{\sqrt{2 \pi} }  \exp \left (
-\frac{\sigma^2 \eta^2}{8   \sigma^2_{t-1} ({\bm x}^\ast,{\bm w}^\ast )} 
\right ) \leq 
2^{|\Omega|} \frac{1}{\sqrt{2 \pi} }  \exp \left (
-\frac{\sigma^2 \eta^2}{8   \sigma^2_{t-1} ({\bm x}^\ast,{\bm w}^\ast )} 
\right ) .
\end{align*}
Therefore, from the definition of $a_{t-1} ({\bm x}^\ast,{\bm w}^\ast )$, we have 
\begin{align*}
a_{t-1} ({\bm x}^\ast,{\bm w}^\ast ) &= \sum_{ {\bm x} \in U_t }  \mathbb{E}_{y^\ast}  [ \1 [l^{(F)}_t ({\bm x} ;0 |{\bm x}^\ast,{\bm w}^\ast, y^\ast) > \alpha ] ]  \\
& \leq \sum_{ {\bm x} \in U_t } 2^{|\Omega|} \frac{1}{\sqrt{2 \pi} }  \exp \left (
-\frac{\sigma^2 \eta^2}{8   \sigma^2_{t-1} ({\bm x}^\ast,{\bm w}^\ast )} 
\right )  \\
&= |U_t | 2^{|\Omega|} \frac{1}{\sqrt{2 \pi} }  \exp \left (
-\frac{\sigma^2 \eta^2}{8   \sigma^2_{t-1} ({\bm x}^\ast,{\bm w}^\ast )} 
\right )  \leq  |\mathcal{X} | 2^{|\Omega|} \frac{1}{\sqrt{2 \pi} }  \exp \left (
-\frac{\sigma^2 \eta^2}{8   \sigma^2_{t-1} ({\bm x}^\ast,{\bm w}^\ast )} 
\right ) .
\end{align*}
\end{proof}

\begin{lemma}\label{lem:bound_xast}
Let $\eta >0$, $\beta_t >0$ and $\gamma >0$. Also let $({\bm x}_t ,{\bm w}_t ) \in \mathcal{X} \times \Omega$ be a maximum point of $a^{(1)}_{t-1} ({\bm x}^\ast,{\bm w}^\ast ) $. 
Assume that the following inequalities hold for some $T \geq 1$:
\begin{align}
\sigma^{-2} \sigma^2_{T-1} ({\bm x}_T,{\bm w}_T ) \beta^{1/2}_T &< \frac{\eta }{2}, \label{eq:app_b14} \\
\sigma^{-2} \sigma^2_{T-1} ({\bm x}_T,{\bm w}_T )  &< \eta^2/4, \label{eq:app_b15} \\
\sigma^2_{T-1} ({\bm x}_T,{\bm w}_T ) \beta_T & < \eta^2/4, \label{eq:app_b16} \\
\frac{1}{2} \log \beta_T - \frac{\eta^2 \sigma^2}{8 \sigma^2_{T-1}   ({\bm x}_T,{\bm w}_T )  } &< \log (|\mathcal{X}|^{-1}  2^{-|\Omega|} \eta \gamma 2^{-1} \sqrt{2 \pi }). \label{eq:app_b17}
\end{align}
Then, Algorithm \ref{alg:1} terminates after at most $T$ iterations.
\end{lemma}
\begin{proof}
From the definitions of $a^{(1)}_{t-1} ({\bm x}^\ast,{\bm w}^\ast ) $ and $({\bm x}_t ,{\bm w}_t )$, the following holds for any $({\bm x},{\bm w} ) \in \mathcal{X} \times \Omega$:
\begin{align}
\gamma \sigma_{T-1} ({\bm x},{\bm w} ) \leq a^{(1)}_{T-1} ({\bm x},{\bm w} ) \leq a^{(1)}_{T-1} ({\bm x}_T,{\bm w}_T ) 
= \max \{   a_{T-1} ({\bm x}_T,{\bm w}_T )  , \gamma \sigma_{T-1} ({\bm x}_T,{\bm w}_T )   \}. \label{eq:app_b18}
\end{align}
In addition, from \eqref{eq:app_b14}, \eqref{eq:app_b15} and Lemma \ref{lem:PTRMILE_bound}, $  a_{T-1} ({\bm x}_T,{\bm w}_T )$ can be bounded as 
\begin{align}
a_{T-1} ({\bm x}_T,{\bm w}_T ) \leq |\mathcal{X} | 2^{|\Omega|} \frac{1}{\sqrt{2 \pi} }  \exp \left (
-\frac{\sigma^2 \eta^2}{8   \sigma^2_{T-1} ({\bm x}_T,{\bm w}_T)} 
\right ) . \label{eq:app_b19}
\end{align}
Thus, by substituting \eqref{eq:app_b19} into \eqref{eq:app_b18}, we have 
$$
\gamma \sigma_{T-1} ({\bm x},{\bm w} ) \leq \max \left \{    |\mathcal{X} | 2^{|\Omega|} \frac{1}{\sqrt{2 \pi} }  \exp \left (
-\frac{\sigma^2 \eta^2}{8   \sigma^2_{T-1} ({\bm x}_T,{\bm w}_T)} 
\right )  , \gamma \sigma_{T-1} ({\bm x}_T,{\bm w}_T )    \right \}.
$$
This implies that 
\begin{align}
\beta^{1/2}_T \sigma_{T-1} ({\bm x},{\bm w} ) \leq \max \left \{ \gamma^{-1} \beta^{1/2}_T   |\mathcal{X} | 2^{|\Omega|} \frac{1}{\sqrt{2 \pi} }  \exp \left (
-\frac{\sigma^2 \eta^2}{8   \sigma^2_{T-1} ({\bm x}_T,{\bm w}_T)} 
\right )  , \beta^{1/2}_T  \sigma_{T-1} ({\bm x}_T,{\bm w}_T )    \right \}. \label{eq:app_b20}
\end{align}
On the other hand, \eqref{eq:app_b16} and \eqref{eq:app_b17} are equal to the following inequalities, respectively:
\begin{align}
\beta^{1/2}_T \sigma_{T-1} ({\bm x}_T,{\bm w}_T )  & < \eta/2, \label{eq:app_b21} \\
\exp \left ( - \frac{\eta^2 \sigma^2}{8 \sigma^2_{T-1}   ({\bm x}_T,{\bm w}_T )  } \right )  &< \frac{ |\mathcal{X}|^{-1}  2^{-|\Omega|} \eta \gamma 2^{-1} \sqrt{2 \pi }}{\beta^{1/2}_T }.  \label{eq:app_b22}
\end{align}
Hence, by combining \eqref{eq:app_b20}, \eqref{eq:app_b21} and \eqref{eq:app_b22}, we get $\beta^{1/2}_T \sigma_{T-1} ({\bm x}_T,{\bm w}_T ) < \eta /2 $. Therefore, from Lemma \ref{lem:eta_end}, we have Lemma \ref{lem:bound_xast}.
\end{proof}

\begin{lemma}\label{lem:RMILE_bound}
Let $ \eta >0 $ and $\beta_t >0$. Assume that \eqref{eq:app_b2} and \eqref{eq:app_b3} hold for some $({\bm x}^\ast,{\bm w}^\ast) \in \mathcal{X} \times \Omega $. Then, ${\rm MILE}_{t-1} ({\bm x}^\ast ,{\bm w}^\ast ) $ can be bounded as 
\begin{align}
{\rm MILE}_{t-1} ({\bm x}^\ast ,{\bm w}^\ast ) \leq 
 |\mathcal{X} \times \Omega | \frac{1}{\sqrt{2 \pi} }  \exp \left (
-\frac{\sigma^2 \eta^2}{8   \sigma^2_{t-1} ({\bm x}^\ast,{\bm w}^\ast )} 
\right ) . \nonumber 
\end{align}
\end{lemma}
\begin{proof}
From Lemma 2 of \cite{zanette2018robust} and the definition of ${\rm MILE}_{t-1} ({\bm x},{\bm w} )$, the following holds:
\begin{align}
&{\rm MILE}_{t-1} ({\bm x}^\ast,{\bm w}^\ast ) \nonumber \\
&= 
\sum_{ ({\bm x},{\bm w} ) \in U_t \times \Omega }  \mathbb{E} _{y^\ast }  [ \1 [  l_t ({\bm x},{\bm w} |{\bm x}^\ast,{\bm w}^\ast,y^\ast   )>h ]  ] 
- | \{   ({\bm x},{\bm w} ) \in U_t \times \Omega  \mid l_t ({\bm x},{\bm w} ) >h-\eta \} | \nonumber \\
&= 
\sum_{ ({\bm x},{\bm w} ) \in U_t \times \Omega }  \mathbb{P} _{y^\ast }  [   l_t ({\bm x},{\bm w} |{\bm x}^\ast,{\bm w}^\ast,y^\ast   )>h  ] 
- | \{   ({\bm x},{\bm w} ) \in U_t \times \Omega  \mid l_t ({\bm x},{\bm w} ) >h-\eta \} | \nonumber \\
&\leq \sum_{ ({\bm x},{\bm w} ) \in U_t \times \Omega }  
\Phi \left ( 
\frac{\sqrt{\sigma^2_{t-1} ({\bm x}^\ast,{\bm w}^\ast )   +\sigma^2 }}{|k_{t-1} (({\bm x},{\bm w}),  ({\bm x}^\ast,{\bm w}^\ast)   )|} (\mu_{t-1}  ({\bm x},{\bm w} )  -\beta^{1/2}_t \sigma_{t-1}   ({\bm x},{\bm w} | {\bm x}^\ast,{\bm w}^\ast )    -h  ) 
\right ) \nonumber \\
&\quad - | \{   ({\bm x},{\bm w} ) \in U_t \times \Omega  \mid l_t ({\bm x},{\bm w} ) >h-\eta \} |.  \nonumber \\
&=     \sum_{ ({\bm x},{\bm w} ) \in U_t \times \Omega }  \left \{ 
\Phi \left ( 
\frac{\sqrt{\sigma^2_{t-1} ({\bm x}^\ast,{\bm w}^\ast )   +\sigma^2 }}{|k_{t-1} (({\bm x},{\bm w}),  ({\bm x}^\ast,{\bm w}^\ast)   )|} (\mu_{t-1}  ({\bm x},{\bm w} )  -\beta^{1/2}_t \sigma_{t-1}   ({\bm x},{\bm w} | {\bm x}^\ast,{\bm w}^\ast )    -h  ) 
\right ) 
- \1[ l_t ({\bm x},{\bm w} ) >h-\eta ] \right \} .     \label{eq:app_b24} 
\end{align}
Next, for each $({\bm x},{\bm w} ) \in U_t \times \Omega$, we consider the two cases of $l_t ({\bm x},{\bm w} ) >h-\eta $ and 
$l_t ({\bm x},{\bm w} ) \leq h-\eta $. 
If $l_t ({\bm x},{\bm w} ) >h-\eta $, then the following inequality holds: 
\begin{align*}
&\Phi \left ( 
\frac{\sqrt{\sigma^2_{t-1} ({\bm x}^\ast,{\bm w}^\ast )   +\sigma^2 }}{|k_{t-1} (({\bm x},{\bm w}),  ({\bm x}^\ast,{\bm w}^\ast)   )|} (\mu_{t-1}  ({\bm x},{\bm w} )  -\beta^{1/2}_t \sigma_{t-1}   ({\bm x},{\bm w} | {\bm x}^\ast,{\bm w}^\ast )    -h  ) 
\right ) 
- \1[ l_t ({\bm x},{\bm w} ) >h-\eta ] \\
&\leq 0 \leq \frac{1}{\sqrt{2 \pi} }  \exp \left (
-\frac{\sigma^2 \eta^2}{8   \sigma^2_{t-1} ({\bm x}^\ast,{\bm w}^\ast )} 
\right ). 
\end{align*}
On the other hand, if $l_t ({\bm x},{\bm w} ) \leq h-\eta $, then using \eqref{eq:app_b7}--\eqref{eq:app_b13} we have 
\begin{align*}
&\Phi \left ( 
\frac{\sqrt{\sigma^2_{t-1} ({\bm x}^\ast,{\bm w}^\ast )   +\sigma^2 }}{|k_{t-1} (({\bm x},{\bm w}),  ({\bm x}^\ast,{\bm w}^\ast)   )|} (\mu_{t-1}  ({\bm x},{\bm w} )  -\beta^{1/2}_t \sigma_{t-1}   ({\bm x},{\bm w} | {\bm x}^\ast,{\bm w}^\ast )    -h  ) 
\right ) 
- \1[ l_t ({\bm x},{\bm w} ) >h-\eta ] \\
&=
\Phi \left ( 
\frac{\sqrt{\sigma^2_{t-1} ({\bm x}^\ast,{\bm w}^\ast )   +\sigma^2 }}{|k_{t-1} (({\bm x},{\bm w}),  ({\bm x}^\ast,{\bm w}^\ast)   )|} (\mu_{t-1}  ({\bm x},{\bm w} )  -\beta^{1/2}_t \sigma_{t-1}   ({\bm x},{\bm w} | {\bm x}^\ast,{\bm w}^\ast )    -h  ) 
\right )
 \leq \frac{1}{\sqrt{2 \pi} }  \exp \left (
-\frac{\sigma^2 \eta^2}{8   \sigma^2_{t-1} ({\bm x}^\ast,{\bm w}^\ast )} 
\right ). 
\end{align*}
Therefore, in both cases, the following inequality holds: 
\begin{align}
&\Phi \left ( 
\frac{\sqrt{\sigma^2_{t-1} ({\bm x}^\ast,{\bm w}^\ast )   +\sigma^2 }}{|k_{t-1} (({\bm x},{\bm w}),  ({\bm x}^\ast,{\bm w}^\ast)   )|} (\mu_{t-1}  ({\bm x},{\bm w} )  -\beta^{1/2}_t \sigma_{t-1}   ({\bm x},{\bm w} | {\bm x}^\ast,{\bm w}^\ast )    -h  ) 
\right ) 
- \1[ l_t ({\bm x},{\bm w} ) >h-\eta ] \nonumber \\
&
 \leq \frac{1}{\sqrt{2 \pi} }  \exp \left (
-\frac{\sigma^2 \eta^2}{8   \sigma^2_{t-1} ({\bm x}^\ast,{\bm w}^\ast )} 
\right ). \label{eq:app_b25}
\end{align}
Thus, by substituting \eqref{eq:app_b25} into \eqref{eq:app_b24}, we obtain 
\begin{align*}
{\rm MILE}_{t-1} ({\bm x}^\ast,{\bm w}^\ast )  \leq 
\sum_{ ({\bm x},{\bm w} ) \in U_t \times \Omega }   \frac{1}{\sqrt{2 \pi} }  \exp \left (
-\frac{\sigma^2 \eta^2}{8   \sigma^2_{t-1} ({\bm x}^\ast,{\bm w}^\ast )} 
\right ) &= 
|U_t \times \Omega | \frac{1}{\sqrt{2 \pi} }  \exp \left (
-\frac{\sigma^2 \eta^2}{8   \sigma^2_{t-1} ({\bm x}^\ast,{\bm w}^\ast )} 
\right ) \\
&\leq 
|\mathcal{X} \times \Omega | \frac{1}{\sqrt{2 \pi} }  \exp \left (
-\frac{\sigma^2 \eta^2}{8   \sigma^2_{t-1} ({\bm x}^\ast,{\bm w}^\ast )} 
\right ) .
\end{align*}
\end{proof}

\begin{lemma}\label{lem:bound_xast2}
Let $\eta >0$, $\beta_t >0$, $\gamma >0$ and $\tilde{\gamma} >0$. Also let $({\bm x}_t ,{\bm w}_t ) \in \mathcal{X} \times \Omega$ be a maximum point of $a^{(2)}_{t-1} ({\bm x}^\ast,{\bm w}^\ast ) $. 
Assume that the  inequalities \eqref{eq:app_b14}, \eqref{eq:app_b15} and \eqref{eq:app_b16} hold for some $T \geq 1$. 
In addition, assume that the following inequalities hold:
\begin{align}
\frac{1}{2} \log \beta_T - \frac{\eta^2 \sigma^2}{8 \sigma^2_{T-1}   ({\bm x}_T,{\bm w}_T )  } &< \log (|\mathcal{X}|^{-1}  2^{-|\Omega|} \eta \gamma \tilde{\gamma} 2^{-1} \sqrt{2 \pi }), \label{eq:app_b26} \\
\frac{1}{2} \log \beta_T - \frac{\eta^2 \sigma^2}{8 \sigma^2_{T-1}   ({\bm x}_T,{\bm w}_T )  } &< \log (|\mathcal{X} \times \Omega |^{-1}   \eta  \tilde{\gamma} 2^{-1} \sqrt{2 \pi }). \label{eq:app_b27} 
\end{align}
Then, Algorithm \ref{alg:1} terminates after at most $T$ iterations.
\end{lemma}
\begin{proof}
From the definition of $a^{(2)}_{t-1} ({\bm x}^\ast,{\bm w}^\ast)$ and  $({\bm x}_t ,{\bm w}_t ) $,  the following holds for any $({\bm x},{\bm w} ) \in \mathcal{X} \times \Omega$:
\begin{align}
\gamma \tilde{\gamma}  \sigma_{T-1}  ({\bm x},{\bm w} ) \leq \gamma {\rm RMILE}_{T-1} ({\bm x},{\bm w} ) \leq a^{(2)}_{T-1} ({\bm x},{\bm w} ) & \leq a^{(2)}_{T-1} ({\bm x}_T,{\bm w}_T ) \nonumber \\
&= \max \{   a_{T-1} ({\bm x}_T,{\bm w}_T )  , \gamma  {\rm RMILE}_{T-1} ({\bm x}_T,{\bm w}_T )   \}. \label{eq:app_b28}
\end{align}
Furthermore, from \eqref{eq:app_b14}, \eqref{eq:app_b15} and Lemma \ref{lem:RMILE_bound}, we have 
\begin{align}
\gamma  {\rm RMILE}_{T-1} ({\bm x}_T,{\bm w}_T ) &= \max \{  \gamma {\rm MILE}_{T-1}   ({\bm x}_T,{\bm w}_T ) , \gamma \tilde{\gamma}  \sigma_{T-1}  ({\bm x}_T,{\bm w}_T ) \}   \nonumber \\
&\leq 
\max \left \{
\gamma |\mathcal{X} \times \Omega | \frac{1}{\sqrt{2 \pi} }  \exp \left (
-\frac{\sigma^2 \eta^2}{8   \sigma^2_{T-1} ({\bm x}_T,{\bm w}_T )} 
\right ), \gamma \tilde{\gamma}  \sigma_{T-1}  ({\bm x}_T,{\bm w}_T )
\right \}. \label{eq:app_b29}
\end{align}
Moreover, from \eqref{eq:app_b16} and \eqref{eq:app_b27}, we get the following inequalities:
\begin{align}
\sigma_{T-1} ({\bm x}_T,{\bm w}_T) &< \beta^{-1/2}_T \eta /2 , \label{eq:app_b30} \\
|\mathcal{X} \times \Omega | \frac{1}{\sqrt{2 \pi} }  \exp \left (
-\frac{\sigma^2 \eta^2}{8   \sigma^2_{T-1} ({\bm x}_T,{\bm w}_T )} 
\right ) &< \beta^{-1/2}_T  \eta \tilde{\gamma} /2. \label{eq:app_b31}
\end{align}
Thus, by substituting \eqref{eq:app_b30} and \eqref{eq:app_b31} into \eqref{eq:app_b29}, we obtain 
\begin{align}
\gamma  {\rm RMILE}_{T-1} ({\bm x}_T,{\bm w}_T ) \leq \gamma \tilde{\gamma} \beta^{-1/2}_T \eta/2 . \label{eq:app_b32}
\end{align}
Similarly, from \eqref{eq:app_b14}, \eqref{eq:app_b15}, \eqref{eq:app_b26} and Lemma \ref{lem:PTRMILE_bound}, 
$ a_{T-1} ({\bm x}_T,{\bm w}_T ) $ can be bounded as
\begin{align}
a_{T-1} ({\bm x}_T,{\bm w}_T )  \leq  | \mathcal{X} |  2^{|\Omega|} \frac{1}{\sqrt{2 \pi} } \exp \left (
-\frac{\sigma^2 \eta^2}{8\sigma^2_{T-1} ({\bm x}_T,{\bm w}_T)}
\right ) \leq \gamma \tilde{\gamma} \beta^{-1/2}_T \eta /2. \label{eq:app_b33}
\end{align}
Hence, by combining \eqref{eq:app_b32} and \eqref{eq:app_b33} into \eqref{eq:app_b28}, we get 
$$
\gamma \tilde{\gamma}  \sigma_{T-1}  ({\bm x},{\bm w} ) \leq \gamma \tilde{\gamma} \beta^{-1/2}_T \eta /2.
$$
This implies that $2  \beta^{1/2}_T \sigma_{T-1}  ({\bm x},{\bm w} ) <\eta$. 
Therefore, from Lemma \ref{lem:eta_end}, we have Lemma  \ref{lem:bound_xast2}.
\end{proof}

\begin{lemma}\label{lem:postvar_bound}
Let $({\bm x}_1,{\bm w}_1 ),\ldots ,({\bm x}_t,{\bm w} _t ) $ be selected points, and define $C_1 =2/\log (1+\sigma^{-2} ) $.
 Then, there exists a natural number $t^\prime \leq t$ such that 
$$
\sigma^2_{t^\prime -1 } ( {\bm x}_{t^\prime},{\bm w}_{t^\prime } ) \leq \frac{C_1 \kappa_t}{t}.
$$
\end{lemma}
\begin{proof}
From Lemma 5.3 in \cite{SrinivasGPUCB}, the mutual information $I({\bm y}_A;f)$ can be expressed as
\begin{align}
 I({\bm y}_A;f) = \frac{1}{2}   \sum_{i=1}^t \log (1 + \sigma^{-2} \sigma^2_{i-1} ({\bm x}_i,{\bm w}_i ) ). \label{eq:app_b34}
\end{align}
Similarly, from Lemma 5.4 in  \cite{SrinivasGPUCB}, $\sigma^2_{i-1} ({\bm x}_i,{\bm w}_i )$ can be bounded as 
\begin{align}
\sigma^2_{i-1} ({\bm x}_i,{\bm w}_i ) \leq \frac{ \log (1+ \sigma^{-2} \sigma^2_{i-1} ({\bm x}_i ,{\bm w}_i ) )  }{\log (1+\sigma^{-2})}. \label{eq:app_b35}
\end{align}
Hence, by using \eqref{eq:app_b34} and \eqref{eq:app_b35}, we get 
\begin{align}
\sum_{i=1}^t \sigma^2_{i-1} ({\bm x}_i,{\bm w}_i ) \leq \frac{2}{\log (1+ \sigma^{-2} ) } I({\bm y}_t;f) \leq C_1 \kappa_t. \label{eq:app_b36}
\end{align}
Next, we define $t^\prime $ as 
$
t^\prime = \argmin _{ 1 \leq i \leq t} \sigma^2_{i-1} ({\bm x}_i,{\bm w}_i )
$. 
Then, it follows that 
\begin{align}
t \sigma^2_{t^\prime -1} ({\bm x}_{t^\prime},{\bm w}_{t^\prime} ) \leq \sum_{i=1}^t \sigma^2_{i-1} ({\bm x}_i,{\bm w}_i ) . \label{eq:app_b37}
\end{align}
Therefore, by combining \eqref{eq:app_b36} and \eqref{eq:app_b37}, we have the desired inequality.
\end{proof}
Finally, using Lemma \ref{lem:bound_xast}, \ref{lem:bound_xast2} and \ref{lem:postvar_bound}, we prove Theorem \ref{thm:convergence1} and \ref{thm:convergence2}.

\begin{proof}
From Lemma \ref{lem:postvar_bound} and monotonicity of $\beta_t$, for any $t \geq 1$, there exists a natural number $t^\prime \leq t$ such that 
\begin{equation}
\begin{split}
\sigma^{-2} \sigma^2_{t^\prime-1} ({\bm x}_{t^\prime},{\bm w}_{t^\prime} ) \beta^{1/2}_{t^\prime} &\leq \frac{\sigma^{-2} \beta^{1/2}_{t^\prime} C_1 \kappa_t }{t} \leq      \frac{\sigma^{-2} \beta^{1/2}_{t} C_1 \kappa_t }{t}             ,  \\
\sigma^{-2} \sigma^2_{t^\prime-1} ({\bm x}_{t^\prime},{\bm w}_{t^\prime} )  &\leq \frac{ \sigma^{-2} C_1 \kappa_t }{t},  \\
\sigma^2_{t^\prime-1} ({\bm x}_{t^\prime},{\bm w}_{t^\prime} ) \beta_{t^\prime} & \leq \frac{   C_1 \beta_{t^\prime} \kappa_t  }{t}   \leq 
         \frac{   C_1 \beta_{t} \kappa_t  }{t}         ,  \\
\frac{1}{2} \log \beta_{t^\prime} - \frac{\eta^2 \sigma^2}{8 \sigma^2_{t^\prime-1}   ({\bm x}_{t^\prime},{\bm w}_{t^\prime} )  } &\leq  \frac{1}{2} \log \beta_{t^\prime} - \frac{T\eta^2 \sigma^2}{8 C_1 \kappa_t  } \leq      \frac{1}{2} \log \beta_{t} - \frac{T\eta^2 \sigma^2}{8 C_1 \kappa_t  }         .
\end{split} \label{eq:app_b38}
\end{equation}
Hence, from \eqref{eq:app_b38}, if the inequality conditions in Theorem \ref{thm:convergence1} hold, then the inequality conditions in Lemma   \ref{lem:bound_xast} also hold for some $\tilde{T} \leq T$. 
Therefore, from Lemma   \ref{lem:bound_xast}, Algorithm \ref{alg:1} terminates after at most $\tilde{T}$ iterations, i.e., Theorem \ref{thm:convergence1} holds. 
By using the same argument, Theorem \ref{thm:convergence2} can also be proved.
\end{proof}


\subsection{Proof of Lemma \ref{lem:exp_cal} and \ref{lem:teigi_bound}}
First, we prove Lemma \ref{lem:exp_cal}
\begin{proof}
From GP properties, the posterior mean $\mu_{t-1} ({\bm x},{\bm w} |{\bm x}^\ast,{\bm w}^\ast,y^\ast)$ and the posterior variance 
$\sigma^2_{t-1}  ({\bm x},{\bm w} |{\bm x}^\ast,{\bm w}^\ast)$ of $f({\bm x},{\bm w})$ after adding $({\bm x}^\ast,{\bm w}^\ast,y^\ast)$ can be written as follows (see, e.g., \cite{williams2006gaussian}):
\begin{align*}
\mu_{t-1} ({\bm x},{\bm w} |{\bm x}^\ast,{\bm w}^\ast,y^\ast) &= \mu_{t-1} ({\bm x},{\bm w}) -\frac{  k_{t-1}  ( ({\bm x},{\bm w}) ,  ({\bm x}^\ast,{\bm w}^\ast)   )     }{\sigma^2_{t-1} ({\bm x}^\ast,{\bm w}^\ast)+\sigma^2} (y^\ast-\mu_{t-1} ({\bm x}^\ast,{\bm w}^\ast)), \\
\sigma^2_{t-1}  ({\bm x},{\bm w} |{\bm x}^\ast,{\bm w}^\ast) &= \sigma^2_{t-1} ({\bm x},{\bm w}) -\frac{  k^2_{t-1}  ( ({\bm x},{\bm w}) ,  ({\bm x}^\ast,{\bm w}^\ast)   )     }{\sigma^2_{t-1} ({\bm x}^\ast,{\bm w}^\ast)+\sigma^2}.
\end{align*}
Thus, $l_t ({\bm x},{\bm w} | {\bm x}^\ast,{\bm w}^\ast, y^\ast )$ is a linear function with respect to (w.r.t.) $y^\ast$. 
Hence, the indicator function $\1[l_t ({\bm x},{\bm w}_j | {\bm x}^\ast,{\bm w}^\ast, y^\ast )>h]$ is a piecewise constant function w.r.t. $y^\ast$, where the breakpoint is $y^\ast = r_j $. 
Therefore, for any $s \in \{1,\ldots , |\Omega|+1 \}$, the following holds:
\begin{align*}
&( \1[l_t ({\bm x},{\bm w}_1 | {\bm x}^\ast,{\bm w}^\ast, c )>h], \ldots , \1[l_t ({\bm x},{\bm w}_{|\Omega|} | {\bm x}^\ast,{\bm w}^\ast, c )>h] )^\top \\
&=
( \1[l_t ({\bm x},{\bm w}_1 | {\bm x}^\ast,{\bm w}^\ast, c^\prime )>h], \ldots , \1[l_t ({\bm x},{\bm w}_{|\Omega|} | {\bm x}^\ast,{\bm w}^\ast, c^\prime )>h] )^\top , \quad \f c,c^\prime \in R_s.
\end{align*}
This implies that 
$$
l^{(F)}_t ({\bm x};0|{\bm x}^\ast,{\bm w}^\ast, c) = l^{(F)}_t ({\bm x};0|{\bm x}^\ast,{\bm w}^\ast, c^\prime) , \quad \f c,c^\prime \in R_s.
$$
Hence, using this we have 
\begin{align*}
&\mathbb{E}_{y^\ast}   [\1[l^{(F)}_t ({\bm x};0|{\bm x}^\ast,{\bm w}^\ast, y^\ast)>\alpha]] \\
&=\int \1[l^{(F)}_t ({\bm x};0|{\bm x}^\ast,{\bm w}^\ast, y^\ast)>\alpha] p(y^\ast ) \text{d} y^\ast \\
&=\sum_{s=1}^ {|\Omega|+1} \int _{y^\ast \in R_s} \1[l^{(F)}_t ({\bm x};0|{\bm x}^\ast,{\bm w}^\ast, y^\ast)>\alpha] p(y^\ast ) \text{d} y^\ast \\
&=\sum_{s=1}^ {|\Omega|+1} \1[l^{(F)}_t ({\bm x};0|{\bm x}^\ast,{\bm w}^\ast, c_s)>\alpha] \int _{y^\ast \in R_s}  p(y^\ast ) \text{d} y^\ast \\
&=\sum_{s=1}^ {|\Omega|+1} \mathbb{P} (y^\ast \in R_s) \1[l^{(F)}_t ({\bm x};0|{\bm x}^\ast,{\bm w}^\ast, c_s)>\alpha] .
\end{align*}
\end{proof}

Next, we prove Lemma \ref{lem:teigi_bound}.
\begin{proof}
From the definition of $l^{(F)}_t ({\bm x} ;0 |{\bm x}^\ast,{\bm w}^\ast, c_s) $, 
$l^{(F)}_t ({\bm x} ;0 |{\bm x}^\ast,{\bm w}^\ast, c_s) $ can be expressed as 
\begin{align*}
l^{(F)}_t ({\bm x} ;0 |{\bm x}^\ast,{\bm w}^\ast, c_s) 
=   
\inf _{p({\bm w})  \in \mathcal{A} } \sum_{ {\bm w} \in \Omega } \1[l_t ({\bm x},{\bm w} | {\bm x}^\ast,{\bm w}^\ast,c_s ) > h ]  p ({\bm w} ).
\end{align*}
Moreover, since $p^\ast ({\bm w}) \in \mathcal{A}$, the following holds: 
\begin{align*}
\inf _{p({\bm w})  \in \mathcal{A} } \sum_{ {\bm w} \in \Omega } \1[l_t ({\bm x},{\bm w} | {\bm x}^\ast,{\bm w}^\ast,c_s ) > h ]  p ({\bm w} ) 
\leq 
 \sum_{ {\bm w} \in \Omega } \1[l_t ({\bm x},{\bm w} | {\bm x}^\ast,{\bm w}^\ast,c_s ) > h ]  p^\ast ({\bm w} ).
\end{align*}
Therefore, we have 
\begin{align*}
l^{(F)}_t ({\bm x} ;0 |{\bm x}^\ast,{\bm w}^\ast, c_s) 
\leq \sum_{ {\bm w} \in \Omega } \1[l_t ({\bm x},{\bm w} | {\bm x}^\ast,{\bm w}^\ast,c_s ) > h ]  p^\ast ({\bm w} ).
\end{align*}
Hence, if the inequality assumption in Lemma \ref{lem:teigi_bound} holds, then we get $l^{(F)}_t ({\bm x} ;0 |{\bm x}^\ast,{\bm w}^\ast, c_s) \leq \alpha$.
 This implies that $\1 [ l^{(F)}_t ({\bm x} ;0 |{\bm x}^\ast,{\bm w}^\ast, c_s) > \alpha ] =0$.
\end{proof}

\section{Additional experiments}
\subsection{Synthetic and real data experiments in the $L2$-norm setting}
In this section, we performed the same experiment as in Subsection \ref{syn_experiment} and \ref{subsec:real} under the setting that  the distance function is $L2$-norm.
Similarly, we used Uniform and Normal as the reference distribution.
Here, the parameters used in the synthetic data experiments are listed in Table  \ref{tab:setting1}.
On the other hand, the same parameters as in Subsection \ref{subsec:real}  were used in the real data experiments.
Under this setup, we took one initial point at random, and ran the algorithms until the number of iterations reached 300 (resp. 100) in the synthetic data (resp. real data) experiments.
We performed 50 Monte Carlo simulations and obtained the average F-score.
From Figures \ref{fig:exp3} and \ref{fig:exp4}, it can be confirmed  that our proposed methods outperform other existing methods 
as well as the results of synthetic data experiments using $L1$-norm as the distance function. 
From Figure \ref{fig:exp5_2}, it can also be confirmed that the same results as in Subsection \ref{subsec:real} are obtained in real data experiments.  

\begin{table*}[t]
  \begin{center}
    \caption{Parameter setting in synthetic data experiments}
\scalebox{0.8}{
    \begin{tabular}{c||c|c||c|c} \hline \hline
       & $L1$-Uniform & $L1$-Normal & $L2$-Uniform & $L2$-Normal \\ \hline 
        & $h=100$, $\alpha =0.62$ & $h=100$, $\alpha =0.5$, & $h=100$, $\alpha =0.5$, & $h=100$, $\alpha =0.5$, \\
                & $\sigma^2=10^{-4}$, $\sigma^2_f = 1300^2$,  & $\sigma^2=10^{-4}$, $\sigma^2_f = 1300^2$,  & $\sigma^2=10^{-4}$, $\sigma^2_f = 1300^2$,  & $\sigma^2=10^{-4}$, $\sigma^2_f = 1300^2$, \\
      Booth  & $L=4$, $\beta^{1/2}_t =2$, $\epsilon = 0.65$, & $L=4$, $\beta^{1/2}_t =2$, $\epsilon = 0.65$, & $L=4$, $\beta^{1/2}_t =2$, $\epsilon = 0.05$, & $L=4$, $\beta^{1/2}_t =2$, $\epsilon = 0.1$,  \\
                & $L_1=-10$, $U_1 =10$,   & $L_1=-10$, $U_1 =10$,  & $L_1=-10$, $U_1 =10$,  & $L_1=-10$, $U_1 =10$,   \\ 
                & $L_2=-10$, $U_2 =10$  & $L_2=-10$, $U_2 =10$  & $L_2=-10$, $U_2 =10$  & $L_2=-10$, $U_2 =10$   \\ \hline
          & $h=5$, $\alpha =0.53$ & $h=5$, $\alpha =0.53$, & $h=5$, $\alpha =0.5$, & $h=5$, $\alpha =0.5$, \\
                & $\sigma^2=10^{-4}$, $\sigma^2_f = 50^2$,  & $\sigma^2=10^{-4}$, $\sigma^2_f = 50^2$,  & $\sigma^2=10^{-4}$, $\sigma^2_f = 50^2$,  & $\sigma^2=10^{-4}$, $\sigma^2_f = 50^2$, \\
      Matyas  & $L=4$, $\beta^{1/2}_t =2$, $\epsilon = 0.15$, & $L=4$, $\beta^{1/2}_t =2$, $\epsilon = 0.15$, & $L=4$, $\beta^{1/2}_t =2$, $\epsilon = 0.05$, & $L=4$, $\beta^{1/2}_t =2$, $\epsilon = 0.1$,  \\
                & $L_1=-10$, $U_1 =10$,   & $L_1=-10$, $U_1 =10$,  & $L_1=-10$, $U_1 =10$,  & $L_1=-10$, $U_1 =10$,   \\ 
                & $L_2=-10$, $U_2 =10$  & $L_2=-10$, $U_2 =10$  & $L_2=-10$, $U_2 =10$  & $L_2=-10$, $U_2 =10$   \\ \hline
      & $h=1$, $\alpha =0.57$ & $h=1$, $\alpha =0.59$, & $h=2$, $\alpha =0.5$, & $h=1$, $\alpha =0.55$, \\
                & $\sigma^2=10^{-4}$, $\sigma^2_f = 20^2$,  & $\sigma^2=10^{-6}$, $\sigma^2_f = 20^2$,  & $\sigma^2=10^{-6}$, $\sigma^2_f = 20^2$,  & $\sigma^2=10^{-6}$, $\sigma^2_f = 20^2$, \\
      McCormick  & $L=1$, $\beta^{1/2}_t =3$, $\epsilon = 0.25$, & $L=1$, $\beta^{1/2}_t =3$, $\epsilon = 0.15$, & $L=1$, $\beta^{1/2}_t =3$, $\epsilon = 0.05$, & $L=1$, $\beta^{1/2}_t =3$, $\epsilon = 0.07$,  \\
                & $L_1=-1.5$, $U_1 =4$,   & $L_1=-1.5$, $U_1 =4$,  & $L_1=-1.5$, $U_1 =4$,  & $L_1=-1.5$, $U_1 =4$,   \\ 
                & $L_2=-3$, $U_2 =4$  & $L_2=-3$, $U_2 =4$  & $L_2=-3$, $U_2 =4$  & $L_2=-3$, $U_2 =4$   \\ \hline
      & $h=-3990$, $\alpha =0.61$ & $h=-3990$, $\alpha =0.61$, & $h=-3990$, $\alpha =0.7$, & $h=-3990$, $\alpha =0.7$, \\
                & $\sigma^2=10^{-4}$, $\sigma^2_f = 2000^2$,  & $\sigma^2=10^{-4}$, $\sigma^2_f = 2000^2$,  & $\sigma^2=10^{-4}$, $\sigma^2_f = 2000^2$,  & $\sigma^2=10^{-4}$, $\sigma^2_f = 2000^2$, \\
      Styblinski-Tang  & $L=3$, $\beta^{1/2}_t =2$, $\epsilon = 0.2$, & $L=3$, $\beta^{1/2}_t =2$, $\epsilon = 0.2$, & $L=3$, $\beta^{1/2}_t =2$, $\epsilon = 0.05$, & $L=3$, $\beta^{1/2}_t =2$, $\epsilon = 0.1$,  \\
                & $L_1=-10$, $U_1 =10$,   & $L_1=-10$, $U_1 =10$,  & $L_1=-10$, $U_1 =10$,  & $L_1=-10$, $U_1 =10$,   \\ 
                & $L_2=-10$, $U_2 =10$  & $L_2=-10$, $U_2 =10$  & $L_2=-10$, $U_2 =10$  & $L_2=-10$, $U_2 =10$   \\ \hline \hline
    \end{tabular}
}
    \label{tab:setting1}
  \end{center}
\end{table*}

\begin{figure}[!t]
\begin{center}
 \begin{tabular}{cccc}
 \includegraphics[width=0.225\textwidth]{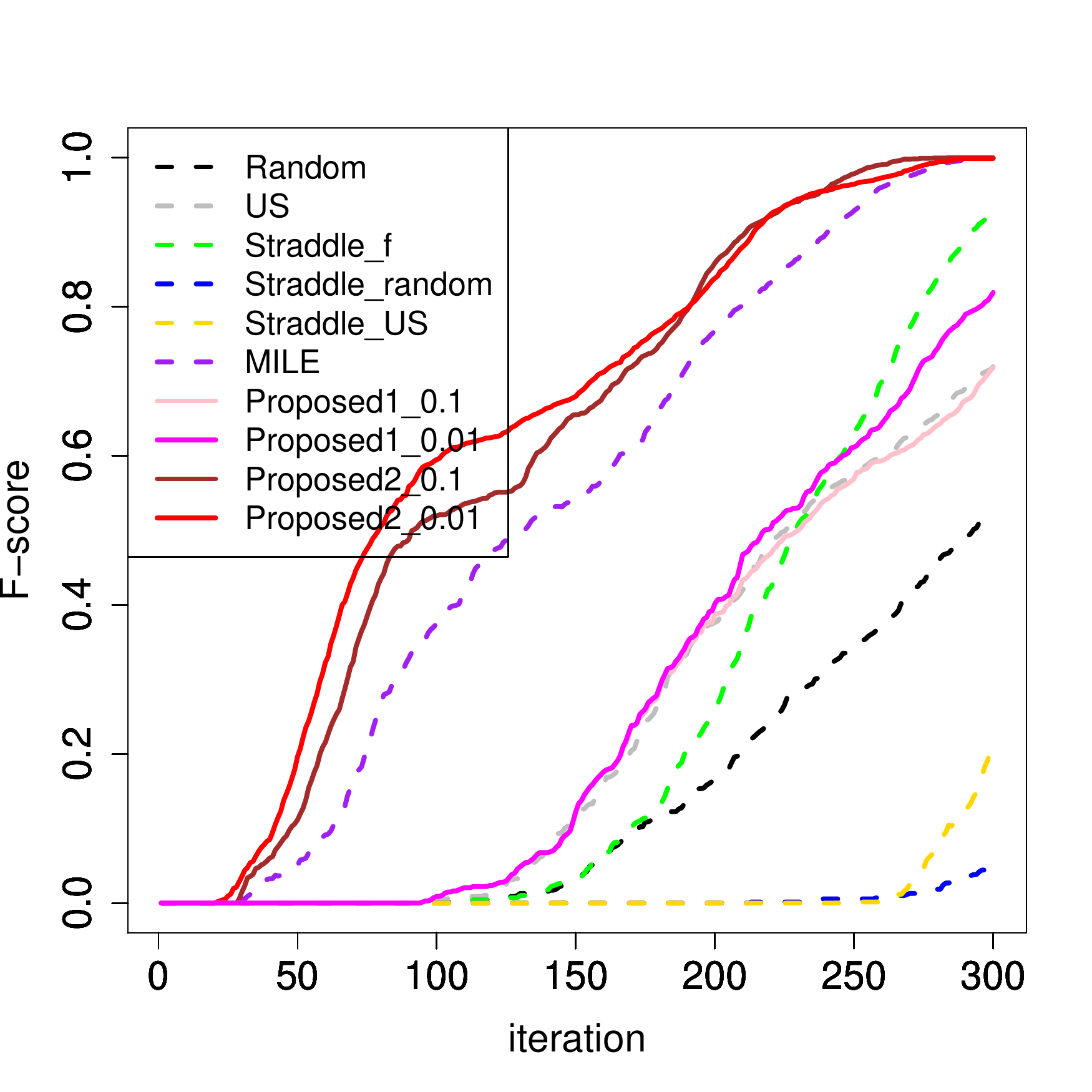} 
&
 \includegraphics[width=0.225\textwidth]{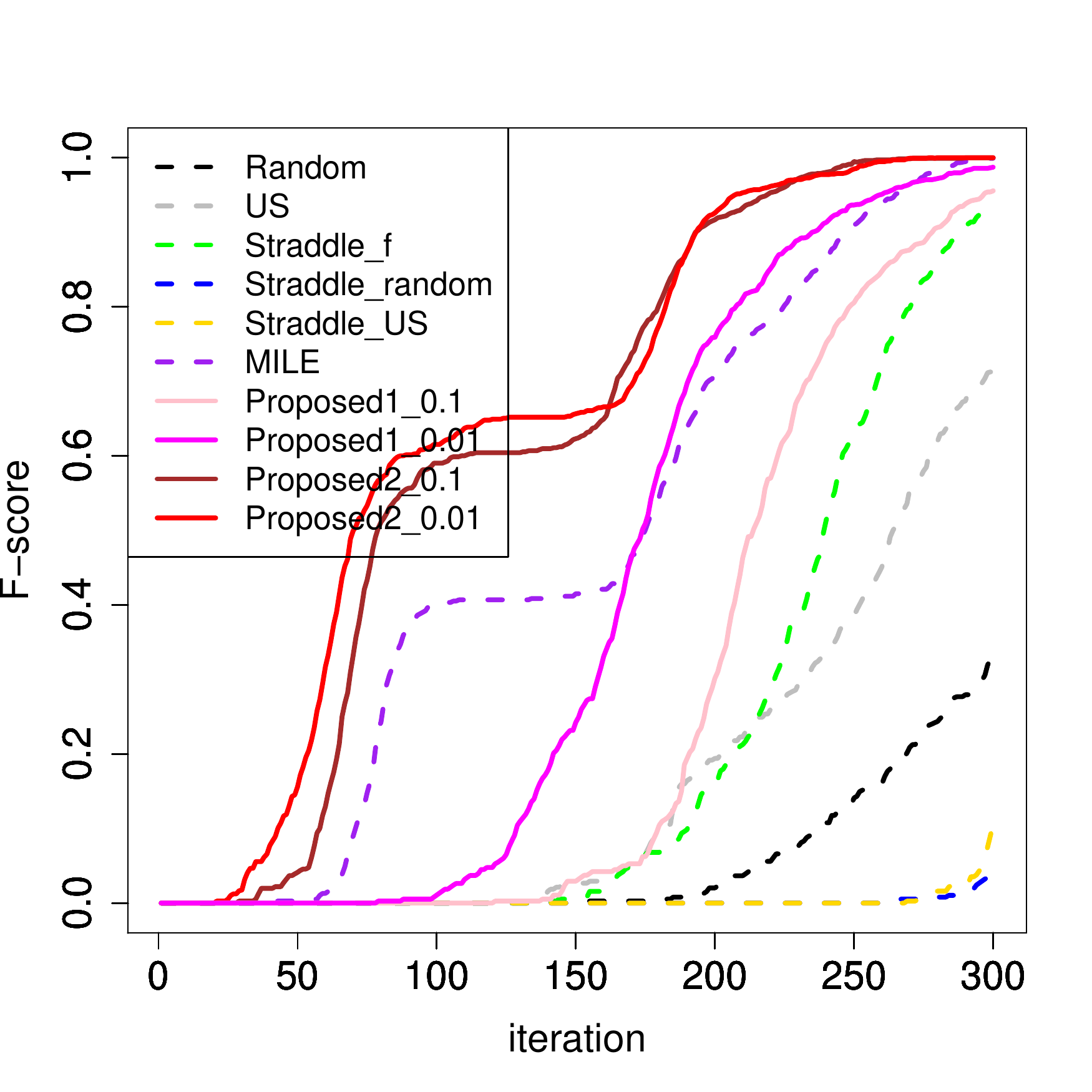} 
&
\includegraphics[width=0.225\textwidth]{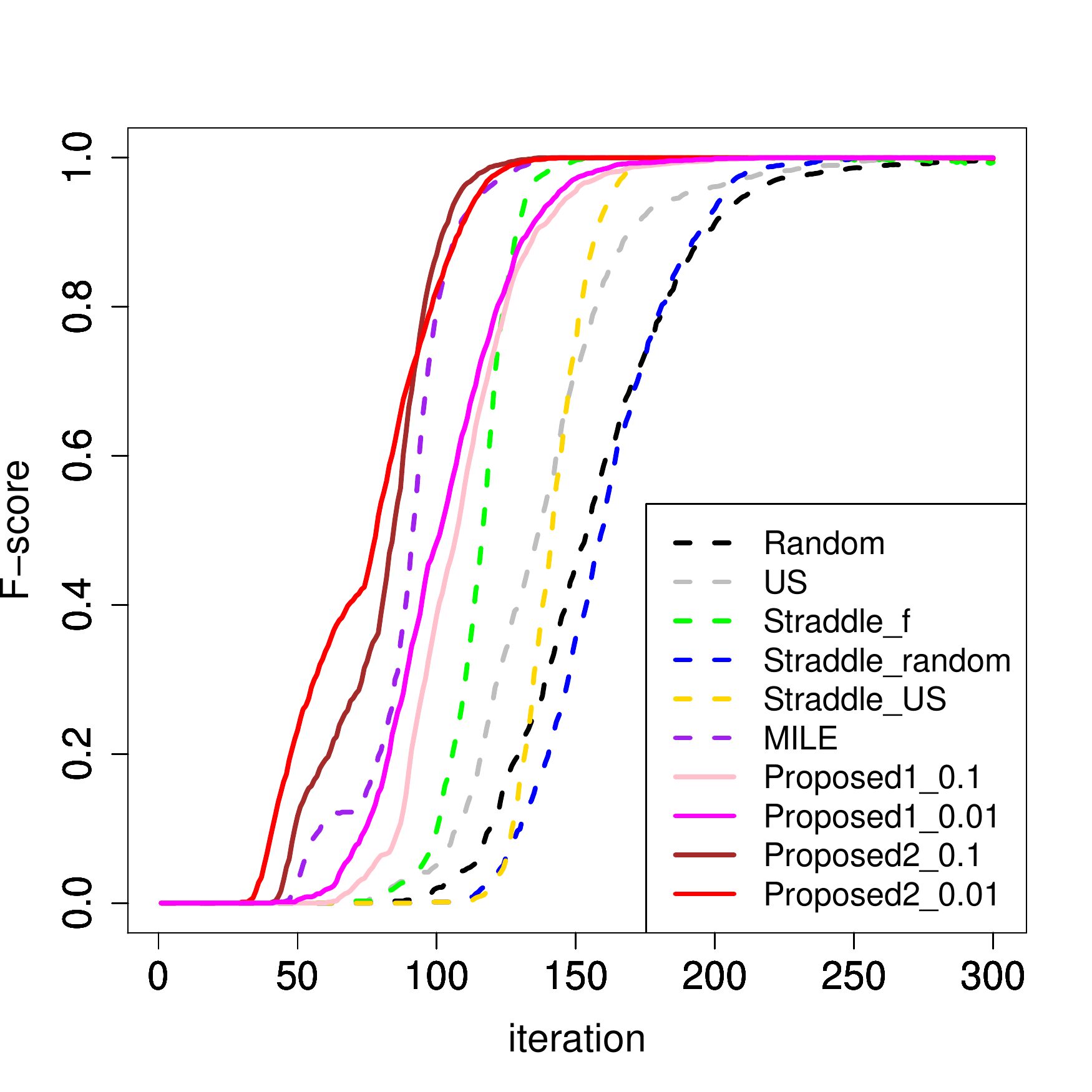} 
&
 \includegraphics[width=0.225\textwidth]{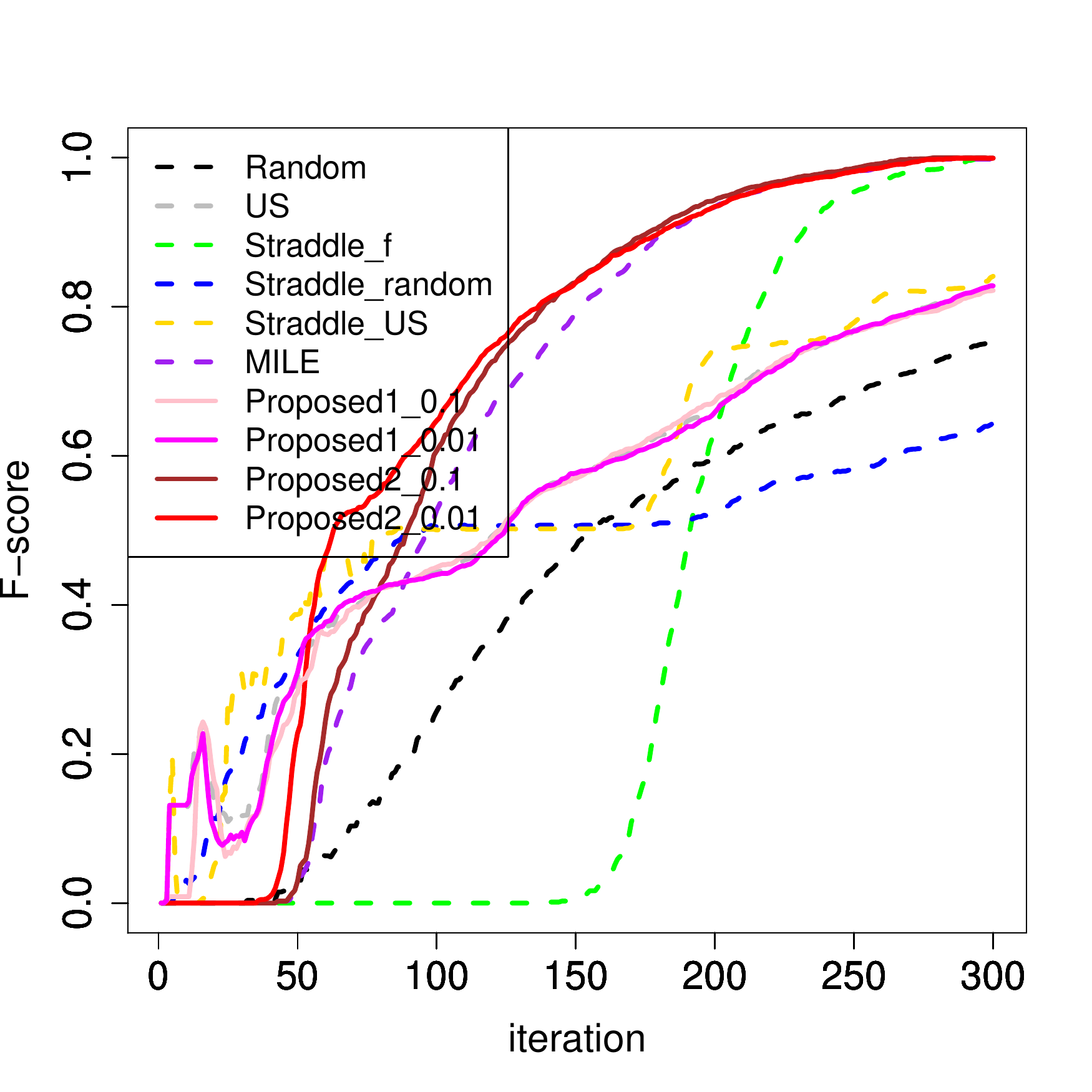} 
\\
Booth & Matyas &
McCormick & Styblinski-Tang
 \end{tabular}
\end{center}
 \caption{Average F-score over 50 simulations with four benchmark functions when the distance function and reference distribution are $L2$-norm and Uniform, respectively.}
\label{fig:exp3}
\end{figure}

\begin{figure}[!t]
\begin{center}
 \begin{tabular}{cccc}
\includegraphics[width=0.225\textwidth]{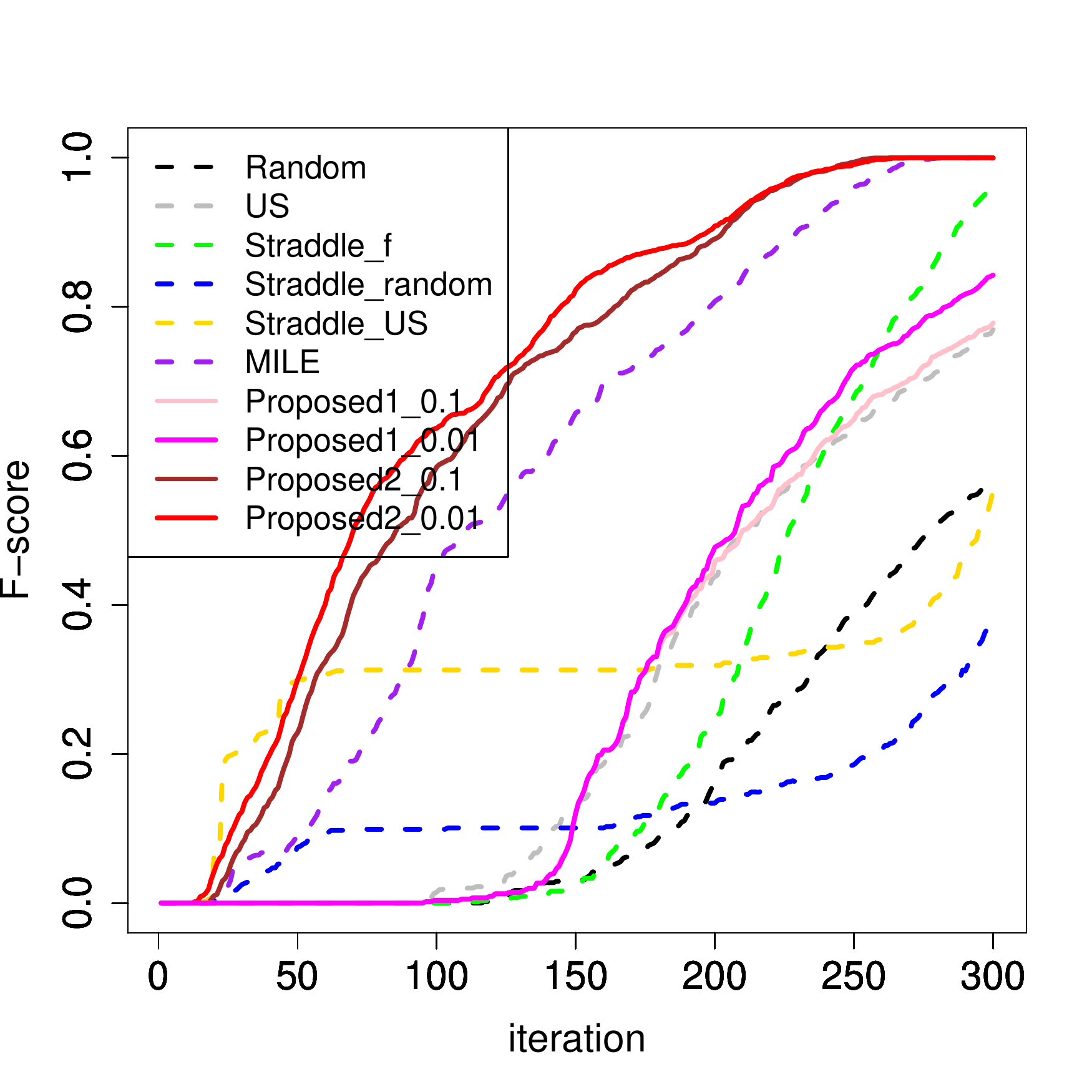} 
&
 \includegraphics[width=0.225\textwidth]{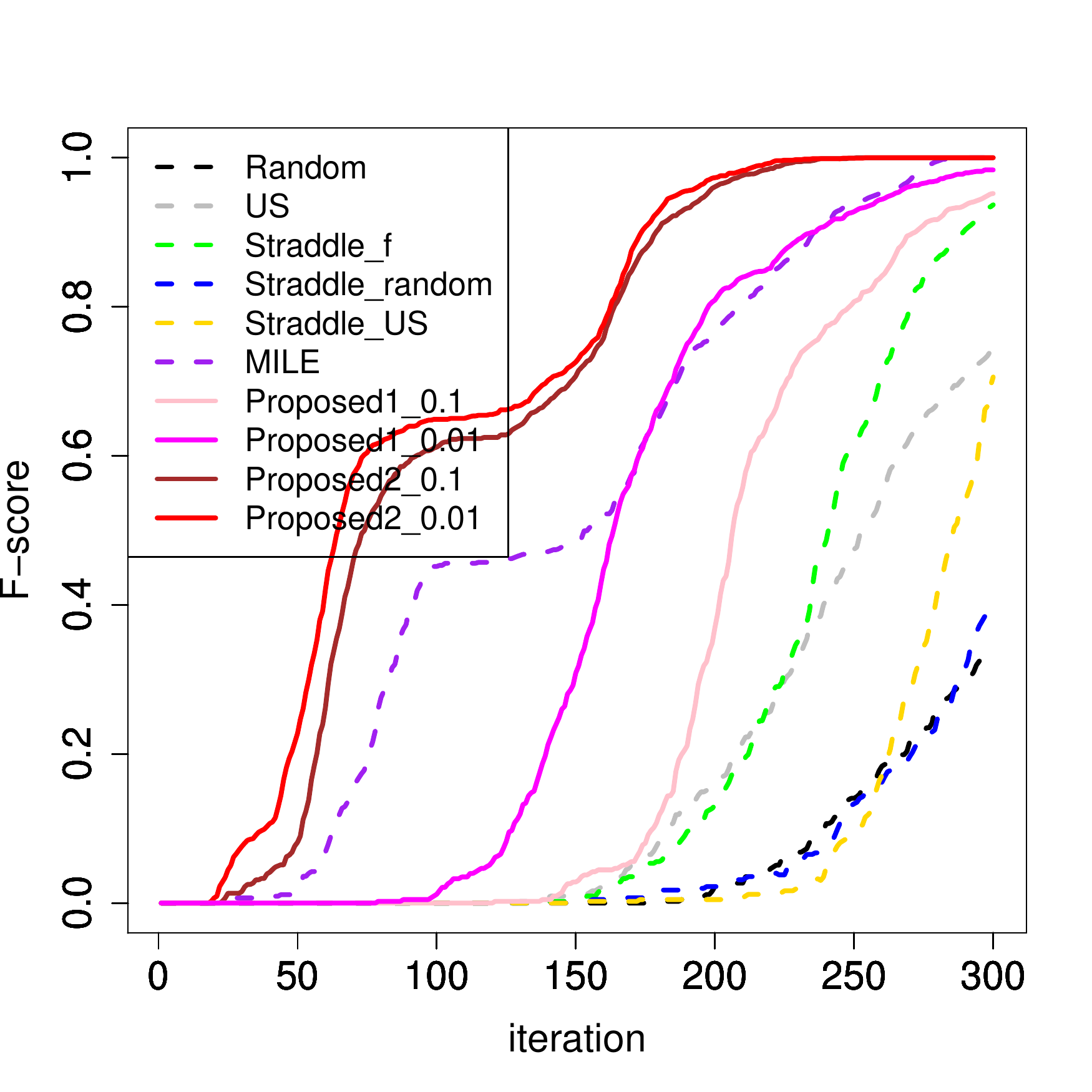} 
&
 \includegraphics[width=0.225\textwidth]{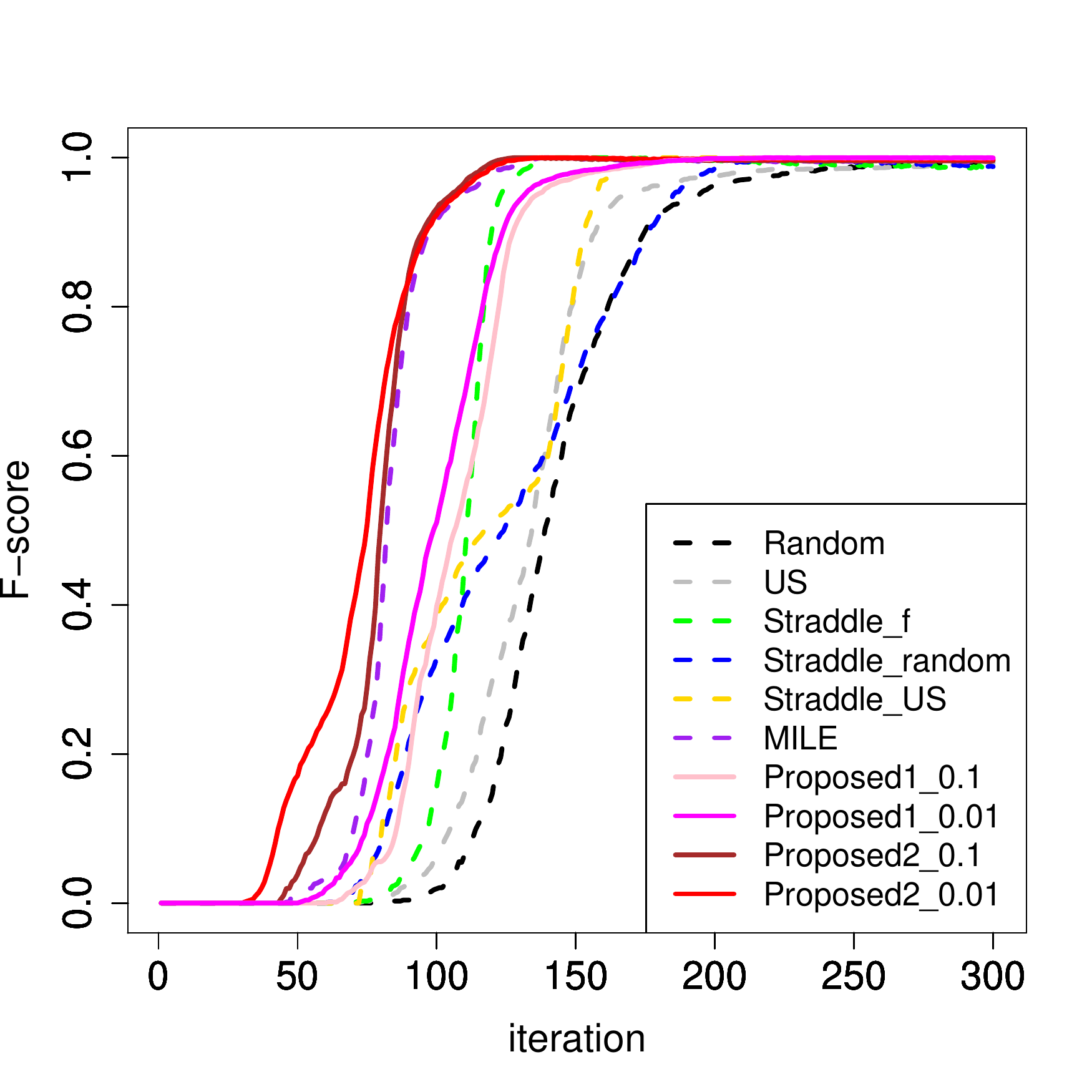} 
&
 \includegraphics[width=0.225\textwidth]{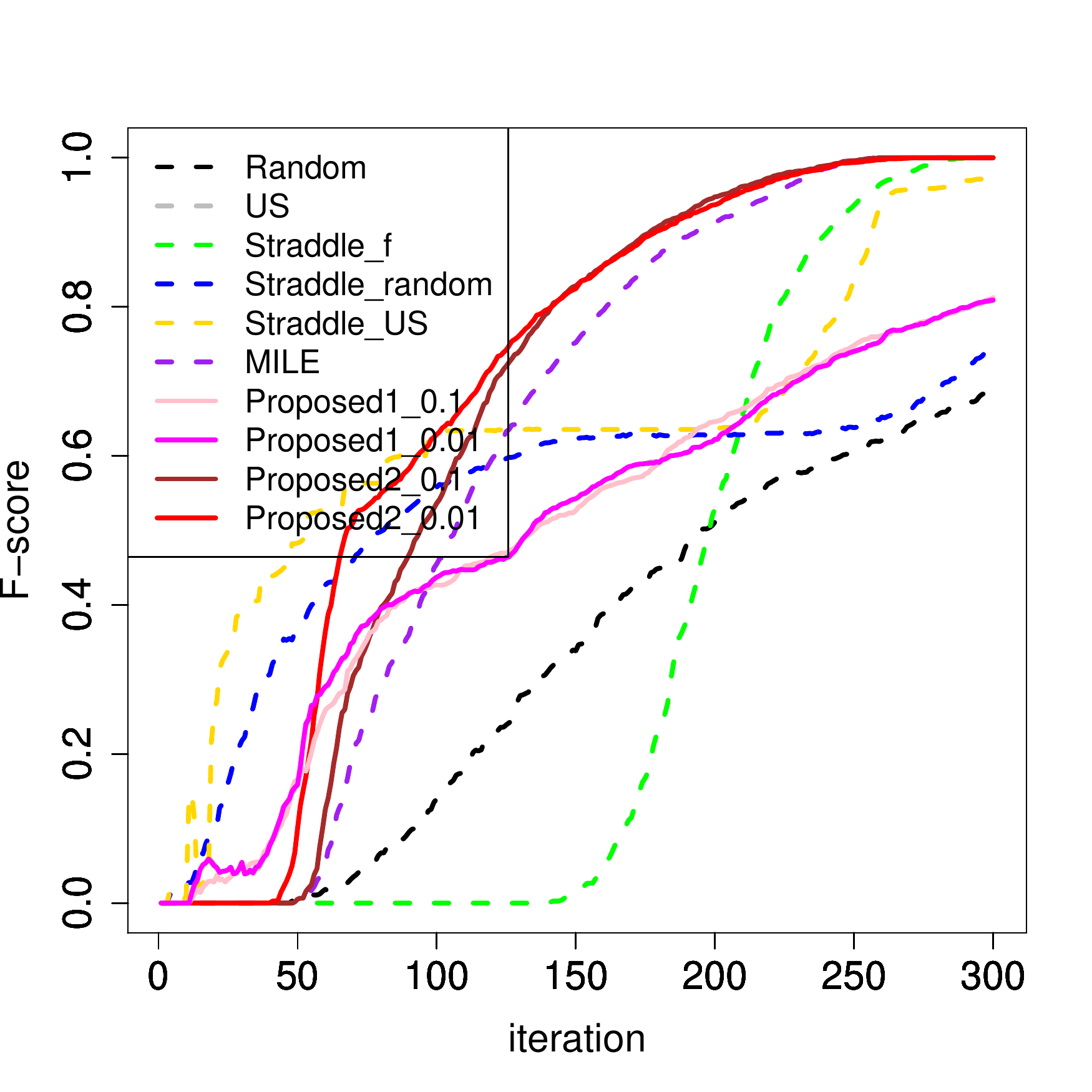} 
\\
Booth & Matyas &
 McCormick & Styblinski-Tang
 \end{tabular}
\end{center}
 \caption{Average F-score over 50 simulations with four benchmark functions when the distance function and reference distribution are $L2$-norm and Normal, respectively.}
\label{fig:exp4}
\end{figure}

\begin{figure}[!t]
\begin{center}
 \begin{tabular}{cc}
\includegraphics[width=0.45\textwidth]{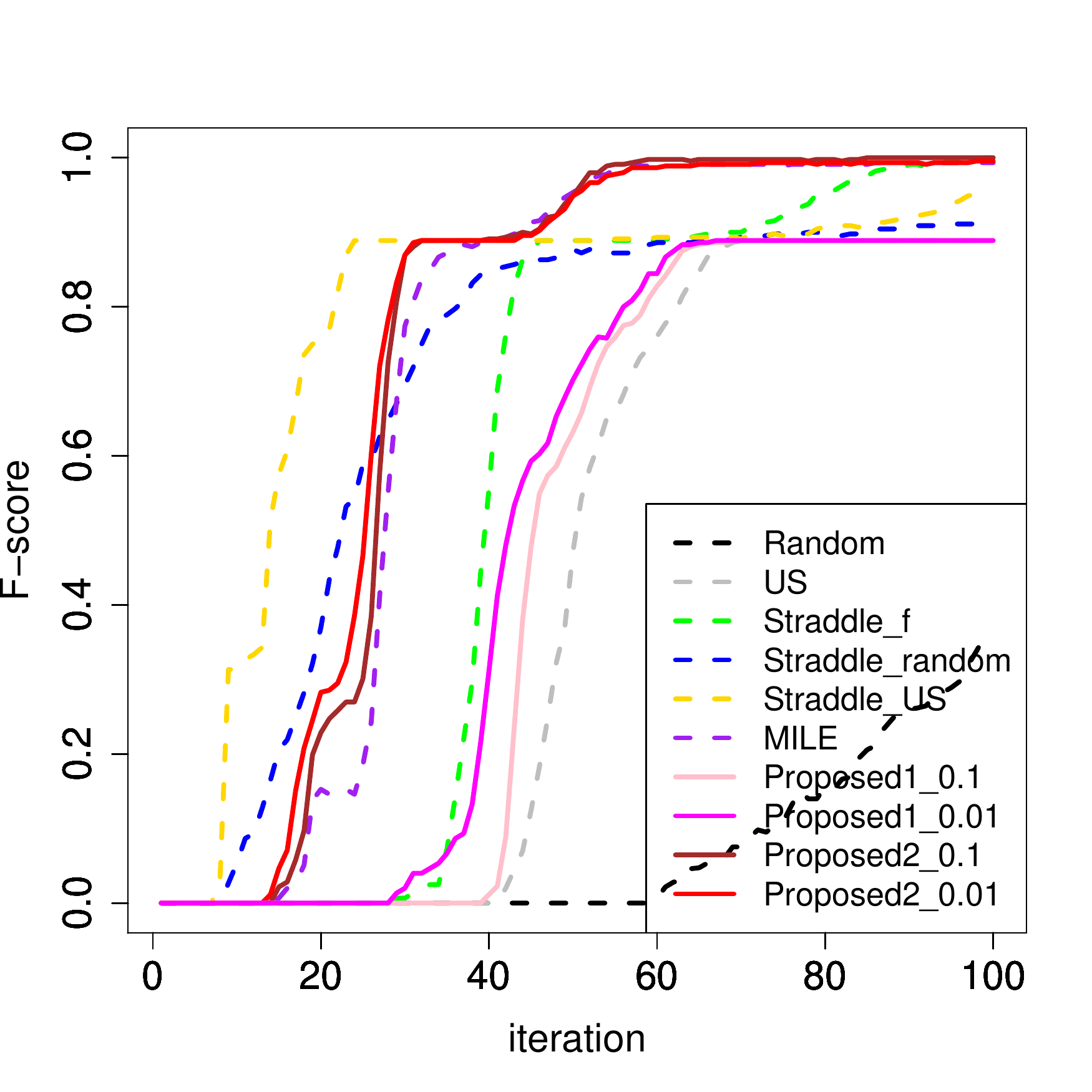} 
&
\includegraphics[width=0.45\textwidth]{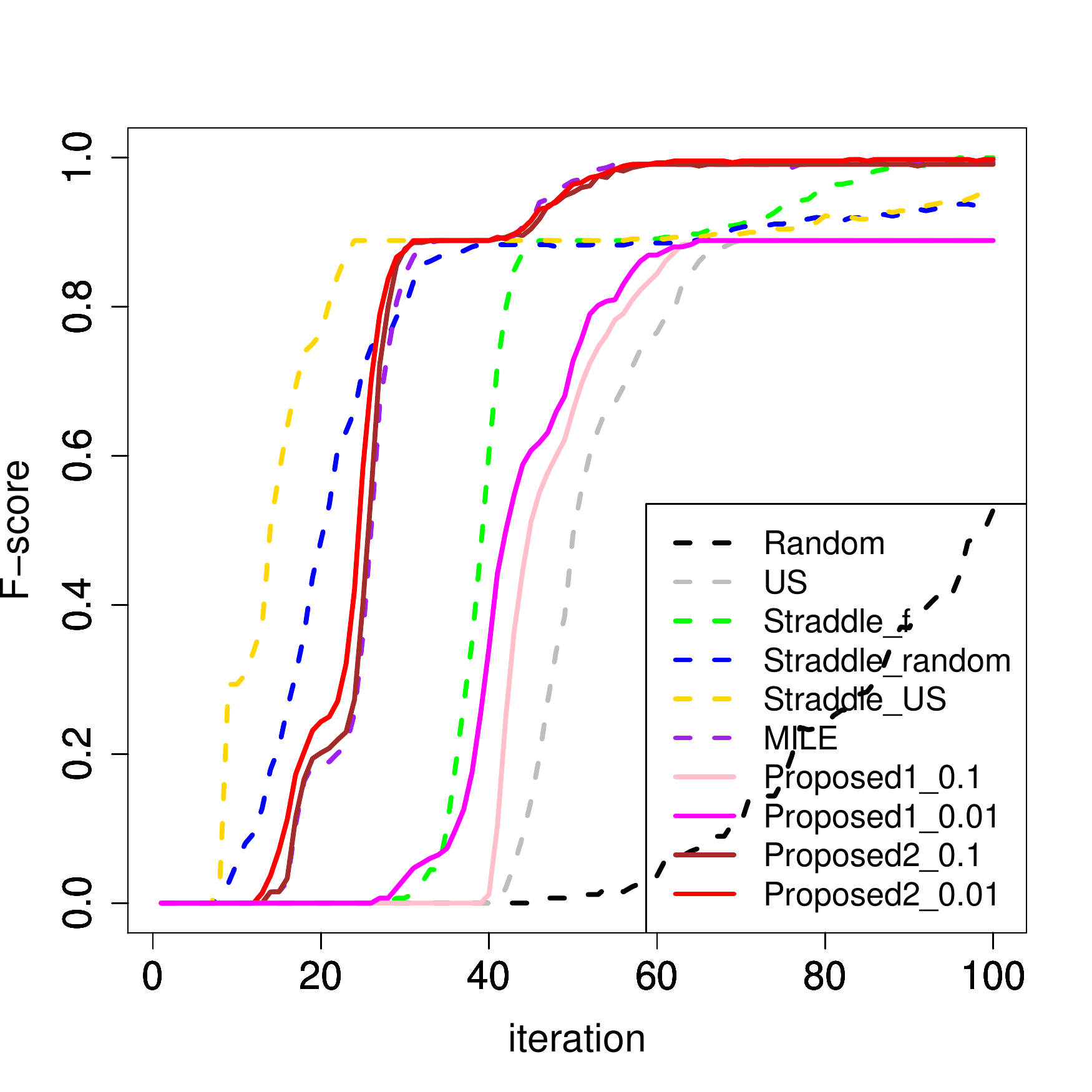} 
\\
   $L2$-Uniform & $L2$-Normal
 \end{tabular}
\end{center}
 \caption{Average F-score over 50 simulations in the infection control problem when the distance function is $L2$-norm.}
\label{fig:exp5_2}
\end{figure}

\subsection{Computation time experiments in the other benchmark function setting}
In this section, we performed the same experiment as in Subsection \ref{subsec:comp_time} for the Matyas, McCormick and Styblinski-Tang benchmark functions.
We evaluated the computation time of \eqref{eq:AFteigi} when we performed the same experiment as in Subsection \ref{subsec:comp_time} using Proposed1\_$0.01$ and Proposed2\_$0.01$.
Here, as for the parameter settings, we considered only the case of $L1$-Normal in Table \ref{tab:setting1}. 
Under this setup, we took one initial point at random and ran the algorithms until the number of iterations reached to 300. 
Furthermore, for each trial $t$, we evaluated the computation time to calculate \eqref{eq:AFteigi} for all candidate points $({\bm x}^\ast,{\bm w}^\ast ) \in \mathcal{X} \times \Omega $, and calculated the average computation time over 300 trials.
From Tables \ref{tab:time_Matyas}, \ref{tab:time_McCormick} and \ref{tab:time_ST}, it can be confirmed that the same results as in Subsection \ref{subsec:comp_time} are obtained in the three benchmark function settings.


\begin{table*}[t]
  \begin{center}
    \caption{Computation time (second) for the Matyas function setting}
\scalebox{0.85}{
    \begin{tabular}{c||c|c|c|c|c|c} \hline \hline
       & Naive & L1 & L2 & L3 $(10^{-4})$ & L3 $(10^{-8})$ & L3 $(10^{-12})$ \\ \hline 
   Proposed1\_$0.01$   &$112403.30 \pm 24588.33$  & $6211.88 \pm 1514.06$  &   $1297.19 \pm 726.31$    & $32.12 \pm 7.36$& $32.76 \pm 7.18$ & $33.25 \pm 7.06$   \\ \hline 
   Proposed2\_$0.01$   & $98478.43 \pm 19995.68$ & $5504.84 \pm 1362.62$  &   $1831.17 \pm 1109.59$    & $32.86 \pm 5.43$& $37.50 \pm 3.58$ & $38.24 \pm 4.90$   \\ \hline \hline
    \end{tabular}
}
    \label{tab:time_Matyas}
  \end{center}
\end{table*}

\begin{table*}[t]
  \begin{center}
    \caption{Computation time (second) for the McCormick function setting}
\scalebox{0.85}{
    \begin{tabular}{c||c|c|c|c|c|c} \hline \hline
       & Naive & L1 & L2 & L3 $(10^{-4})$ & L3 $(10^{-8})$ & L3 $(10^{-12})$ \\ \hline 
   Proposed1\_$0.01$   & $83608.24 \pm 39551.78$ & $4692.96 \pm 2274.72$  &   $1094.40 \pm 523.81$    & $39.66 \pm 6.27$& $41.25 \pm 6.20$ & $42.74 \pm 6.86$   \\ \hline 
   Proposed2\_$0.01$   & $79782.95 \pm 39221.70$ & $4383.04 \pm 2286.23$  &   $1525.80 \pm 931.80$    & $49.67 \pm 10.33$& $56.79 \pm 17.54$ & $62.59 \pm 23.83$   \\ \hline \hline
    \end{tabular}
}
    \label{tab:time_McCormick}
  \end{center}
\end{table*}

\begin{table*}[t]
  \begin{center}
    \caption{Computation time (second) for the Styblinski-Tang function setting}
\scalebox{0.85}{
    \begin{tabular}{c||c|c|c|c|c|c} \hline \hline
       & Naive & L1 & L2 & L3 $(10^{-4})$ & L3 $(10^{-8})$ & L3 $(10^{-12})$ \\ \hline 
   Proposed1\_$0.01$   &$118443.10 \pm 16290.13$  & $6297.18 \pm 1039.76$  &   $900.67 \pm 698.84$    & $44.88 \pm 18.66$& $47.32 \pm 20.67$ & $48.35 \pm 21.82$   \\ \hline 
   Proposed2\_$0.01$   &$96731.93 \pm 25845.16$  & $5240.58 \pm 1516.16$  &   $686.64 \pm 796.10$    & $26.77 \pm 10.92$& $27.42 \pm 11.87$ & $28.25 \pm 12.78$   \\ \hline \hline
    \end{tabular}
}
    \label{tab:time_ST}
  \end{center}
\end{table*}

\subsection{Hyperparameter sensitivity in the proposed acquisition function}
In this section, we evaluated how the performance is affected by the hyperparameter $\gamma$ in the proposed acquisition function.
We calculated the F-score for the cases with acquisition functions Proposed1\_$\gamma$ and Proposed2\_$\gamma$ when we performed the same experiment as in Subsection \ref{syn_experiment} for Booth, Matyas, McCormick and Styblinski-Tang functions.
Here, Proposed1\_$\gamma$ and Proposed2\_$\gamma$ respectively represent the acquisition functions $a^{(1)}_t ({\bm x}^\ast,{\bm w}^\ast )$ and $a^{(2)}_t ({\bm x}^\ast,{\bm w}^\ast )$ with the parameter $\gamma$, and we considered $\gamma$ as 
$0$, $10^{-0.5}$, $10^{-1}$, $10^{-2}$, $10^{-3}$ and $10^{-4}$. 
In this experiment, as for the parameter settings, we considered only the case of $L1$-Uniform in Table \ref{tab:setting1}. 
Under this setup, we took one initial point at random and ran the algorithms until the number of iterations reached 300 (or 200).
We performed 50 Monte Carlo simulations and calculated the average F-score.

\begin{figure*}[t]
\begin{center}
 \begin{tabular}{cccc}
 \includegraphics[width=0.225\textwidth]{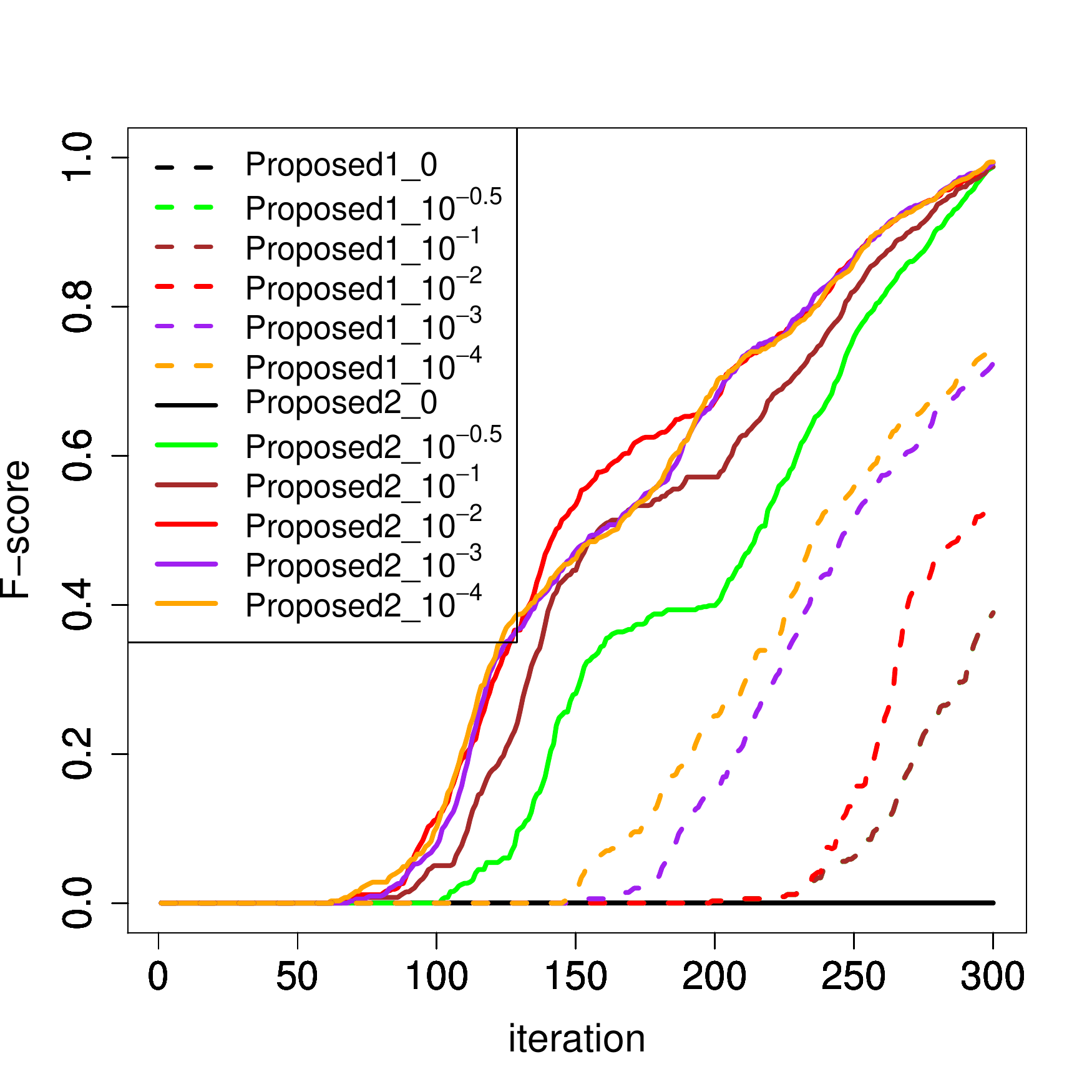} 
&
 \includegraphics[width=0.225\textwidth]{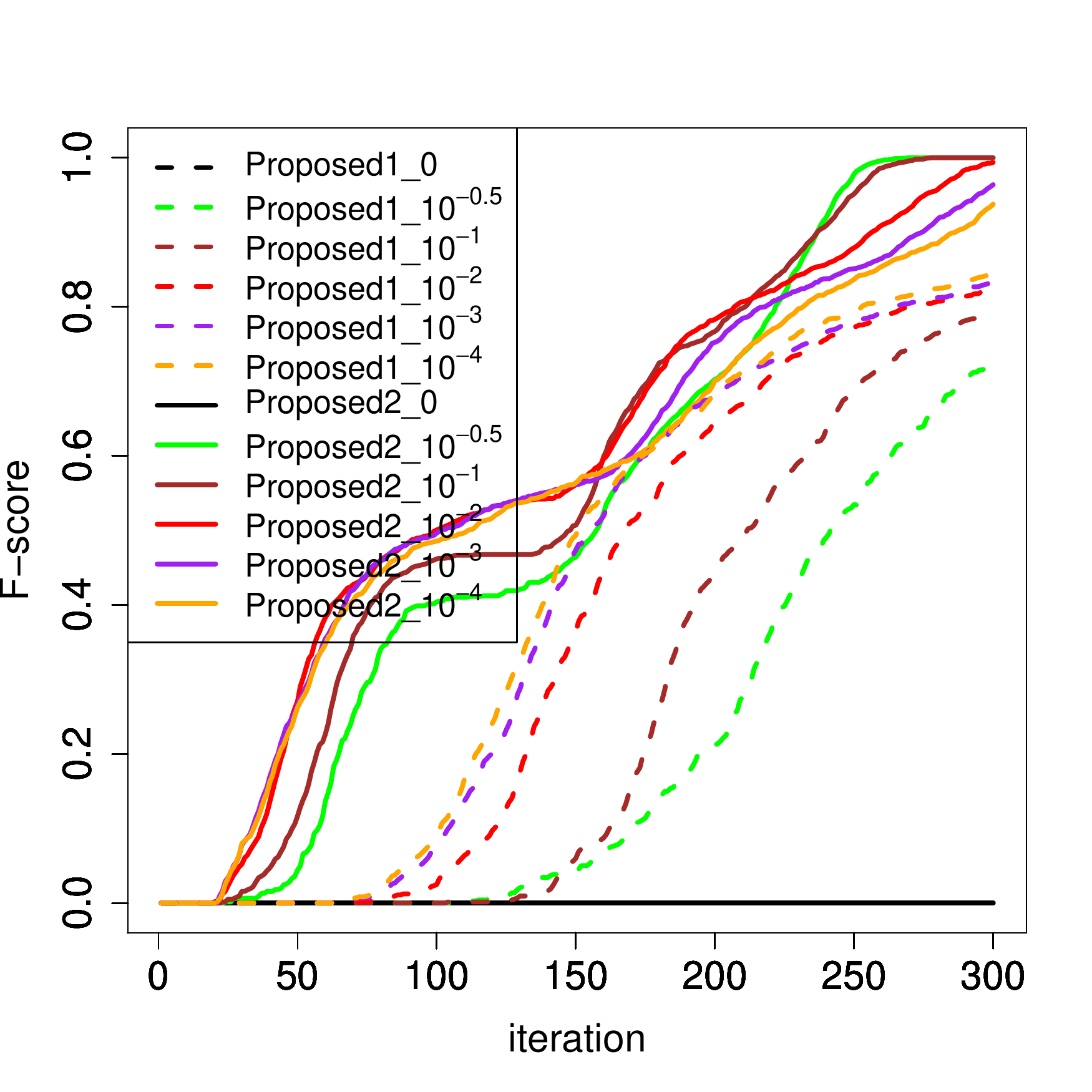} 
&
 \includegraphics[width=0.225\textwidth]{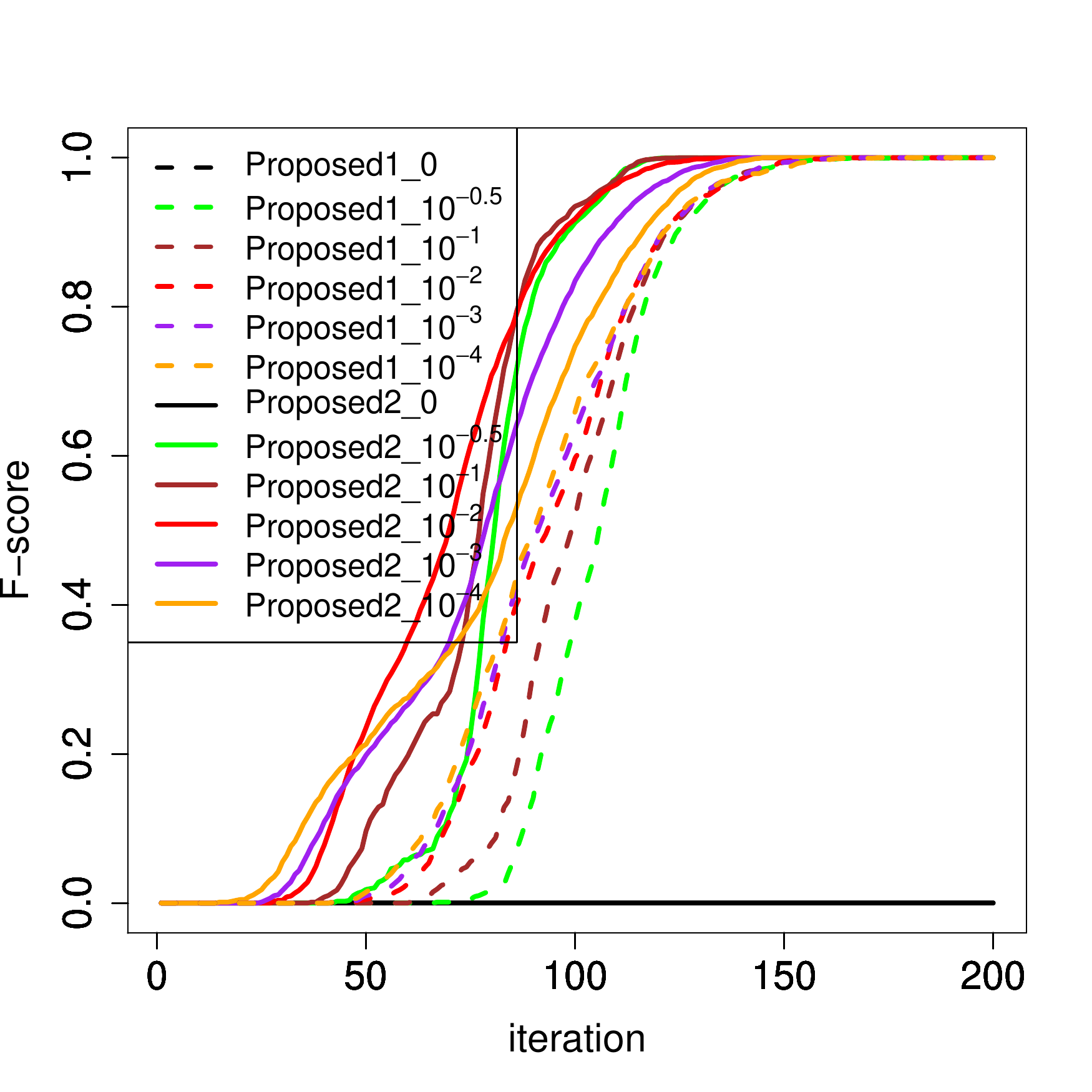} 
&
 \includegraphics[width=0.225\textwidth]{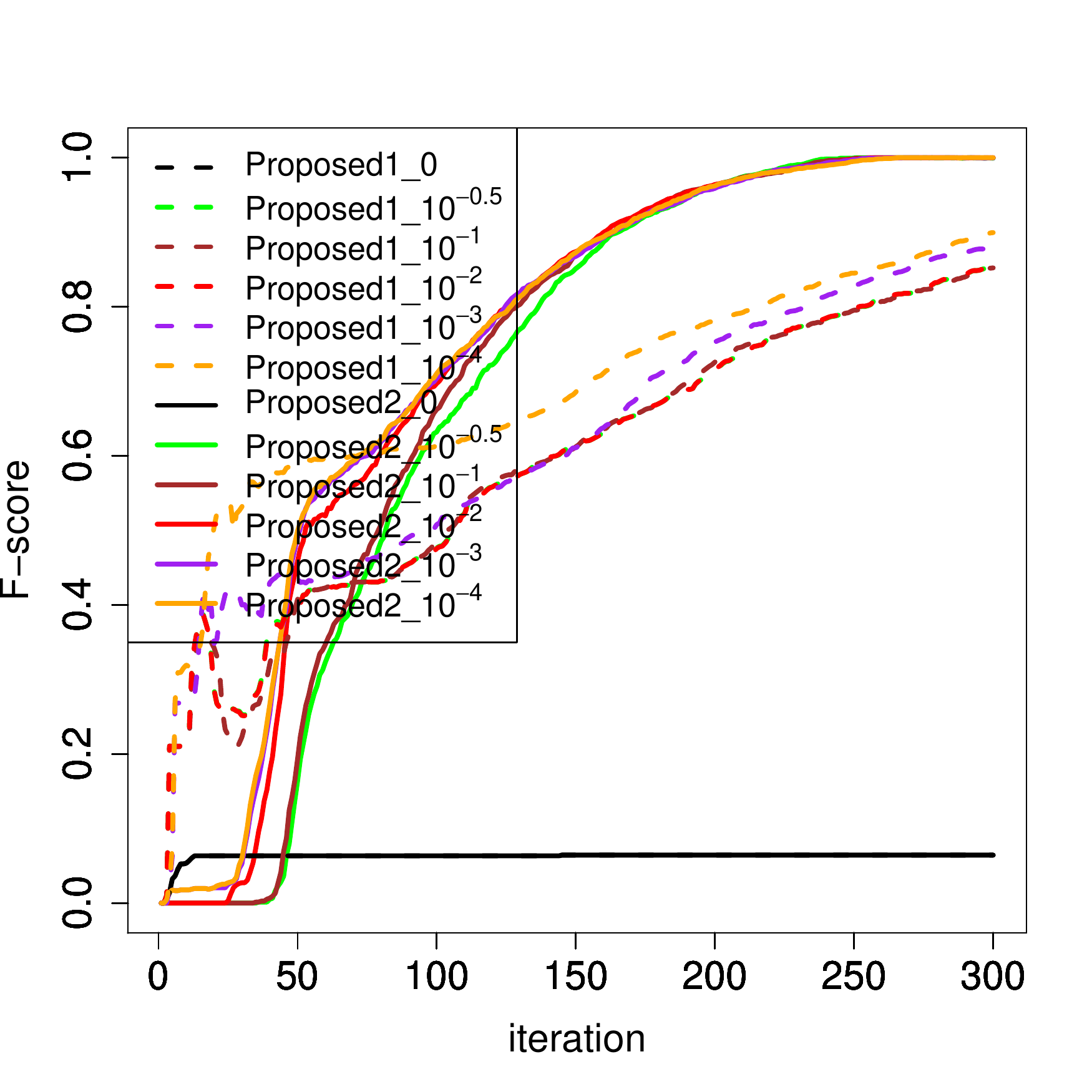} 
\\
Booth & Matyas & McCormick & Styblinski-Tang
 \end{tabular}
\end{center}
 \caption{Difference in average F-score for different hyperparameters with four benchmark functions when the distance function and reference distribution are $L1$-norm and Uniform, respectively.}
\label{fig:hyper_sensitivity}
\end{figure*}

From Figure \ref{fig:hyper_sensitivity}, it can be confirmed that the acquisition function does not work for all benchmark functions when $\gamma=0$.
The reason is that $a_t ({\bm x}^\ast,{\bm w}^\ast )$ was zero for all $({\bm x}^\ast, {\bm w}^\ast ) \in \mathcal{X} \times \Omega$ when the number of data was small.
Furthermore, when $\gamma >0$, it can be seen that the performance of Proposed1\_$\gamma$ decreases as $\gamma$ increases. 
One reason is that although $a^{(1)}_t ({\bm x}^\ast,{\bm w}^\ast )$ is closer to uncertainty sampling (US) as $\gamma$ becomes large, US is not the acquisition function for efficiently estimating $H_t$.
On the other hand, it can be confirmed that the performance of Proposed2\_$\gamma$ is not necessarily better when $\gamma$ is smaller. 
From the definition of Proposed2\_$\gamma$, when $\gamma$ is large, $a^{(2)}_t ({\bm x}^\ast,{\bm w}^\ast )$ behaves similarly to RMILE.
 RMILE is the acquisition function that works to efficiently identify $({\bm x} ,{\bm w} ) $ that satisfies $f({\bm x} ,{\bm w} ) >h $.
However, since $F({\bm x} )$ is given as the function of $\1[f({\bm x} ,{\bm w} ) >h]$, as a result, RMILE also works to efficiently estimate $H_t$.
This is one of the reasons why Proposed2\_$\gamma$ sometimes has good performance even at large $\gamma$.

  \end{document}